\documentclass{article}

     \PassOptionsToPackage{round, compress}{natbib}

\usepackage[utf8]{inputenc} 
\usepackage[T1]{fontenc}    

\usepackage{wrapfig,lipsum,booktabs}

\usepackage{paralist}
\usepackage{booktabs}       
\usepackage{nicefrac}       
\usepackage{multirow}
 \usepackage{xargs}

\usepackage[tbtags]{amsmath}
\usepackage{amsthm}
\usepackage{bm}
\allowdisplaybreaks
\usepackage{amssymb,mathrsfs}
\usepackage{amsfonts}
\usepackage{upgreek}
\usepackage{graphicx}
\usepackage{wrapfig}
\usepackage[dvipsnames]{xcolor}
\usepackage{soul}
\usepackage{pifont}
\usepackage{bbm}
\usepackage[colorlinks = true, citecolor = blue]{hyperref}
\usepackage{algpseudocode,algorithm,algorithmicx}
\usepackage{stmaryrd}
\usepackage{array}
\usepackage{enumitem}

\usepackage{tikz}
\usepackage{pgfplots}
\usepackage{aliascnt}
\usepackage{todonotes}
\usepackage{cleveref}
\usepackage{autonum}
\usepackage{accents}

\newcommand{\xmark}{\textcolor{red}{\ding{55}}}

\definecolor{lightred}{rgb}{1, 0.8, 0.8}

\makeatletter
\newtheorem{theorem}{Theorem}
\crefname{theorem}{theorem}{Theorems}
\Crefname{Theorem}{Theorem}{Theorems}

\newtheorem*{lemma_nonumber*}{Lemma}

\newtheorem{lemma}{Lemma}

\newaliascnt{corollary}{theorem}

\aliascntresetthe{corollary}
\crefname{corollary}{corollary}{corollaries}
\Crefname{Corollary}{Corollary}{Corollaries}

\newaliascnt{proposition}{theorem}

\aliascntresetthe{proposition}
\crefname{proposition}{proposition}{propositions}
\Crefname{Proposition}{Proposition}{Propositions}

\newaliascnt{definition}{theorem}

\aliascntresetthe{definition}
\crefname{definition}{definition}{definitions}
\Crefname{Definition}{Definition}{Definitions}

\newaliascnt{remark}{theorem}

\aliascntresetthe{remark}
\crefname{remark}{remark}{remarks}
\Crefname{Remark}{Remark}{Remarks}

\crefname{example}{example}{examples}
\Crefname{Example}{Example}{Examples}

\crefname{figure}{figure}{figures}
\Crefname{Figure}{Figure}{Figures}

\newtheorem{assumption}{\textbf{H}\hspace{-3pt}}
\crefformat{assumption}{{\textbf{H}}#2#1#3}
\newtheorem{assumptionsup}{\textbf{A}\hspace{-3pt}}
\crefformat{assumptionsup}{{\textbf{A}}#2#1#3}

\definecolor{darkgreen}{RGB}{0,128,0}
\definecolor{darkorange}{RGB}{255,140,0}
\definecolor{darkblue}{RGB}{0,0,139}

\def\msb{\mathsf{B}}

\def\msu{\mathsf{U}}

\def\mcb{\mathcal{B}}

\def\rset{\mathbb{R}}

\def\nsets{\mathbb{N}^*}

\def\rmd{\mathrm{d}}

\def\rmc{\mathrm{C}}

\newcommand{\argmax}{\operatorname*{arg\,max}}

\newcommand{\LeftEqNo}{\let\veqno\@@leqno}

\newcommand{\tvnorm}[1]{\| #1 \|_{\mathrm{TV}}}

\newcommand{\parentheseLigne}[1]{(#1 )}

\newcommand{\ocint}[1]{\left(#1\right]}

\newcommand{\ccint}[1]{\left[#1\right]}

\newcommand{\ball}[2]{\operatorname{B}(#1,#2)}

\newcommandx\sequence[3][2=,3=]
{\ifthenelse{\equal{#3}{}}{\ensuremath{( #1_{#2})}}{\ensuremath{( #1_{#2})_{ #2 \in #3}}}}

\newcommandx\sequencet[3][2=,3=]
{\ifthenelse{\equal{#3}{}}{\ensuremath{( #1_{#2})}}{\ensuremath{( #1_{#2})_{ #2 \geq #3}}}}

\newcommand{\wrt}{w.r.t.}

\def\iid{\text{i.i.d.}}

\newcommand{\opnorm}[1]{{\left\vert\kern-0.25ex\left\vert\kern-0.25ex\left\vert #1
    \right\vert\kern-0.25ex\right\vert\kern-0.25ex\right\vert}}

\newcommand\coupling[2]{\Gamma(\mu,\nu)}

\def\mtt{\mathtt{m}}

\newcommand{\N}{\mathbb{N}}

\newcommand{\R}{\mathbb{R}}

\newcommand{\msa}{\mathsf{A}}
\newcommand{\msx}{\mathsf{X}}
\newcommand{\msy}{\mathsf{Y}}
\newcommand{\mcy}{\mathcal{Y}}
\newcommand{\mcx}{\mathcal{X}}

\newcommand{\Rd}{\mathbb{R}^{d}}

\newcommand{\dd}{\mathrm{d}}

\newcommand{\argmin}{\operatornamewithlimits{\arg\min}}

\newcommand{\eqsp}{\,}  

\newcommand{\pr}[1]{\left({#1}\right)}
\newcommand{\prn}[1]{({\textstyle{#1}})}
\newcommand{\prLigne}[1]{({\textstyle{#1}})}

\newcommand{\br}[1]{\left[{#1}\right]}
\newcommand{\bbr}[1]{\left\{{#1}\right\}}

\newcommand{\brnn}[1]{#1}

\newcommand{\norm}[1]{\left\|{#1}\right\|}

\newcommand{\abs}[1]{\left\lvert{#1}\right\rvert}

\newcommand{\gauss}{\mathrm{N}}
\newcommand{\up}{\mathscr{C}}

\def\thetas{\theta^{\star}}
\def\thetasi{\theta^{\star,(i)}}

\makeatletter
\newcommand{\ostar}{\mathbin{\mathpalette\make@circled\star}}
\newcommand{\pstar}{\mathbin{\mathpalette\make@circled+}}
\newcommand{\make@circled}[2]{%
  \ooalign{$\m@th#1\smallbigcirc{#1}$\cr\hidewidth$\m@th#1#2$\hidewidth\cr}%
}
\newcommand{\smallbigcirc}[1]{%
  \vcenter{\hbox{\scalebox{0.77778}{$\m@th#1\bigcirc$}}}%
}
\makeatother

\def\dtheta{d_{\Theta}}

\newcommandx{\Vnorm}[2][1=V]{\| #2 \|_{#1}}
\newcommandx{\VnormEq}[2][1=V]{\left\| #2 \right\|_{#1}}
\newcommand{\Kker}{\mathrm{K}}

\newcommand{\Pker}{\mathrm{P}}
\def\Rdtheta{\rset^{d_{\Theta}}}

\usepackage{xr-hyper}
\usepackage{xr}
\newcommand{\greencheck}{{\color{OliveGreen}\checkmark}}
\newcommand{\redcross}{{\color{BrickRed}\xmark}}

     \usepackage[final]{neurips_2022}

\usepackage[utf8]{inputenc} 
\usepackage[T1]{fontenc}    
\usepackage{hyperref}       
\usepackage{url}            
\usepackage{booktabs}       
\usepackage{amsfonts}       
\usepackage{nicefrac}       
\usepackage{microtype}      
\usepackage{xcolor}         
\usepackage{wrapfig}

\title{FedPop: A Bayesian Approach for Personalised Federated Learning}

\author{%
  Nikita Kotelevskii$^*$\\
  Skolkovo Institute of Science and
  Technology \\
  Moscow, Russia\\
  \texttt{Nikita.Kotelevskii@skoltech.ru} \\
  $^*$ Equal contribution \\
  \And 
  Maxime Vono$^*$\\
  Criteo AI Lab\\
  Paris, France\\
  \texttt{m.vono@criteo.com} \\
  $^*$ Equal contribution \\
  \And
  Alain Durmus\\
  ENS Paris-Saclay \\
  \texttt{alain.durmus@ens-paris-saclay.fr} \\
  \And
  Eric Moulines\\
  Ecole Polytechnique \\
  \texttt{eric.moulines@polytechnique.edu} \\
}

\begin{document}

\maketitle

\begin{abstract}
  Personalised federated learning (FL) aims at collaboratively learning a machine learning model tailored for each client. 
  Albeit promising advances have been made in this direction, most of existing approaches do not allow for uncertainty quantification which is crucial in many applications.
  In addition, personalisation in the cross-silo and cross-device setting still involves important issues, especially for new clients or those having small number of observations.
  This paper aims at filling these gaps.
  To this end, we propose a novel methodology coined \texttt{FedPop} by recasting personalised FL into the population modeling paradigm where clients' models involve \emph{fixed} common population parameters and \emph{random} effects, aiming at explaining data heterogeneity.
  To derive convergence guarantees for our scheme, we introduce a new class of federated stochastic optimisation algorithms which relies on Markov chain Monte Carlo methods. 
  Compared to existing personalised FL methods, the proposed methodology has important benefits: it is robust to client drift, practical for inference on new clients, and above all, enables uncertainty quantification under mild computational and memory overheads.
  We provide non-asymptotic convergence guarantees for the proposed algorithms and illustrate their performances on various personalised federated learning tasks.
\end{abstract}

\section{Introduction}
\label{sec:intro}

Federated learning (FL) is a recent machine learning paradigm in which distributed clients holding siloed data collaborate in solving a learning problem, usually under the coordination of a central server \citep{FLreview2021,kairouz2019advances}.
One of the main focus of FL is on so-called \emph{cross-device} applications where a large number of personal electronic devices such as mobile phones, wearable devices or home assistants collect and store data at the edges of a decentralised network \citep{mcmahan17}.

While standard FL methods \citep{mcmahan17,alistarh2017qsgd,karimireddy2020scaffold,DIANA19,FedProx} have focused on training a global model that can be applied to  individual agents, the relevance of such inferences has  recently been questioned  due to statistical \emph{heterogeneity} between clients.
Indeed, the considered common model may not generalise well on a client with a local data distribution that differs significantly from the global data distribution, especially if that client has not participated in the training process.
In fact, it might even be better for such clients to derive a local model from their own data set.
To circumvent these issues, a number of \emph{personalised federated learning} approaches have recently been proposed, that use local models  to fit client-specific data distribution while capturing some shared knowledge via a federated scheme \citep{PFL_review}.
Personalisation has  previously been approached using multi-task learning \citep{smith2017federated}, meta-learning \citep{jiang2019improving,khodak2019adaptive}, client clustering \citep{briggs2020federated}, data interpolation \citep{Mansour3}, model interpolation \citep{hanzely2020federated,hanzely2020lower} or partially local models \citep{NEURIPS2021_5d44a2b0,pmlr-v139-collins21a}. 
While these methodologies are  promising, they only partially solve the personalisation problem in highly heterogeneous federated settings and have no means of quantifying uncertainty.
In addition, cross-device FL also faces other important challenges such as (extreme) partial device participation, small local data sets, limited upload bandwidth and device capabilities \citep{kairouz2019advances}.
Addressing these problems for personalised FL requires new paradigms regarding how model knowledge is shared and personalisation is performed locally. 

\noindent\textbf{Proposed Approach.} In this paper, we adopt a novel perspective to model the problem of  personalised FL.
This paradigm, called \emph{mixed-effects modeling} (also known as multi-level or population approach) is widely used to analyse data that have a clustered or nested structure, as in medical or biological research where multiple observations per patient are available \citep{GelmanHill:2007,Long2011,LaviellePopulation}.
Although the  hierarchical structure of FL has already been noted \citep{2021_ICML_DGLMC,grant2018recasting,2022_bayesian_bandits}, the mixed-effects paradigm has interestingly never been considered.
Leveraging this framework, we develop a new model for personalised FL called \texttt{FedPop} and propose an efficient computational solution to perform inference under this model.
More precisely, we introduce a novel class of federated stochastic approximation algorithms based on parallel Markov Chain Monte Carlo (MCMC) methods.
In the proposed framework, we also pay special attention to the cross-device setting by taking into account partial client participation, and by addressing the communication bottleneck with both multiple local updates and the use of lossy compression operators.

\noindent \textbf{Benefits.} Up to the authors' knowledge, \texttt{FedPop} is the first \emph{personalised FL} approach that allows \emph{cheap uncertainty quantification} with a theoretically-grounded methodology.
The proposed framework also comes with other interesting properties.
First, in contrast to most of personalised FL methods that only consider ``fixed-effects'' models \citep{pmlr-v139-collins21a,Hanzely2021Personalized,smith2017federated}, \texttt{FedPop} provides credibility information (via credibility regions) and allows more accurate inference for clients with small and heterogeneous local data via \emph{partial pooling} \citep{GelmanHill:2007}.
In addition, inference for new clients which did not participate in the training phase can  be easily performed by sampling from the prior over the local random effects. 
Second, contrary to existing Bayesian FL approaches that aim to provide credibility information by sampling from a target posterior distribution \citep{2022_bayesian_bandits,bayesianMAML,VonoQLSD,2020_conducive_gradients}, \texttt{FedPop} allows to perform personalisation and cheaper on-device uncertainty quantification taking an empirical Bayes prediction approach.
Finally, an important benefit of \texttt{FedPop} is its ability to allow for multiple local updates without suffering from the client-drift phenomenon \citep{karimireddy2020scaffold}.

\noindent \textbf{Outline and Contributions.} Our contributions are fourfold. 
First,  in \Cref{sec:fed_pop}, we propose a novel probabilistic methodology, which we call \texttt{FedPop}, to address personalisation under the cross-device FL paradigm.
To perform efficient inference under this model, we introduce a general class of stochastic approximation algorithms based on MCMC.
Second, we provide in \Cref{sec:theory} non-asymptotic convergence guarantees for the proposed methodology. 
Then, we perform in \Cref{sec:related_work} a comparison between the proposed approach and exisiting works.
Finally, we illustrate in \Cref{sec:experiments} the benefits of our methodology on several federated learning benchmarks involving both synthetic and real data. 


\section{Proposed Approach}
\label{sec:fed_pop}


In this section, we present the statistical estimation problem we are considering and the proposed methodology called \texttt{FedPop} to address it.


\noindent \textbf{Problem Formulation.} We are interested  in the \emph{cross-device} FL setting involving a large number $b \in \N^*$ of clients, potentially unreliable \emph{i.e.} not necessarily available at each communication round.
These clients are assumed to own sensitive local data sets $\{\mathrm{D}_{i}\}_{i \in [b]}$.
In this framework, we aim to make both uncertainty quantification and personalised statistical inference by learning a local model tailored to each client.
To this end, and inspired by the population approach used in the biological and physical sciences \citep{LaviellePopulation}, we consider mixed-effects modeling for each client leading to the local marginal likelihood function defined, for any $i \in [b]$, by
\begin{equation}
  \label{eq:marginal_likelihood}
  p\prLigne{\mathrm{D}_i \mid \phi, \beta} = \int_{\mathbb{R}^d} p\prLigne{\mathrm{D}_i \mid \phi \eqsp, z^{(i)}} p\prLigne{z^{(i)}\mid \beta}\mathrm{d}z^{(i)}\eqsp, 
\end{equation}
where $\phi \in \Phi \subseteq \R^{d_\Phi}$ stands for a \emph{fixed effect} and $\{z^{(i)}\}_{i\in[b]} \in \mathsf{Z}$, $\mathsf{Z} = \prod_{i=1}^b \R^{d}$, represent \emph{random effects} aimed at explaining statistical heterogeneity between local data sets $\{\mathrm{D}_i\}_{i\in [b]}$.
\begin{wrapfigure}{r}{0.3\textwidth}
\vspace{-0.5cm}
  \begin{center}
    \includegraphics[scale=0.26]{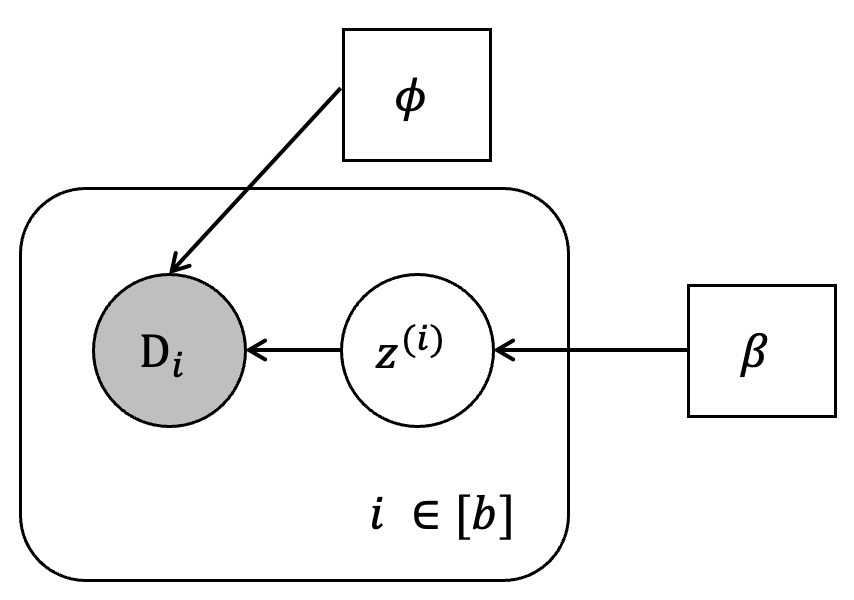}
  \end{center}
  \caption{DAG for \texttt{FedPop}.}
  \label{fig:DAG}
  \vspace{-0.5cm}
\end{wrapfigure}
The objective of the fixed (\emph{i.e.} constant across all clients) part is to capture a common representation (\emph{e.g.} same features across different classes of images) while the random part, which is typically low-dimensional, performs personalisation and is assumed to be drawn from a \emph{population} prior whose variance aims at modeling data heterogeneity.

\Cref{fig:DAG} illustrates this statistical framework, referred to as \texttt{FedPop}, by showing its directed acyclic graph (DAG) where grey-filled shapes indicate observed variables, white-filled shapes unknown variables and squared shapes variables to be estimated. 

When the size of the local data set $\mathrm{D}_i$ is small, this common prior leverages information from other clients to limit the risk of overfitting and is often called \emph{partial pooling} in the multi-level statistical literature \citep[Section 12]{GelmanHill:2007}. 
Examples of model architectures involving $\phi$ and $\{z^{(i)}\}_{i \in [b]}$ include for instance \emph{composition}-based architectures $p\prLigne{\mathrm{D}_i \mid \phi, z^{(i)}}  = p\prLigne{\mathrm{D}_i \mid h_{\phi} \circ h_{z^{(i)}}}$ where $h_{\phi}$ and $h_{z^{(i)}}$ are two neural networks \citep{pmlr-v139-collins21a,Arivazhagan2019}. 
For the sake of generality, we propose to adopt a flexible energy-based prior distribution of the form for each $i\in [b]$,
\begin{equation}
  p\prn{z^{(i)}\mid \beta} = \frac{1}{Z(\beta)}\exp\bbr{-E(z^{(i)};\beta)}\eqsp, \text{ where $Z(\beta) = \int_{\Rd}\exp\bbr{-E(z^{(i)};\beta)} \dd z^{(i)}$ }\label{eq:EBM} \eqsp.
\end{equation}
Here, $Z(\beta)$ is a normalising constant and $E(\cdot;\beta)$ represents an \emph{energy} function, typically a neural network, parameterised by a set of parameters
$\beta \in \mathsf{B} \subseteq \R^{d_\mathsf{B}}$ \citep{LeCun06atutorial}. 
This framework is particularly interesting in the cross-device setting where the number of clients $b$ is large as it allows for efficient enrichment of the model.
However, in the case where $b$ is small, the inference of the parameter $\beta$ is difficult. 
In this situation, a more pragmatic solution is to consider a common prior for the local random effects $\{z^{(i)}\}_{i\in[b]}$ which is held fixed, \emph{i.e.} $p\prn{z^{(i)}\mid \beta} \propto  \exp\{-E(z^{(i)})\}$ for any $\beta \in \msb$. 
Finally, for completeness, we allow the use of a prior model $p(\phi,\beta)$ for the hyperparameters $\{\phi,\beta\}$.  
Using Bayes' rule \citep{Robert94} and by denoting $\mathrm{D} = \sqcup_{i=1}^{b} \mathrm{D}_{i}$ the global data set, the posterior distribution
associated with these hyperparameters admits a probability density
function which can be written as
\begin{equation}
  \label{eq:posterior}
  p\pr{\phi, \beta \mid \mathrm{D}} \propto p(\phi,\beta)\prod_{i=1}^b \br{\int_{\mathbb{R}^d} p\prn{\mathrm{D}_i \mid \phi, z^{(i)}} p\prn{z^{(i)}\mid \beta}\mathrm{d}z^{(i)}}\eqsp.
\end{equation}
Set $\theta = \{\phi,\beta\} \in \Theta$ with $\Theta = \Phi \times \mathsf{B}$.
In the sequel, we will be interested in solving the maximum a posteriori problem given by 
\begin{align}
  \theta^\star
  &\in \argmax_{\theta \in \Theta} \ \log p\prn{\phi, \beta \mid \mathrm{D}}\eqsp,  \label{eq:minimization_pb}\\
 \log p\prn{\phi, \beta \mid \mathrm{D}}  &= \log p(\phi,\beta) + \sum_{i=1}^b \br{\log \int_{\mathbb{R}^d} p\prn{\mathrm{D}_i \mid \phi, z^{(i)}} p\prn{z^{(i)}\mid \beta}\mathrm{d}z^{(i)}} + C\eqsp, \label{eq:minimization_objective}
\end{align}
where $C \in \R$ is a constant independent of $\theta$.
Once we have estimated $\theta^\star$, using an empirical Bayesian approach, we can perform ``for free'' on-device uncertainty quantification for each client $i \in [b]$ by sampling from the local posterior distribution $p(z^{(i)} \mid \mathrm{D}_i, \phi^\star, \beta^\star)$, which is typically designed to be low-dimensional.

\noindent \textbf{Algorithm.} To solve the optimisation problem \eqref{eq:minimization_pb}, we can either use an \emph{alternating maximisation} algorithm or perform global maximisation over $\Theta$.
Since the former approach requires more upload bandwidth,  in this work we consider the second alternative which is more suitable for FL.
The gradient of the objective function \eqref{eq:minimization_objective} being intractable, we propose to resort to the stochastic approximation framework \citep{robbins1951} which iteratively defines $(\phi_k,\beta_k)_{k \in \N}$, starting from any $(\phi_0,\beta_0)\in \Theta$, via the recursions for any $k \in \N$,
\begin{align}
  \beta_{k+1}
  &= \Pi_{\mathsf{B}}\pr{\beta_k+ \eta_{k+1}^{(1)} \br{\nabla_\beta\log p(\phi_k,\beta_k) + \sum_{i=1}^b\textsl{g}^{(i)}_k(\phi_k,\beta_k)}}\eqsp, \label{eq:SA_beta}\\
  \phi_{k+1} &= \Pi_{\Phi}\pr{\phi_k + \eta_{k+1}^{(2)}\br{\nabla_\phi \log p(\phi_k,\beta_k) + \sum_{i=1}^b \textsl{h}^{(i)}_k(\phi_k,\beta_k)}}\eqsp,\label{eq:SA_phi}
\end{align}
where $\Pi_{\mathsf{C}}$ denotes the projection onto $\mathsf{C} \in \{\Phi,\mathsf{B}\}$, $(\eta_k^{(1)},\eta_k^{(2)})_{k \in \N^*}$ are sequences of step-sizes, and $\{\textsl{g}^{(i)}_k \, : \, i \in [b]\eqsp,  k \in\nsets \}$ and $\{\textsl{h}^{(i)}_k \, : \, i \in [b]\eqsp,  k \in\nsets \}$ are estimators of the intractable gradients $(\phi,\beta) \mapsto \nabla_{\beta} \log p(\mathrm{D}_i\mid \phi,\beta)$ and $(\phi,\beta) \mapsto \nabla_{\phi} \log p(\mathrm{D}_i\mid \phi,\beta)$ at $(\phi_k,\beta_k)$, where $p(\mathrm{D}_i\mid \phi,\beta)$ is defined in \eqref{eq:marginal_likelihood} for any $i \in [b]$.

The choices of the estimators $\{\textsl{g}^{(i)}_k \, : \, i \in [b]\eqsp,  k \in\nsets \}$ and $\{\textsl{h}^{(i)}_k \, : \, i \in [b]\eqsp,  k \in\nsets \}$ are motivated by the Fisher identity. 
More precisely, under mild regularity assumptions, and using the Lebesgue dominated convergence theorem, we have for any, $(\phi,\beta) \in \Theta$, $i \in [b]$ 
\begin{align}
  \nabla_{\beta} \log p(\mathrm{D}_i\mid \phi,\beta)
  &= \int_{\mathbb{R}^d}\br{ \nabla_{\beta} \log p\prn{\mathrm{D}_i, z^{(i)} \mid \phi,\beta}} p\prn{z^{(i)} \mid \mathrm{D}_i,\phi,\beta}\mathrm{d}z^{(i)} \eqsp, \label{eq:fisher_beta}\\
  \nabla_{\phi} \log p(\mathrm{D}_i\mid \phi,\beta)
  &= \int_{\mathbb{R}^d} \br{\nabla_{\phi} \log p\prn{\mathrm{D}_i, z^{(i)} \mid \phi,\beta}} p\prn{z^{(i)} \mid \mathrm{D}_i,\phi,\beta}\mathrm{d}z^{(i)}\eqsp, \label{eq:fisher_phi}
\end{align}
which suggests to consider 
\begin{align}
 \textsl{g}_k^{(i)}(\phi,\beta)
  &  = \frac{1}{M} \sum_{m=1}^M \nabla_{\beta} \log \brnn{p\prn{ Z_k^{(i,m)}\mid \beta}} \eqsp, \label{eq:estimator_beta}\\
\textsl{h}_k^{(i)}(\phi,\beta) 
    &=  \frac{1}{M} \sum_{m=1}^M \nabla_{\phi} \log \brnn{p\prn{\mathrm{D}_i \mid Z_k^{(i,m)}, \phi}}
    \eqsp, \label{eq:estimator_phi}
\end{align}
where $M \in \N^*$ and $Z_k^{(i,1:M)} = (Z_k^{(i,m)})_{m\in[M]}$ are approximate  samples from $p(z^{(i)}\mid\mathrm{D}_i, \phi,\beta)$. 
More precisely, we consider a family $\{Q^{(i)}_{\gamma,\theta} : \gamma \in (0,\bar{\gamma}], \theta \in \Theta\}$ where for any step-size $\gamma$, $Q^{(i)}_{\gamma,\theta}$ is a Markov kernel which targets a close approximation of $p(z^{(i)}\mid\mathrm{D}_i, \theta)$ with $\theta=\{\phi,\beta\}$.
As an example, we can use overdamped Langevin dynamics \citep{Roberts1996,Welling11} to generate these samples. 
In this case, $Q^{(i)}_{\gamma,\theta}$ is associated with a Gaussian probability density function $q^{(i)}_{\gamma,\theta}(z^{(i)},\cdot)$ with mean $z^{(i)} - \gamma \nabla_{z} \log p(z^{(i)} \mid \mathrm{D}_i, \theta)$ and variance $2\gamma \mathrm{I}_d$.
Note that the number of Monte Carlo draws per iteration $k$ is considered constant here but we can easily generalise our scheme  to the non-constant setting.
In addition, our scheme can also be generalised by taking into account \emph{stochastic} gradient estimators of \eqref{eq:estimator_beta} and \eqref{eq:estimator_phi}. For the sake of simplicity, we present our approach with standard gradients.

In this framework, we present the main steps of the corresponding stochastic approximation algorithm, called \texttt{FedSOUK}, in \Cref{algo:FedSOUL}. 
Since we consider the \emph{cross-device} federated setting, note that only a random subset $\mathsf{A}_{k+1}$ of active (\emph{i.e.} available) clients communicates with the central server at each iteration $k \in \N$.
In addition, due to limited upload bandwidth, the potentially high-dimensional gradient estimator \eqref{eq:estimator_phi} is compressed locally via an unbiased stochastic compression operator $\mathscr{C}_{k+1}$ before being sent to the central server \citep{alistarh2017qsgd,Artemis20}.

\noindent \textbf{Stateful and Stateless Versions.} Depending on local memory constraints and the participation rate, we allow for a possible warm-start strategy across communication rounds to improve the convergence properties of the proposed algorithm so that the proposed algorithm becomes \emph{stateful}, see steps $4-7$ in \Cref{algo:FedSOUL}.
In cases the participation rate is very low (\emph{e.g.} each client might only participates once to the training process), we replace this warm-start strategy by the initialisation $Z_{k}^{(i,0)} \sim p(\cdot \mid \beta_k)$ for any $i \in [b]$ and $k \in \{0,\ldots,K-1\}$.
This yields a \emph{stateless} version of our algorithm more suitable to the cross-device setting. 
Obviously, compared to the previously proposed warm-start strategy, the performances of \Cref{algo:FedSOUL} will be affected negatively if we are using the same number of local iterations $M$. 
We end up with an interesting trade-off between local computations and communication: if client-server communication is a bottleneck, the stateless version of the algorithm allows to reduce the communication overhead at the price of longer sampling procedures on each client.
Such analyses will be illustrated empirically in \Cref{sec:experiments}.

\noindent \textbf{Computation Complexity.} Compared to standard FL methods, our approach has an additional computational cost on the client side associated with Monte Carlo approximations $\{I_k^{(i)}\}$ and $\{J_k^{(i)}\}$.
In practice, this cost is negligible. 
Indeed, in our experiments, we found that using a small value of $M \in [1,10]$ was sufficient to obtain state-of-the-art results in terms of accuracy on the test dataset. 
We would like to emphasize that this additional computational cost has also two side advantages compared to existing FL approaches: (1) it allows us to communicate less frequently with the central server and (2) it allows us to converge faster when the number of local iterations increases since Monte Carlo approximation becomes better.

\noindent \textbf{Communication Overhead.} As pointed out in \Cref{table:overview}, our methodology \texttt{FedPop} improves upon existing FL approaches regarding the communication overhead. Indeed, \texttt{FedPop} offers the flexibility to use both compression for sending updates to the server, and multiple local steps to reduce the communication frequency. As such, depending on the bandwidth and local computational power, the practitioner can adapt the number of local iterations and the parameter of the compression operator. Up to our knowledge, this work is the first one combining compression and multiple local steps for personalised FL.

\noindent \textbf{Robustness to client drift.} For simplicity, we will take the example of \texttt{FedSOUL} (see \Cref{algo:FedSOUL}) which uses the Markov kernel associated with Langevin Monte Carlo to compute gradient estimates of the local marginal likelihood. 
However, our answer holds for general Markov kernels (adjusted or unadjusted). In this scenario, $M$ steps of Langevin Monte Carlo are performed on each device to draw samples $\{Z_k^{(i,m)}\}$ used to compute Monte Carlo estimates 
$\{I_k^{(i)}\}$ and $\{J_k^{(i)}\}$. Increasing the number of local steps $M$ does not slow down convergence but instead allows for more accurate Monte Carlo integration and hence better convergence properties. In contrast, the client drift phenomenon for classical FL approaches (e.g. \texttt{FedAvg} proposed in \citet{mcmahan17}) slows down convergence as the number of local iterations $M$ increases.

\noindent \textbf{Simple inference on new clients.} Typical personalized FL approaches such as \texttt{DITTO} or \texttt{FedRep} require additional local training for inference on new clients. In contrast, the proposed  methodology \texttt{FedPop} allows for a cheaper two-step approach once we have estimated $\theta^\star = (\phi^\star,\beta^\star)$, as detailed below for a new client $b+1$ with feature vector $x$:
\begin{enumerate}
  \item Sample $\{Z_{b+1}^{(l)}\}_{l \in [L]}$ in a i.i.d. manner.
  \item Estimate the posterior predictive function by $p(\cdot \mid x) \approx L^{-1} \sum_{l=1}^L p(\cdot \mid \phi^\star,Z_{b+1}^{(l)},x)$.
\end{enumerate}
The prior $p(z_{b+1} \mid \beta^\star)$ is typically chosen so that sampling is computationally cheap, \emph{e.g.} a Gaussian with diagonal covariance matrix as in our experiments, see \Cref{sec:experiments}.

\begin{algorithm}[h]
   \caption{FL via Stochastic Optimisation using Unadjusted Kernel (\texttt{FedSOUK})}
   \label{algo:FedSOUL}
  \begin{algorithmic}[1]
     \State {\bfseries Input:} nb. outer iterations $K$, nb. local iterations $M$, Markov kernels $\{Q^{(i)}_{\gamma,\theta}\}_{\gamma,\theta,i}$, step-sizes $\{\eta_k^{(1)},\eta_k^{(2)}\}_{k \in [K],i\in[b]}$ and initial points $Z_0^{(0)} \in \Rd$, $\beta_0 \in \mathsf{B}$ and $\phi_0 \in \Phi$.
     \For{$k=0$ {\bfseries to} $K-1$}
     \State Server sends $\{\beta_{k+1},\phi_{k+1}\}$ to clients belonging to $\mathsf{A}_{k+1}$.
     \For{$i \in \mathsf{A}_{k+1}$ \Comment{On active clients $\mathsf{A}_{k+1}$}}
        \State\Comment{Warm-start of the SA scheme if possible}
        \If{$k \ge 1$}
          \State Set $Z_{k}^{(i,0)} = Z_{k-1}^{(i,M)}$.
        \EndIf
        \State \Comment{Computation of key quantities using MCMC}
        \For{$m=0$ {\bfseries to} $M-1$}
          \State Draw $Z_{k}^{(i,m+1)} \sim Q^{(i)}_{\gamma,\theta_k}\pr{Z_{k}^{(i,m)},\cdot}$.
          \State \Comment{For Langevin dynamics} 
          \State \Comment{Draw $\xi_{k}^{(i,m+1)} \sim \mathrm{N}(0_d,\mathrm{I}_d)$.}
          \State \Comment{Set $Z_{k}^{(i,m+1)} = Z_{k}^{(i,m)} + \gamma \nabla_z \log p(Z_{k}^{(i,m)} \mid \mathrm{D}_i, \phi_k,\beta_k) + \sqrt{2\gamma}\xi_{k}^{(i,m+1)}$.}
        \EndFor
       \State \Comment{Communication with the server}
       \State Set $I_k^{(i)} = \frac{1}{M}\sum_{m=1}^M \nabla_{\beta} \log p \pr{Z_k^{(i,m)} \mid \beta_k}$.
       \State Set $J_k^{(i)} = \frac{1}{M}\sum_{m=1}^{M} \nabla_{\phi} \log p \pr{\mathrm{D}_i \mid Z_k^{(i,m)}, \phi_k}$.
       \State Send $I_{k}^{(i)}$ and $\mathscr{C}_{k+1}\pr{J_{k}^{(i)}}$ to the central server.
     \EndFor
     \State Set $\beta_{k+1} = \Pi_{\mathsf{B}}\pr{\beta_{k} + \eta_{k+1}^{(1)} \br{\nabla_{\beta} \log p(\phi_k,\beta_k) + \frac{b}{|\mathsf{A}_{k+1}|}\sum_{i\in \mathsf{A}_{k+1}}I_k^{(i)}}}$.
     \State Set $\phi_{k+1} = \Pi_{\Phi}\pr{\phi_{k} + \eta_{k+1}^{(2)}\br{\nabla_\phi \log p(\phi_k,\beta_k) + \frac{b}{|\mathsf{A}_{k+1}|}\sum_{i\in \mathsf{A}_{k+1}} \mathscr{C}_{k+1}\pr{J_{k}^{(i)}}}}$.
     \EndFor
     \State {\bfseries Output:} $\{\phi_K,\beta_K\}$ and samples $\{Z_{K-1}^{(1:b,m)}\}_{m=1}^{M}$.
  \end{algorithmic}
\end{algorithm}

\section{Theoretical Guarantees}
\label{sec:theory}

In this section, we present non-asymptotic convergence guarantees for \Cref{algo:FedSOUL} when the family of Markov kernels $\{Q^{(i)}_{\gamma,\theta} : \gamma \in (0,\bar{\gamma}], \theta \in \Theta, i \in [b]\}$ is associated to unadjusted, \emph{i.e.} without Metropolis acceptance step, overdamped Langevin dynamics \citep{Durmus2017,Dalalyan2017}.
The bounds we derive allow to showcase explicitly the impact of FL constraints, namely partial participation and compression.
Results for general unadjusted Markov kernels are postponed to the supplement.

To show our theoretical results and resort to standard assumptions made in the stochastic approximation literature, we consider a minimisation problem and rewrite the opposite of the objective function \eqref{eq:minimization_objective} for any $\theta \in \Theta$ as
\begin{equation}
  \label{eq:def_function_f_fedsouk}
  f(\theta) = b^{-1} \sum_{i=1}^b f_i(\theta)\eqsp, \quad \text{where} \ f_i\pr{\theta} = -\log p(\phi,\beta) - b\log p\pr{\mathrm{D}_i\mid\phi,\beta} \eqsp.
\end{equation}

\noindent \textbf{Non-Asymptotic Convergence Bounds.} For the sake of better readability, we only detail in the main paper assumptions regarding the objective function, compression operators and the partial participation scenario.
Technical assumptions related to the Markov kernels $\{Q^{(i)}_{\gamma,\theta}\}$ are postponed to the supplement.
In spirit, we require, for any $i \in [b], \theta \in \Theta$ and $\gamma$, that $Q^{(i)}_{\gamma,\theta}$ satisfies some ergodic condition and can provide samples sufficiently close to the local posterior distribution $p(z^{(i)} \mid \mathrm{D}_i, \theta)$.
For the sake of simplicity, we also assume that for any $k \in \N^*, \eta_k^{(1)} = \eta^{(2)}_k = \eta_k$, see \Cref{algo:FedSOUL}.

We make the following assumptions on $\Theta$ and the family of functions $\{f_i : i \in [b]\}$.
\begin{assumption}
    \label{ass:convex_set}
  $\Theta$ is convex, closed subset of $\rset^{\dtheta}$ and $\Theta \subset \mathrm{B}(0,R_\Theta)$ for $R_\Theta>0$.
\end{assumption}

\begin{assumption}\label{ass:function_f}
  For any $i\in[b]$, the following conditions hold.

  \begin{enumerate}[wide, labelwidth=!, labelindent=0pt,label=(\roman*),noitemsep,nolistsep]

    \item \label{ass:function_f_1} The function $f_i$ defined in \eqref{eq:def_function_f_fedsouk} is convex.

    \item \label{ass:function_f_2} There exist an open set $\mathsf{U} \in \R^{d_{\Theta}}$ and $L_f > 0$ such that $\Theta \subset \mathsf{U}$, $f_i \in \mathrm{C}^1(\mathsf{U},\R)$ and for any $\theta_1,\theta_2 \in \Theta$,
      $$
      \norm{\nabla f_i(\theta_2) - \nabla f_i(\theta_1)} \le L_f \norm{\theta_2 - \theta_1}\eqsp.
      $$
    \end{enumerate}
\end{assumption}

The assumption below requires compression operators $\{\mathscr{C}_{k}\}_{k\in \N^*}$ to be unbiased and to have a bounded variance.
Such an assumption is for instance verified by stochastic quantisation operators, see \citet{alistarh2017qsgd}.
\begin{assumption}\label{ass:compression_main}
  The compression operators $\{\mathscr{C}_{k}\}_{k\in \N^*}$ are independent and satisfy the following conditions.
  \begin{enumerate}[wide, labelwidth=!, labelindent=0pt,label=(\roman*),noitemsep,nolistsep]
    \item \label{ass:compression_main:unbiased} For any $k \in \N^*$, $v\in\R^d$, $\mathbb{E}[\up_{k}\prn{v}] = v$.
    \item \label{ass:compression_main:variance} There exists $\omega\ge 1$, such that for any $k \in \N^*$, $v\in\R^d$, $\mathbb{E}[\norm{\up_k\prn{v}-v}^{2}] \le \omega\norm{v}^{2}$.
  \end{enumerate}
\end{assumption}

We finally assume that each client has probability $p \in\ocint{0,1}$ to be active at each communication round. 
We would like to point out that this partial participation assumption can be associated to a specific compression operator satisfying \Cref{ass:compression_main}.

\begin{assumption}
    \label{ass:A_k_supp}
  For any $k \in \N^*$, $\mathsf{A}_{k} = \{ i \in [b] \,: \, B_{i,k}= 1 \}$ where for any $i \in [b]$, $\{B_{i,k}\,:  \, k\in\nsets\}$ is a family of \iid~Bernouilli random variables with success probability $p \in\ocint{0,1}$.
\end{assumption}

Under these assumptions, the next result establishes that $(\bar{\theta}_k)_{k \in \N}$ defined by $\bar{\theta}_k = \sum_{j=1}^k \eta_j \theta_j /(\sum_{j=1}^k \eta_j)$ converges towards an element of $\argmin_\Theta f$.

\begin{theorem}
  Assume \Cref{ass:convex_set}-\Cref{ass:A_k_supp} along with \Cref{ass:posterior_densities} detailed in the supplement and let for any $k \in [K]$, $\eta_k \in (0,1/L_f]$.
  Then, for any $K \in \N^*$, we have
    \begin{align}
      \mathbb{E}\br{f(\bar{\theta}_{k}) - f(\theta_\star)} \le \mathbb{E}\br{\frac{\sum_{k=1}^K \eta_{k}\{f(\theta_{k}) - f(\theta_\star)\}}{\sum_{k=1}^K \eta_{k}}}
        &\le A(\gamma) + \frac{E_K}{\sum_{k=1}^K \eta_{k}} \eqsp,
    \end{align}
    where $E_K$ depends linearly on $(\omega/p)\sum_{k=1}^K \eta^2_k$; and $A(\gamma) = C \gamma^\alpha$ with $\alpha > 0$ and $C$ is independent of $\omega, p$ and $(\eta_k)$. 
    Closed-form formulas for these constants are provided in the supplement.
\end{theorem}

An interesting feature of \Cref{algo:FedSOUL} is that convergence towards a minimum of $f$ is possible and the impact of partial participation and compression vanishes when $\lim_{k \rightarrow \infty} \eta_k = 0$.
More precisely, $\limsup_{k \rightarrow \infty} E_K/(\sum_{k=1}^K \eta_k) = 0$ and $\lim_{\gamma \rightarrow 0^+} A(\gamma) = 0$ which shows that we can tend towards a minimum of $f$ with arbitrary precision $\epsilon > 0$ by setting the step-size $\gamma$ to a small enough value.

\section{Related Works}
\label{sec:related_work}

As pointed out in \Cref{sec:intro}, many different approaches have been proposed to address personalisation and uncertainty quantification under the federated learning paradigm.
This section reviews the main related existing lines of research and shows that the proposed methodology provides many benefits; see \Cref{table:overview}.
Interestingly, we also show that \texttt{FedPop} encompasses some of the existing FL models.

\noindent \textbf{Bayesian FL.} One of our main motivations is the possibility to perform grounded uncertainty quantification in FL by resorting to the Bayesian paradigm.
In the recent years, many works have suggested to adapt serial workhorses stochastic simulation approaches such as MCMC or variational inference to the FL setting \citep{chen2020fedbe,SimeoneLiu21_bis,SimeoneLiu21,VonoQLSD,2020_conducive_gradients,Corinzia19,ThangBui18,2021_ICML_DGLMC,SGLD-FedAvg}.
Although some of these approaches address important FL challenges such as the communication bottleneck, partial participation or limited computational device resources, they are not suitable for uncertainty quantification in the cross-device FL scenario.
Indeed, all these approaches assume that the posterior distribution targeted by each client is parametrised by a single potentially high-dimensional parameter of size $d_{\Phi} + d$, see \eqref{eq:marginal_likelihood}. 
This prevents  a sufficient number of samples from being stored  locally to perform uncertainty quantification and Bayesian model averaging, especially when the model is a large neural network.
In contrast, our approach decouples this unique high-dimensional parameter into a fixed part $\phi$ and a low-dimensional random part $z^{(i)}$, significantly reducing the memory footprint of  local sample storage. 

In addition, Bayesian FL methods aim at sampling a random parameter from a target probability distribution where $\pi(\theta) \propto e^{-f(\theta)}$ where $f = (1/b)\sum_{i=1}^b f_i$ denotes the negative log-likelihood associated to the $i$-th client. On the other hand the proposed methodology  considers a mixed-effects modeling approach where parameters are divided into two categories: a fixed component  and a random one for each client. As such, the mixed-effects approach is in essence an empirical Bayesian/marginal likelihood approach \citep{CasellaEB,robbins1992empirical}. It corresponds to a hierarchical model that aims to combine the modeling flexibility and uncertainty assessment of Bayesian inference with computational pragmatism. More precisely, a part of the parameters (fixed-effects ) are estimated via marginal likelihood maximisation and the rest (random effects) using common Bayesian techniques, which are in most cases low dimensional. As a result, up to our knowledge, the model and approach that we propose is novel in FL and comes with many benefits as shown in \Cref{table:overview}.

\begin{table}[t]
\caption{Overview of the main existing personalised FL (top rows) and Bayesian FL (bottom rows) approaches related to the proposed framework. 
Column ``PP'' refers to partial participation, ``perso.'' to personalised approaches, ``bounds'' to available convergence guarantees, ``UQ'' to available uncertainty quantification, ``com.'' to the scheme (multiple local steps and/or compression) used to address the communication bottleneck and ``memory'' to the client memory footprint where $M$ stands for the number of samples.}
\label{table:overview}
\vskip 0.15in
\begin{center}
{\small
\begin{sc}
\begin{tabular}{lcccccc|c}
\toprule
method & PP & perso. & bounds & UQ & com. & memory & \small \texttt{FedPop} instance \\
\midrule
\texttt{Per-FedAvg} & \greencheck & \greencheck & \greencheck & \redcross & \tiny local steps & $d + d_{\Phi}$ & \redcross\\
\texttt{pFedMe} & \redcross & \greencheck & \greencheck & \redcross & \tiny local steps & $d + d_{\Phi}$ & \redcross\\
\texttt{FedRep} & \greencheck & \greencheck & \greencheck & \redcross & \tiny local steps & $d + d_{\Phi}$ & \greencheck\\
\texttt{DITTO} & \greencheck & \greencheck & \greencheck & \redcross & \tiny local steps & $d + d_{\Phi}$ & \redcross\\
\texttt{LG-FedAvg} & \greencheck & \greencheck & \greencheck & \redcross & \tiny local steps & $d + d_{\Phi}$ & \redcross\\
\midrule
\texttt{QLSD} & \greencheck & \redcross & \greencheck & \greencheck & \tiny compression & $M(d + d_{\Phi})$ & \redcross \\
\texttt{FSGLD} & \redcross & \redcross & \greencheck & \greencheck & \tiny local steps & $M(d + d_{\Phi})$ & \redcross\\
\texttt{FedBe} & \greencheck & \redcross & \redcross & \greencheck & \tiny local steps & $M(d + d_{\Phi})$ & \redcross \\
\texttt{DG-LMC} & \redcross &  \redcross & \greencheck & \greencheck & \tiny local steps & $M(d + d_{\Phi})$ & \greencheck\\
\midrule
\textcolor{purple}{\texttt{FedPop}} & \greencheck &  \greencheck & \greencheck & \greencheck & \tiny both & $Md + d_{\Phi}$ & --\\
\bottomrule
\end{tabular}
\end{sc}
}
\end{center}
\vskip -0.1in
\end{table}

\noindent \textbf{Personalised FL.} Beside uncertainty quantification, we also aim at providing each client with a dedicated personalised model.
Among the numerous  existing personalised FL approaches, those related to \texttt{FedPop} can be broadly classified  into two groups: \emph{meta-learning} and \emph{partially local methods}.
Meta-learning based FL methods aim at training a global model conducive to fast training of personalised models. Such a goal can be achieved, for example, by local fine-tuning \citep{NEURIPS2020_24389bfe}, regularisation of local models towards their average \citep{hanzely2020federated,Hanzely2021Personalized} -- or the opposite \citep{Ditto}, and model interpolation \citep{liang2020think}.
On the other hand, FL methods based on partial decoupling take an approach similar to ours by splitting the initial model into a backbone component and a local one aimed at personalisation \citep{pmlr-v139-collins21a,Arivazhagan2019,pillutla2022federated}. This partial decoupling could also enhance privacy as discussed in \citet{NEURIPS2021_5d44a2b0}.
The main difference with \texttt{FedPop} is that such approaches based on empirical risk minimisation cannot provide credibility information. 
 
\noindent \textbf{\texttt{FedPop}: A Compromise between Standard and Personalised FL.} Interestingly, we show here that the \texttt{FedPop} framework allows existing FL approaches to be retrieved  in certain regimes.
To this end, we assume that the prior $p(z^{(i)} \mid \beta)$ is Gaussian with mean $\mu$ and covariance matrix $\sigma^2 \mathrm{I}_d$ so that $\beta = \{\mu,\sigma\}$.
If $\sigma \rightarrow 0^+$, then this Gaussian prior tends towards the Dirac distribution centered at $\mu$ and the local likelihood becomes $p(\mathrm{D}_i \mid \phi,\mu)$, which corresponds to the local objective of standard FL approaches such as \texttt{FedAvg} \citep{mcmahan17}.
On the other hand, when $\sigma \rightarrow \infty$, no common information $\beta$ is used to locally regress $z^{(i)}$ and we end up with the \texttt{FedRep} algorithm \citep{pmlr-v139-collins21a}.
This shows that \texttt{FedPop} stands for a subtle compromise between standard and personalised FL which should  benefit clients with small data sets by pooling information via a common prior.
Finally, in the extreme scenario where $\phi$ is the null vector, our approach amounts to the Bayesian FL approach \texttt{DG-LMC} proposed  in \citet{2021_ICML_DGLMC}.

\section{Numerical Experiments}
\label{sec:experiments}

In this section, we illustrate the benefits of our methodology on several FL benchmarks associated to both synthetic and real data.
Since existing Bayesian FL approaches are not suited for personalisation (see \Cref{table:overview}), we only compare the performances of \Cref{algo:FedSOUL} with personalised FL methods.
In all our experiments, we use overdamped Langevin dynamics to sample locally and call this specific instance of \Cref{algo:FedSOUL}, \texttt{FedSOUL}. In addition, we set $p(z^{(i)}\mid \beta) = \mathrm{N}(\mu,\sigma^2\mathrm{I}_d)$ with $\beta = \{\mu,\sigma\}$ for simplicity.
To be comparable with existing personalised FL approaches that only consider periodic communication via multiple local steps, we do not resort to the proposed compression mechanism although the latter could be of interest for real-world applications.
Additional experiments and details about experimental design are provided in the supplement.

\noindent \textbf{Synthetic Data.} We start by showcasing the benefits of \texttt{FedSOUL} for clients having small and highly heterogeneous data sets as pointed out in \Cref{sec:intro} and \Cref{sec:fed_pop}. 
To this end, we consider a similar experimental setting as in \citet{pmlr-v139-collins21a} where synthetic observations $\{y^{(i)}_j\}_{j \in [N_i]} \in \mathrm{D}_i$ are generated via the following procedure: $x^{(i)}_j \sim \mathrm{N}(0_k,\mathrm{I}_k)$ and $y^{(i)}_j \sim \mathrm{N}(z^{(i)}_{\mathrm{true}} \phi_{\mathrm{true}}^\top x^{(i)}_j,0.1)$. The ground-truth parameters $z^{(i)}_{\mathrm{true}} \in \R^{d}$ and $\phi_{\mathrm{true}} \in \R^{k \times d}$ have been randomly generated beforehand with $(d,k) = (2,20)$.
Compared to \citet{pmlr-v139-collins21a}, we use heterogeneous data partitions across clients so that 90\% of the $b=100$ clients have small data sets of size 5 and the remaining 10\% have data sets of size 10.
We compare our results with \texttt{FedRep} \citep{pmlr-v139-collins21a} and \texttt{FedAvg} \citep{mcmahan17} since they stand for two limiting instances of the proposed methodology, see \Cref{sec:related_work} and \citet[Section 12]{GelmanHill:2007}.
\begin{wrapfigure}{r}{0.6\textwidth}
\vspace{-0.5cm}
  \begin{center}
    \includegraphics[scale=0.23]{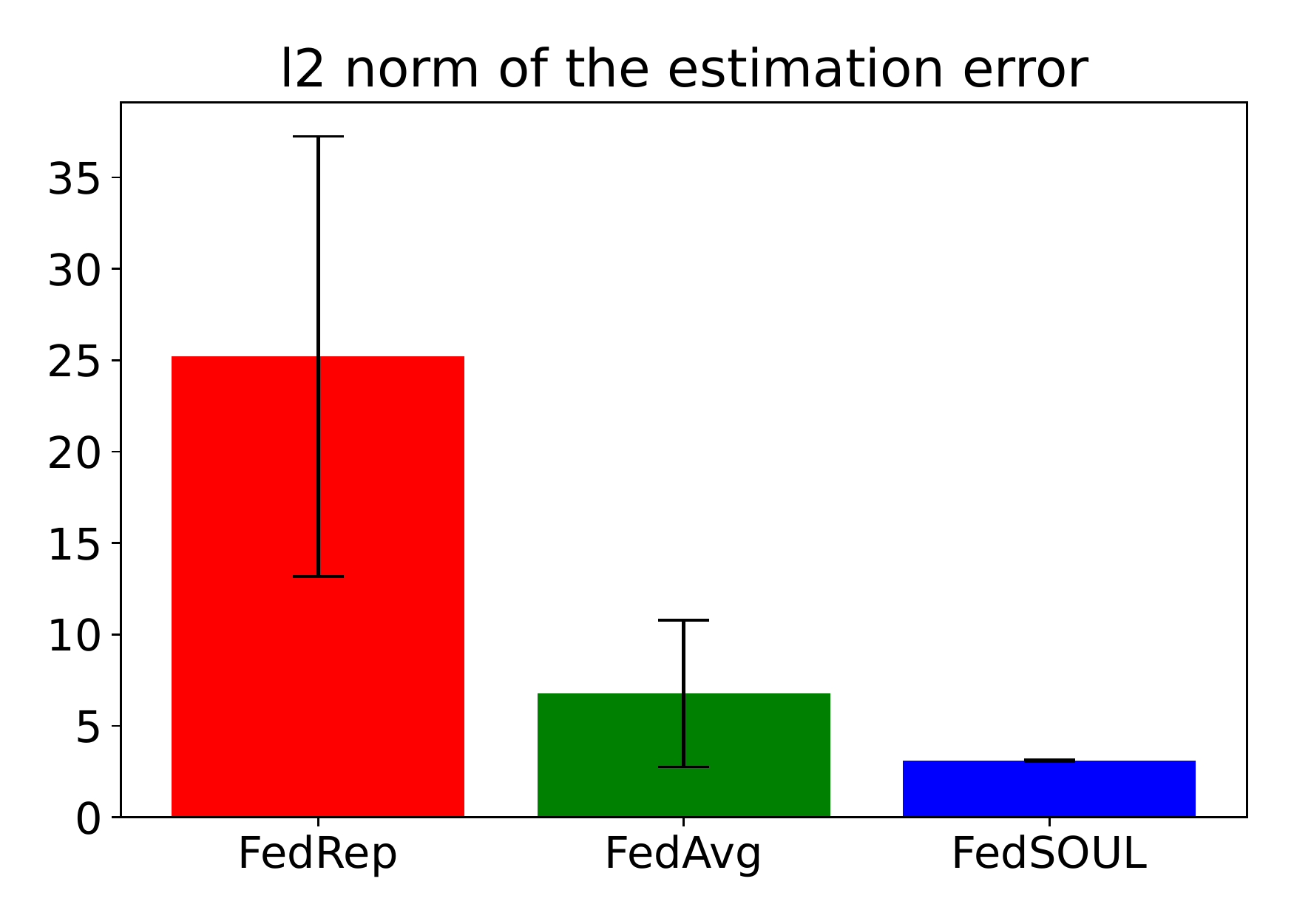}
    \includegraphics[scale=0.23]{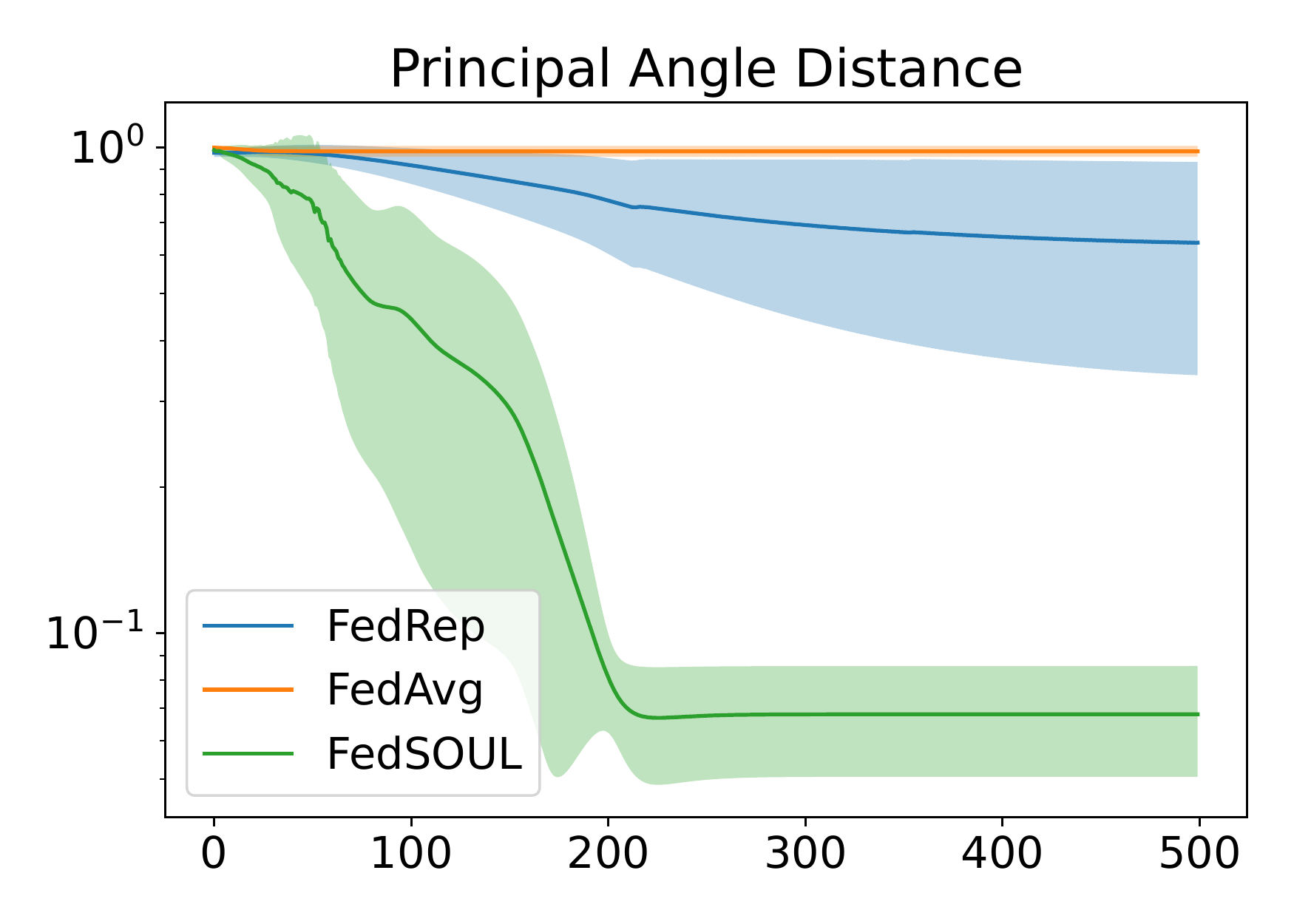}
  \end{center}
  \caption{Small data sets - synthetic data.}
  \label{fig:toy}
  \vspace{-0.2cm}
\end{wrapfigure}
\Cref{fig:toy} compares the different approaches by computing the principle angle distance\footnote{defined in \citep[Definition 1]{pmlr-v139-collins21a}} (respectively the $\ell_2$ norm) between $\phi_{\mathrm{true}}$ (respectively $z^{(i)}_{\mathrm{true}}$) and its estimated value; the lesser the better.
In contrast to its main competitors and based on both metrics, \texttt{FedSOUL} provides an impressive improvement. 
This illustrates the benefits of the introduction of a common prior $p(z^{(i)}\mid \beta)$ which allows to prevent from overfitting on clients with small data sets while performing personalisation. Additional results with other choices for $(b,d,k)$ and data partitioning strategies are available in the supplement.

Moreover, to compare our algorithm with a non-FL setting, we perform a non-distributed and non-federated stochastic approximation algorithm to find $\theta^\ast$ using a large number of iterations to get an accurate approximation of the optimal parameter $\theta^\ast$. Then, we use \texttt{FedPop} to obtain an estimate $\tilde{\theta^\ast}$ and measure the relative error in $l_2$- distance between $\theta^\ast$ and $\tilde{\theta^\ast}$. For some outer iterations $T=100$, the relative error was less than $10^{-3}$, which illustrates the relevance of our theoretical results.
We also test the performances of the proposed approach when the warm-start strategy is not used. In this case, we have to set $M=50$ to achieve the same performances as in the stateful variant of \texttt{FedSOUL}.

\noindent \textbf{Real Data.} We consider now real image data sets, namely CIFAR-10 and CIFAR-100 \citep{CIFAR10}. 
For our likelihood model defined by $p(\mathrm{D}_i\mid \phi,z^{(i)})$, we use 5-layer convolutional neural networks and perform personalisation for the last layer.
We set $b = 100$ for convenience and control data heterogeneity by assigning to each client $N_i$ images belonging to only $S$ different classes.

\begin{wrapfigure}{r}{0.6\textwidth}
  \vspace{-0.5cm}
  \begin{center}
    \includegraphics[scale=0.23]{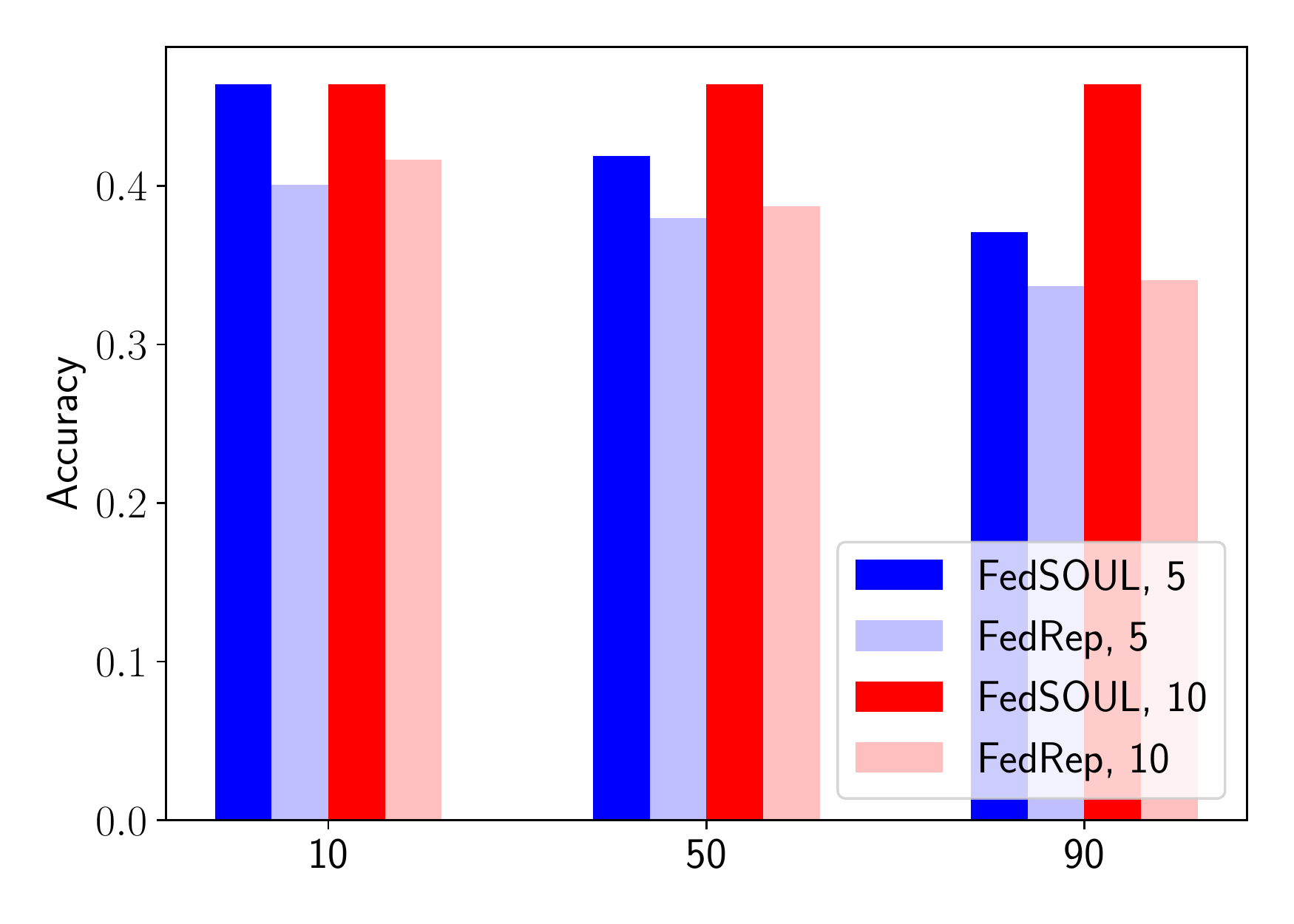}
    \includegraphics[scale=0.23]{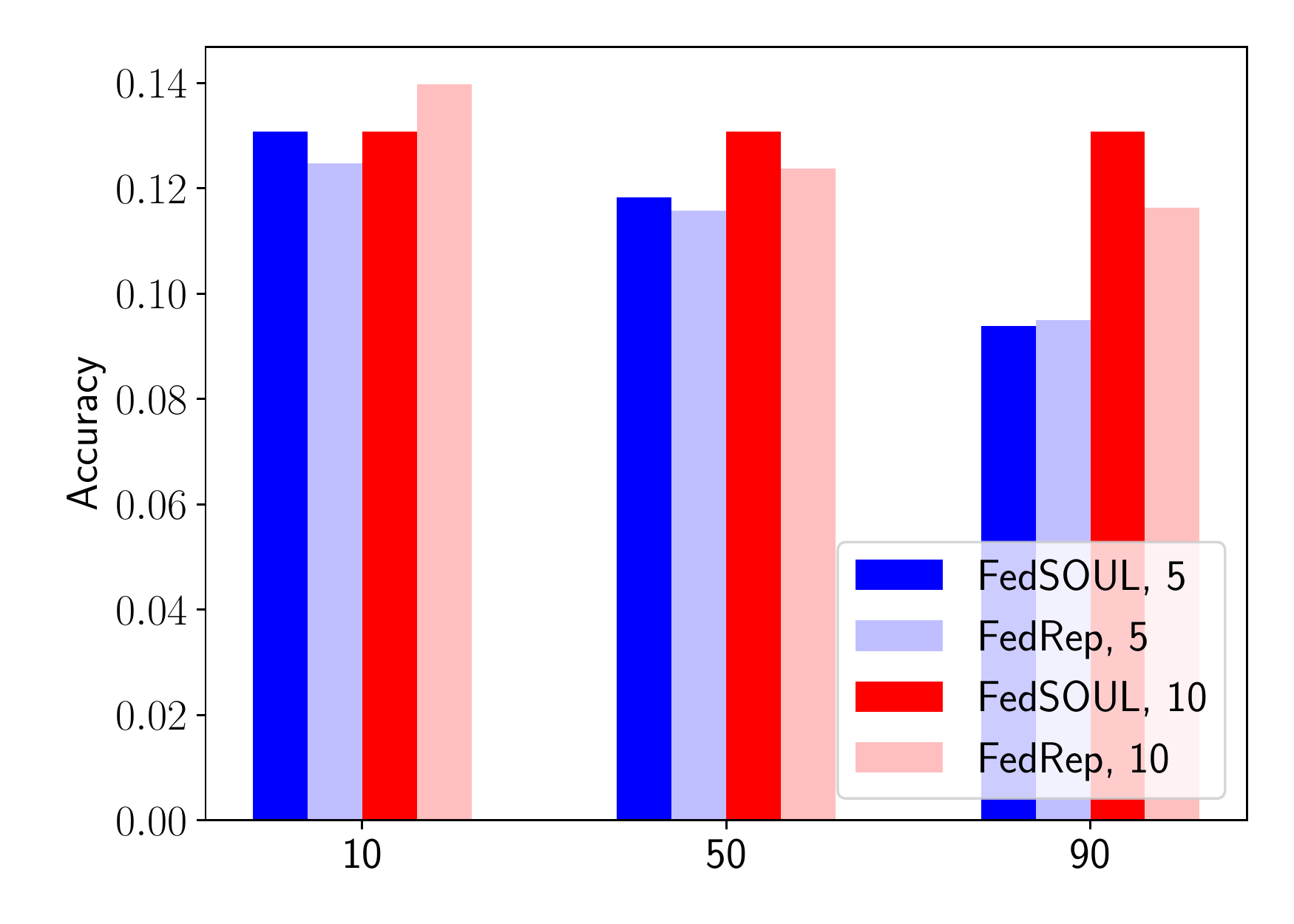}
  \end{center}
  \caption{(right) CIFAR-10 with $S=5$ and (left) CIFAR-100 with $S=20$. The $x$-axis refers to the percentage of clients having $N_i \in \{5,10\}$ images.}
  \label{fig:small_data}
  \vspace{-0.4cm}
\end{wrapfigure}
\noindent \textit{Small data sets.} Under this setting, we first consider (10\%, 50\%, 90\%) of clients having small data sets of size either $N_i = 5$ or $N_i=10$; while remaining clients have larger data sets of size $N_i=25$.
We compare our approach with \texttt{FedRep} since it stands for the state-of-the-art personalised FL approach. 
The algorithms are trained fulfilling the same computational budget.
\Cref{fig:small_data} shows the average accuracy across clients for the two approaches on both CIFAR-10 and CIFAR-100.
We can see that \texttt{FedSOUL} is consistently better than \texttt{FedRep} over different configurations.

\noindent \textit{Full data sets.} In addition to show that the proposed approach achieves state-of-the-art performances on small data sets (which is common in the cross-device scenario), we now illustrate that \texttt{FedSOUL} is also competitive on larger data sets. 
To this end, we use all training images in CIFAR-10 and CIFAR-100 image data sets and consider the same data partitioning as in \citet{pmlr-v139-collins21a}.
More precisely, in this case the number of observations and the number of classes per client are uniformly shared over the clients.
\Cref{tab:real_images_benchmark} shows our results in comparison with state-of-the-art personalised FL approaches.
We can see that that our model outperforms other methods on both CIFAR-10 and CIFAR-100 by a large margin.
Additional results with other personalised FL algorithms are postponed to the supplement.

\noindent \textbf{Uncertainty Quantification on Real Data.} As highlighted in \Cref{table:overview}, one advantage of the proposed approach compared to existing personalised FL methods is the ability to perform uncertainty quantification by sampling locally from the posterior $p(z^{(i)}\mid \mathrm{D}_i, \phi_K,\beta_K)$, see \Cref{algo:FedSOUL}.
We illustrate this feature by computing on CIFAR-10 calibration curves and scores (\emph{e.g.} expected calibration error aka ECE) on a specific client; and by performing an out-of-distribution analysis on MNIST/FashionMNIST data sets.
\Cref{fig:UQ} shows that the proposed approach provides relevant uncertainty diagnosis.
Additional results on uncertainty quantification can be found in the supplement.
\begin{wrapfigure}{r}{0.57\textwidth}
  \vspace{-0.5cm}
  \begin{center}
    \includegraphics[scale=0.22]{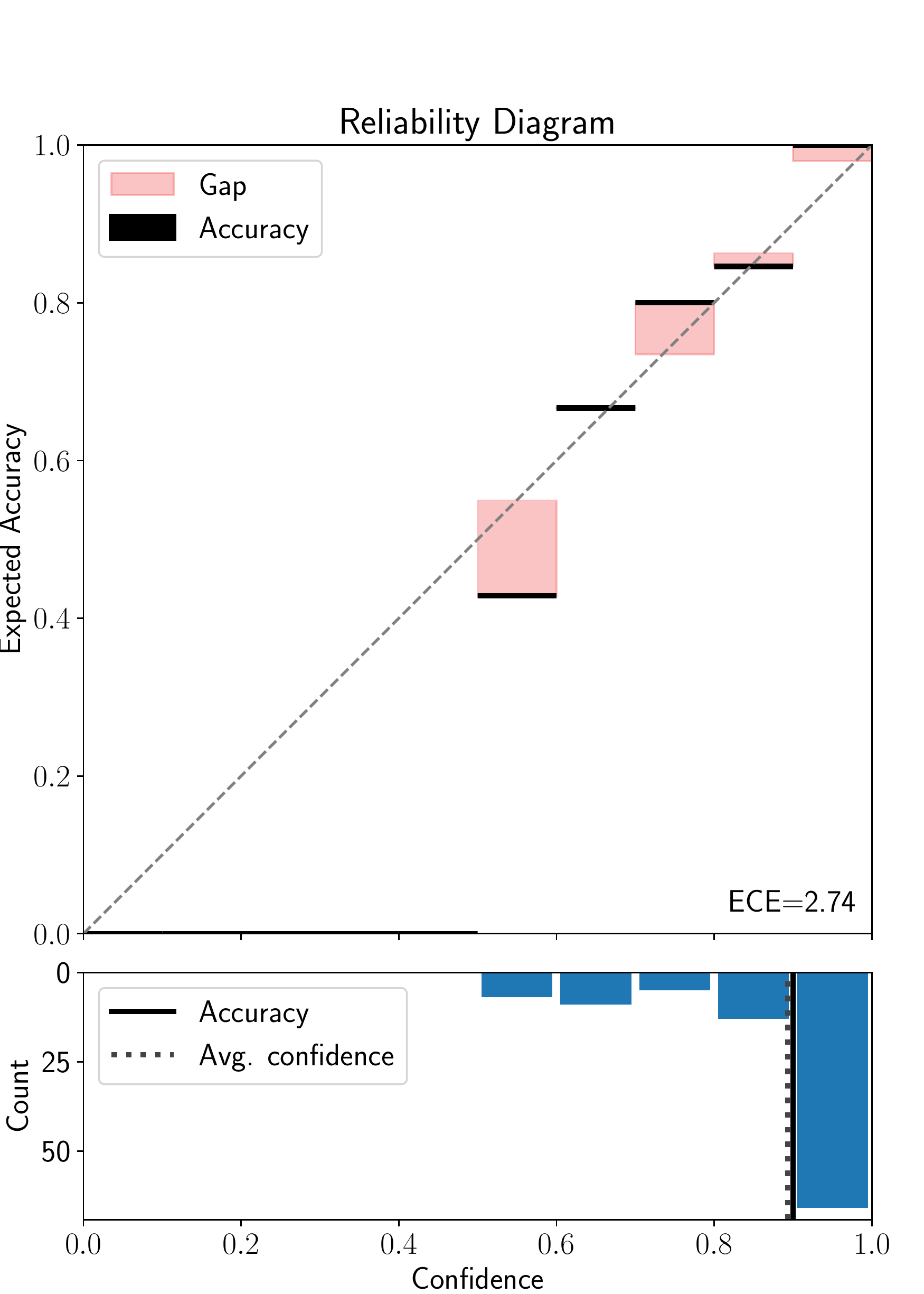}
    \includegraphics[scale=0.3]{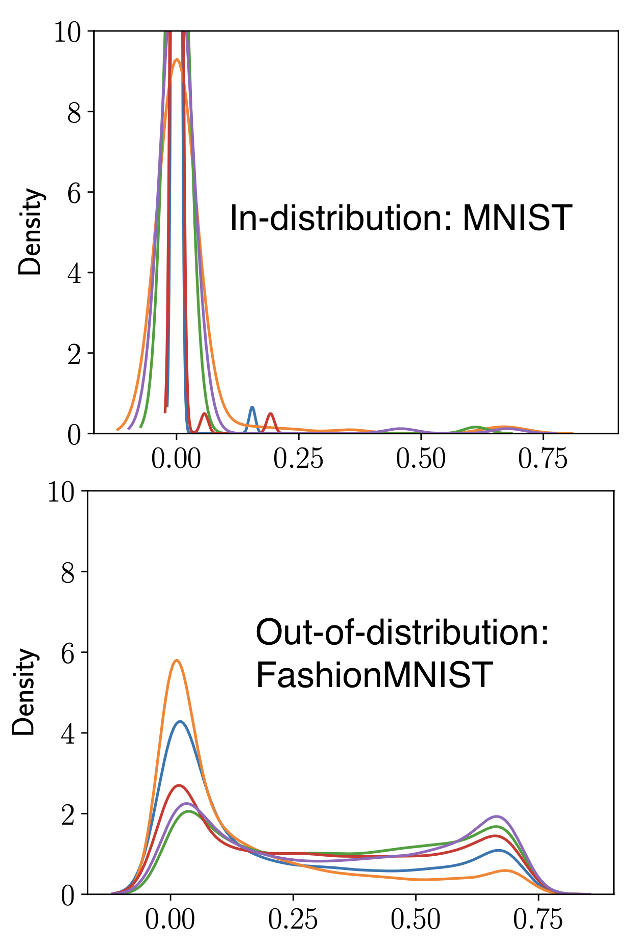}
  \end{center}
  \caption{(right) Calibration on CIFAR-10 for a specific client and (left) OOD analysis with MNIST training \& FashionMNIST inference -- one curve corresponds to one client.}
  \label{fig:UQ}
  \vspace{-0.4cm}
\end{wrapfigure}

\begin{table}[]
\caption{Real data - Full data sets. Accuracy (in \%) on test samples. \texttt{FedAvg} and \texttt{SCAFFOLD} are not personalised FL approaches but stand for well-known FL benchmarks.}
\begin{center}
\begin{tabular}{lllll}
\toprule
& \multicolumn{2}{c}{CIFAR-10} & \multicolumn{2}{c}{CIFAR-100} \\
\midrule
\multicolumn{1}{l}{(\# clients $b$, \# classes per client $S$)} & \multicolumn{1}{c}{(100, 2)} & \multicolumn{1}{c}{(100, 5)} & \multicolumn{1}{c}{(100, 5)} & \multicolumn{1}{c}{(100, 20)} \\
\midrule
\multicolumn{1}{l}{Local learning only} & \multicolumn{1}{l}{89.79}    & \multicolumn{1}{l}{70.68}    & \multicolumn{1}{l}{75.29}    & \multicolumn{1}{l}{41.29} \\
\midrule
\multicolumn{1}{l}{\texttt{FedAvg} \citep{mcmahan17}} & \multicolumn{1}{l}{42.65}    & \multicolumn{1}{l}{51.78}    & \multicolumn{1}{l}{23.94}    & \multicolumn{1}{l}{ 31.97} \\
\multicolumn{1}{l}{\texttt{SCAFFOLD} \citep{karimireddy2020scaffold}} & \multicolumn{1}{l}{37.72}    & \multicolumn{1}{l}{47.33}    & \multicolumn{1}{l}{20.32}    & \multicolumn{1}{l}{ 22.52} \\
\midrule
\multicolumn{1}{l}{\texttt{LG-FedAvg} \citep{liang2020think}}                        & \multicolumn{1}{l}{84.14}    & \multicolumn{1}{l}{63.02}    & \multicolumn{1}{l}{72.44}    & \multicolumn{1}{l}{38.76}     \\
\multicolumn{1}{l}{\texttt{Per-FedAvg} \citep{NEURIPS2020_24389bfe}}                        & \multicolumn{1}{l}{82.27}    & \multicolumn{1}{l}{67.20}    & \multicolumn{1}{l}{72.05}    & \multicolumn{1}{l}{52.49}     \\
\multicolumn{1}{l}{\texttt{L2GD} \citep{hanzely2020federated}}                        & \multicolumn{1}{l}{81.04}    & \multicolumn{1}{l}{ 59.98}    & \multicolumn{1}{l}{72.13}    & \multicolumn{1}{l}{ 42.84}     \\
\multicolumn{1}{l}{\texttt{APFL} \citep{deng2021adaptive}}                        & \multicolumn{1}{l}{83.77}    & \multicolumn{1}{l}{72.29}    & \multicolumn{1}{l}{78.20}    & \multicolumn{1}{l}{55.44}     \\
\multicolumn{1}{l}{\texttt{DITTO} \citep{Ditto}}                                 & \multicolumn{1}{l}{85.39}    & \multicolumn{1}{l}{70.34}    & \multicolumn{1}{l}{78.91}    & \multicolumn{1}{l}{56.34}     \\
\multicolumn{1}{l}{\texttt{FedRep} \citep{pmlr-v139-collins21a}}                 & \multicolumn{1}{l}{87.70}    & \multicolumn{1}{l}{75.68}    & \multicolumn{1}{l}{79.15}    & \multicolumn{1}{l}{56.10}     \\
\multicolumn{1}{l}{\texttt{FedAvg + fine-tuning (FT)}}                 & \multicolumn{1}{l}{85.63}    & \multicolumn{1}{l}{71.32}    & \multicolumn{1}{l}{79.03}    & \multicolumn{1}{l}{56.19}     \\
\midrule
\multicolumn{1}{l}{\textcolor{purple}{\texttt{FedSOUL} (this paper)}}                                 & \multicolumn{1}{l}{\textcolor{purple}{91.12}}     & \multicolumn{1}{l}{\textcolor{purple}{79.48}}     & \multicolumn{1}{l}{\textcolor{purple}{79.56}}     & \multicolumn{1}{l}{\textcolor{purple}{59.73}}     \\
\bottomrule
\end{tabular}
\label{tab:real_images_benchmark}
\end{center}
\vspace{-0.4cm}
\end{table}

\vspace{-0.3cm}
\section{Conclusion}
\vspace{-0.2cm}

In this paper, we proposed a general Bayesian methodology based on a natural mixed-effects modeling approach to model personalisation in federated learning.
Our FL method is the first that allows for both personalisation and cheap uncertainty quantification for (cross-device) federated learning.
By introducing a common prior on the local parameters, we tackle the local overfitting problem in the scenario where clients have highly heterogeneous and small data sets.
In addition, we have shown that the proposed approach has favorable convergence properties.
Some limitations of \texttt{FedPop} pave the way for more advanced personalised FL approaches.
As an example, our model does not allow for training heterogeneous architectures across clients because of the introduced common prior, and only satisfy first-order privacy guarantees.
Regarding the latter, further works include for instance deriving differentially private versions of our framework.

\begin{ack}
The authors acknowledge the Lagrange Mathematics and Computing Research Center for supporting the project. The development of the algorithm and conducting experiments (Section ~\ref{sec:experiments}) was supported by Russian Science Foundation grant 20-71-10135.
\end{ack}

\newpage

\bibliography{bibliography/biblio}
\bibliographystyle{plainnat}

\section*{Checklist}

The checklist follows the references.  Please
read the checklist guidelines carefully for information on how to answer these
questions.  For each question, change the default \answerTODO{} to \answerYes{},
\answerNo{}, or \answerNA{}.  You are strongly encouraged to include a {\bf
justification to your answer}, either by referencing the appropriate section of
your paper or providing a brief inline description.  For example:
\begin{itemize}
  \item Did you include the license to the code and datasets? \answerYes{See Section~\ref{gen_inst}.}
  \item Did you include the license to the code and datasets? \answerNo{The code and the data are proprietary.}
  \item Did you include the license to the code and datasets? \answerNA{}
\end{itemize}
Please do not modify the questions and only use the provided macros for your
answers.  Note that the Checklist section does not count towards the page
limit.  In your paper, please delete this instructions block and only keep the
Checklist section heading above along with the questions/answers below.

\begin{enumerate}

\item For all authors...
\begin{enumerate}
  \item Do the main claims made in the abstract and introduction accurately reflect the paper's contributions and scope?
    \answerYes{}
  \item Did you describe the limitations of your work?
    \answerYes{} We describe these limitations in the conclusion.
  \item Did you discuss any potential negative societal impacts of your work?
    \answerNA{}
  \item Have you read the ethics review guidelines and ensured that your paper conforms to them?
    \answerYes{}
\end{enumerate}

\item If you are including theoretical results...
\begin{enumerate}
  \item Did you state the full set of assumptions of all theoretical results?
    \answerYes{} Some assumptions are explicitly stated in \Cref{sec:theory}. For the sake of readability and due to the space limit, we postponed other technical assumptions in the supplement.
        \item Did you include complete proofs of all theoretical results?
    \answerYes{} All our proofs are postponed to the supplement.
\end{enumerate}

\item If you ran experiments...
\begin{enumerate}
  \item Did you include the code, data, and instructions needed to reproduce the main experimental results (either in the supplemental material or as a URL)?
    \answerYes{} See the supplement.
  \item Did you specify all the training details (e.g., data splits, hyperparameters, how they were chosen)?
    \answerYes{} Some details have been stated in \Cref{sec:experiments} in the main paper. Because of the limited number of pages, others have been postponed to the supplement.
        \item Did you report error bars (e.g., with respect to the random seed after running experiments multiple times)?
    \answerYes{} We did it for some experiments we conducted, see \emph{e.g.} Figure 2.
        \item Did you include the total amount of compute and the type of resources used (e.g., type of GPUs, internal cluster, or cloud provider)?
    \answerYes{} This is postponed to the supplement.
\end{enumerate}

\item If you are using existing assets (e.g., code, data, models) or curating/releasing new assets...
\begin{enumerate}
  \item If your work uses existing assets, did you cite the creators?
    \answerYes{} See for example the citation for CIFAR image data sets.
  \item Did you mention the license of the assets?
    \answerNA{}
  \item Did you include any new assets either in the supplemental material or as a URL?
    \answerNA{}
  \item Did you discuss whether and how consent was obtained from people whose data you're using/curating?
    \answerNA{}
  \item Did you discuss whether the data you are using/curating contains personally identifiable information or offensive content?
    \answerNA{}
\end{enumerate}

\item If you used crowdsourcing or conducted research with human subjects...
\begin{enumerate}
  \item Did you include the full text of instructions given to participants and screenshots, if applicable?
    \answerNA{}
  \item Did you describe any potential participant risks, with links to Institutional Review Board (IRB) approvals, if applicable?
    \answerNA{}
  \item Did you include the estimated hourly wage paid to participants and the total amount spent on participant compensation?
    \answerNA{}
\end{enumerate}

\end{enumerate}

\clearpage
\newpage

\appendix

\newtheorem{unlemma}{Lemma S}
\newtheorem{unproposition}{Proposition S}
\newtheorem{uncorollary}{Corollary S}
\newtheorem{untheorem}{Theorem S}

\setcounter{equation}{0}
\setcounter{figure}{0}
\setcounter{table}{0}
\setcounter{assumption}{0}
\makeatletter
\renewcommand{\theequation}{S\arabic{equation}}
\renewcommand{\thefigure}{S\arabic{figure}}
\renewcommand{\thetheorem}{S\arabic{theorem}}
\renewcommand{\thelemma}{S\arabic{lemma}}
\renewcommand{\thetable}{S\arabic{table}}
\renewcommand{\thesection}{S\arabic{section}}
\renewcommand{\theremark}{S\arabic{remark}}
\renewcommand{\theproposition}{S\arabic{proposition}}
\renewcommand{\thecorollary}{S\arabic{corollary}}

{\begin{center}\Large\textbf{SUPPLEMENTARY MATERIAL}\end{center}}\vspace{1cm}

\paragraph{Notations and conventions.}

For the sake of simplicity, with little abuse, we shall use the same notations for
a probability distribution and its associated probability density function.
For $n \ge 1$, we refer to the set of integers between $1$ and $n$ with the notation $[n]$.
The $d$-multidimensional Gaussian probability distribution with mean $\mu \in \Rd$ and covariance matrix $\Sigma \in \mathbb{R}^{d \times d}$ is denoted by $\gauss\parentheseLigne{\mu,\Sigma}$.
Equations of the form (1) (resp. (S1)) refer to equations in the main paper (resp. in the supplement).

Denote by $\mathcal{B}(\rset^d)$ the Borel $\sigma$-field of
$\rset^d$, and for $f : \rset^d \to \rset$ measurable,
$\norm{f}_{\infty}= \sup_{x \in \rset^d} \abs{f(x)}$.  For $\mu$ a
probability measure on $(\rset^d, \mathcal{B}(\rset^d))$ and $f$ a
$\mu$-integrable function, denote by $\mu(f)$ the integral of $f$
\wrt~$\mu$. For $f: \ \rset^d \to \rset$ measurable, the $V$-norm of
$f$ is given by $\Vnorm[V]{f}= \sup_{x \in \rset^d} |f(x)|/V(x)$. Let
$\xi$ be a finite signed measure on $(\rset^d,\mcb(\rset^d))$. The
$V$-total variation distance of $\xi$ is defined as
\begin{equation}
\textstyle{\Vnorm[V]{\xi} = \sup_{\Vnorm[V]{f} \leq 1}  \abs{\int_{\rset^d } f(x) \rmd \xi (x)} \eqsp.}
\end{equation}
If $V = 1$, then $\Vnorm[V]{\cdot}$ is the total variation denoted by
$\tvnorm{\cdot}$.  Let $\msu$ be an open set of $\rset^d$. We denote
by $\rmc^{k}(\msu, \rset^p)$ the set of $\rset^p$-valued
$k$-differentiable functions, respectively the set of compactly
supported $\rset^p$-valued and $k$-differentiable functions.  Let
$f : \msu \to \rset$, we denote by $\nabla f$, the gradient of $f$ if
it exists. $f$ is said to me $\mtt$-convex with $m\geq 0$ if for all
$x,y \in \rset^d$ and $t \in \ccint{0,1}$,
\begin{equation}
  f(t x + (1-t) y) \leq t f(x)  + (1-t) f(y) -\mtt t(1-t)  \norm{x-y}^2/2  \eqsp.
\end{equation}
For any $a \in \rset^d$ and $R > 0$, denote $\ball{a}{R}$ the open 
ball centered at $a$ with radius $R$.  
Let $(\msx, \mcx)$ and $(\msy, \mcy)$ be two measurable
  spaces. A Markov kernel $\Pker$ is a mapping
  $\Kker: \ \msx \times \mcy \to \ccint{0, 1}$ such that for any
  $x \in \msx$, $\Pker(x, \cdot)$ is a probability measure and for any
  $\msa \in \mcy$, $\Pker(\cdot, \msa)$ is measurable. For any
  probability measure $\mu$ on $(\msx, \mcx)$ and measurable function
  $f : \msy \to \rset_+$ we denote
  $\mu \Pker = \int_{\msx} \Pker(x, \cdot) \rmd \mu(x)$ and
  $\Pker f = \int_{\msy} f(y) \Pker(\cdot , \rmd y)$.  In what follows
  the Dirac mass at $x \in \rset^{d}$ by $\updelta_x$.


\tableofcontents

\section{Theoretical analysis of \texttt{FedSOUK}}

This section aims at recasting the proposed methodology into a stochastic approximation framework and at stating the main assumptions required to show our theoretical results regarding \texttt{FedSOUK}, which uses a general unadjusted Markov kernel. 
Then, we will use these general results to show non-asymptotic convergence guarantees for \texttt{FedSOUL}, which considers an unadjusted Markov kernel associated to overdamped Langevin dynamics.

\subsection{Preliminaries}

We first show that \texttt{FedSOUK} (see \Cref{algo:FedSOUL} in the main paper) can be cast into a general \emph{stochastic approximation} (SA) framework which corresponds to a federated variant of the \emph{stochastic optimization via unadjusted kernel} (\texttt{SOUK}) approach proposed in \citet{SOUL}.
Then, the convergence guarantees for \texttt{FedSOUK} will follow by generalizing the proof techniques used to analyze \texttt{SOUK}.

Recall that $\theta = (\phi,\beta) \in \Theta$ corresponds to the parameter we are seeking to optimize where $\Theta = \Phi \times \mathsf{B} \subset \R^{d_{\Theta}}$.
Define $f : \Theta \to \rset$ of the form 
\begin{equation}
  \label{eq:def_function_f_fedsouk}
  f(\theta) = b^{-1} \sum_{i=1}^b f_i(\theta) \eqsp,
\end{equation}
where for any $i \in [b]$ and $\theta \in \Theta$, 
\begin{equation}
  \label{eq:def_function_f}
  f_i\pr{\theta} = -\log p(\theta) - b\log p\pr{\mathrm{D}_i\mid\phi,\beta}\eqsp,
\end{equation}
where $p(\theta) = p(\phi,\beta) = p(\phi)p(\beta)$ and for any $i\in[b]$, $p(\mathrm{D}_i\mid\phi,\beta)$ is defined in \eqref{eq:marginal_likelihood}.
Then, under these notations, \eqref{eq:minimization_pb} can be written as
\begin{equation}
  \label{eq:FL_problem}
  \thetas = \argmin_{\theta \in \Theta} f(\theta)\eqsp.
\end{equation}
In addition, based on \eqref{eq:estimator_beta} and \eqref{eq:estimator_phi}, the gradient of $f_i$ defined in \eqref{eq:def_function_f} admits the form for $i \in [b]$,
\begin{equation}
  \nabla f_i:
    \begin{cases}
      \R^{d_\Phi+d_{\mathsf{B}}} \to \R^{d_{\Theta}} \\
      \theta \mapsto \int_{\Rd} H^{(i)}_\theta\pr{z^{(i)}} \pi_{\theta}^{(i)}\pr{\dd z^{(i)}}\eqsp, \label{eq:grad_f}
    \end{cases}
\end{equation}
where, for any $i \in [b]$ and $\theta \in \Theta$, $\pi_\theta^{(i)}: z^{(i)} \mapsto p(z^{(i)}\mid \mathrm{D}_i, \theta)$ and for any $\theta \in \Theta$, $H_\theta^{(i)}: z^{(i)} \mapsto -\nabla_{\theta} \log p(\theta)  - b\nabla_{\theta} \log p(\mathrm{D}_i,z^{(i)} \mid \theta)$.

\subsection{Main Assumptions}

We make the following assumption on $\Theta$ and the family of functions $\{f_i : i \in [b]\}$.
\begin{assumptionsup}
    \label{ass:convex_set}
  $\Theta$ is a convex, closed subset of $\rset^{\dtheta}$ and $\Theta \subset \mathrm{B}(0,R_\Theta)$ for $R_\Theta>0$.
\end{assumptionsup}

\begin{assumptionsup}\label{ass:function_f}
  For any $i\in[b]$, the following conditions hold.

  \begin{enumerate}[wide, labelwidth=!, labelindent=0pt,label=(\roman*),noitemsep,nolistsep]

    \item \label{ass:function_f_1} The function $f_i$ defined in \eqref{eq:def_function_f_fedsouk} is convex.

    \item \label{ass:function_f_2} There exist an open set $\mathsf{U} \in \R^{d_{\Theta}}$ and $L_f > 0$ such that $\Theta \subset \mathsf{U}$, $f_i \in \mathrm{C}^1(\mathsf{U},\R)$ and for any $\theta_1,\theta_2 \in \Theta$,
      $$
      \norm{\nabla f_i(\theta_2) - \nabla f_i(\theta_1)} \le L_f \norm{\theta_2 - \theta_1}\eqsp.
      $$
    \end{enumerate}
\end{assumptionsup}

Note that \Cref{ass:function_f}-\ref{ass:function_f_2} implies that the objective function $f$ defined in \eqref{eq:def_function_f_fedsouk} is gradient-Lipschitz with Lipschitz constant $L_f$.

We now consider assumptions on the family of \emph{compression} and \emph{partial participation} operators $\{\mathscr{C}_i,\mathscr{S}_i\}_{i \in [b]}$.

\begin{assumptionsup}\label{ass:compression}
  There exists a probability measure $\nu_{1}$ on a measurable space $(\msx_{1},\mathcal{X}_{1})$ and a family of measurable functions $\{\up_i:\R^{d_{\Phi}}\times\msx_1\to \R^{d_{\Phi}}\}_{i\in [b]}$ such that the following conditions hold.   
  \begin{enumerate}[wide, labelwidth=!, labelindent=0pt,label=(\roman*),noitemsep,nolistsep]
    \item \label{ass:compression:unbiased} For any $v\in\R^{d_\Phi}$ and any $i \in [b]$, $\int_{\msx_{1}}\up_i\parentheseLigne{v,x^{(1)}}\,\nu_{1}\parentheseLigne{\rmd x^{(1)}} = v$.
    \item \label{ass:compression:variance} There exist $\{\omega_i\in\R_+\}_{i\in[b]}$, such that for any $v\in\R^{d_\Phi}$ and any $i \in [b]$,    
    \begin{equation}\int_{\msx_{1}}\norm{\up_i\parentheseLigne{v,x^{\parentheseLigne{1}}}-v}^{2}\,\nu_{1}\parentheseLigne{\rmd x^{\parentheseLigne{1}}} \le \omega_i\norm{v}^{2}\eqsp.\end{equation}
    \end{enumerate}
  \end{assumptionsup}

  In addition, recall that we consider the partial device participation context where at each communication round $k\ge1$, each client has a probability $p_i \in (0,1]$ of participating, independently from other clients.

  \begin{assumptionsup}
    \label{ass:A_k_supp2}
    For any $i \in [b]$, the unbiased partial participation operator $\mathscr{S}_i: \Rdtheta \times \msx_2 \to \Rdtheta$ is defined, for any $\theta \in \Rdtheta$ and $x^{(2)}=\{x^{(2)}_i\}_{i\in[b]} \in \msx_2$ with $\msx_2 = [0,1]^b$ by 
  \begin{equation}
    \mathscr{S}_i(\theta,x^{(2)}) = \mathbf{1}\{x^{(2)}_i \le p_i\} \theta / p_i\eqsp, \label{def:PP}
  \end{equation}
  where $p_i \in (0,1]$.
  \end{assumptionsup}

Note that the assumption \Cref{ass:A_k_supp2} is equivalent to \Cref{ass:A_k_supp} in the main paper.

Let $V:\R^d \rightarrow [1,\infty)$ a measurable function. We consider the following assumption on the family $\{(H_\theta^{(i)},\pi_{\theta}^{(i)}) \, : \,  \theta \in \Theta, i \in [b]\}$.

\begin{assumptionsup}\label{ass:gradient_H}
  For any $i\in[b]$, the following conditions hold.
  \begin{enumerate}[wide, labelwidth=!, labelindent=0pt,label=(\roman*),noitemsep,nolistsep]

    \item \label{ass:gradient_H_1} For any $\theta \in \Theta$, $\pi_\theta^{(i)}(\|H_\theta^{(i)}\|) < \infty$ and $(\theta,z^{(i)}) \mapsto H_\theta^{(i)}(z^{(i)})$ is measurable.

    \item \label{ass:gradient_H_2} There exists $L_H \ge 0$ such that for any $z \in \Rd$ and $\theta_1,\theta_2 \in \Theta$,
    $$
    \norm{H_{\theta_2}^{(i)}(z) - H_{\theta_1}^{(i)}(z)} \le L_H \norm{\theta_2 - \theta_1}V^{1/2}(z)\eqsp.
    $$
    \end{enumerate}
\end{assumptionsup}

\subsection{Stochastic Approximation Framework}

Let $(X_k^{(i,1)})_{k \in \N, i \in [b]}$ a sequence of independent an identically distributed (i.i.d.) random variables with distribution $\nu_1$ independent of the sequence $(X_k^{(i,2)})_{k \in \N,i \in [b]}$ which is i.i.d. and with uniform distribution on $[0,1]$.
We consider a family of unadjusted Markov kernels $\{Q^{(i)}_{\gamma,\theta} : \gamma \in (0,\bar{\gamma}], \theta \in \Theta, i \in [b]\}$.
Let $(\gamma_k)_{k \in \N^*} \in (\R^*_+)^{\N^*}$ a sequence of step-sizes which will be used to obtain approximate samples from $\pi_\theta^{(i)}$ using $Q^{(i)}_{\gamma,\theta}$.

We now recast the proposed approach detailed in \Cref{algo:FedSOUL} into a stochastic approximation framework.

Starting from some initialization $(Z_0^{(1,0)},\ldots,Z_0^{(b,0)},\theta_0) \in \R^{bd} \times \Theta$, we define on a probability space $(\Omega,\mathcal{F},\mathbb{P})$, the sequence $((Z_k^{(1,m)},\ldots,Z_k^{(b,m)})_{m \in[M]}, \theta_k)_{k \in \N}$ via the recursion for $k \in \N$,

\begin{align}
  & \text{ for any $i\in[b]$, given $\mathcal{F}_{k-1}$}, (Z_k^{(i,m)})_{m \in \{0,\ldots,M\}} \text{ is a Markov chain with Markov kernel $Q^{(i)}_{\gamma_k,\theta_k}$} \\
&  \text{with $Z_k^{(i,0)} = Z_{k-1}^{(i,M)}$ \eqsp,} \label{eq:MC}\\
  &\theta_{k+1} = \Pi_{\Theta}\br{\theta_k - \boldsymbol{\eta}_{k+1} \odot \boldsymbol{\Delta}_{\theta_k}\pr{Z_{k+1}^{(1:M)},X_{k+1}^{(1)},X_{k+1}^{(2)}}}\eqsp,\label{eq:global_SA_scheme}
\end{align}

where $\odot$ denotes the Hadamard product and for any $k \in \N$,  $\mathcal{F}_k= \sigma(\theta_0,\{\{Z^{(i,m)}_l\}_{m \in [M]}\, : \,  l \in \{0,\ldots,k\} \, , \, i\in[b]\})$ and $\mathcal{F}_{-1}= \sigma(\theta_0, \{Z_0^{(i,0)} \, : \, i \in[b]\})$.
In addition, for any $k \in \N$, $\boldsymbol{\eta}_{k+1} = (\eta_{k+1}^{(1)}, \eta_{k+1}^{(2)})^\top$, $Z_{k+1}^{(1:M)} = ([Z_{k+1}^{(1,1:M)}]^\top,\ldots,[Z_k^{(b,1:M)}]^\top)^\top$ and for any $\theta \in \Theta$, $z^{(1:M)} \in \R^{Md}$, $x^{(1)} \in \mathsf{X}_1$, $x^{(2)} \in \mathsf{X}_2$ ,
\begin{align}
  \boldsymbol{\Delta}_{\theta}\pr{z^{(1:M)},x^{(1)},x^{(2)}} &= 
  \begin{pmatrix} 
    \Delta_{\phi}\pr{z^{(1:M)},x^{(1)},x^{(2)}} \\
    \Delta_{\beta}\pr{z^{(1:M)},x^{(2)}} 
  \end{pmatrix} \eqsp, \\
  &=
  \begin{pmatrix} 
    \sum_{i=1}^b \mathscr{S}_i\br{\mathscr{C}_i\pr{\Delta_{\phi}^{(i)}(z^{(i,1:M)}),x^{(i,1)}},x^{(i,2)}} \\
    \sum_{i=1}^b \mathscr{S}_i\br{\Delta_{\beta}^{(i)}(z^{(i,1:M)}),x^{(i,2)}} 
  \end{pmatrix} \eqsp,\label{eq:def_global_gradient}
\end{align}
where $\{\Delta_{\beta}^{(i)},\Delta_{\phi}^{(i)}\}_{i \in [b]}$ defined by
\begin{align}
  \Delta_{\beta}^{(i)}(z^{(i,1:M)}) &=  -\frac{1}{M} \sum_{m=1}^M \bbr{(1/b)\nabla_{\beta}p(\beta) + \nabla_{\beta}\log \brnn{p\prn{ z^{(i,m)}\mid \beta}}} \\
   \Delta_{\phi}^{(i)}(z^{(i,1:M)}) &= -\frac{1}{M} \sum_{m=1}^M \bbr{(1/b)\nabla_{\phi}p(\phi) + \nabla_\phi\log \brnn{p\prn{\mathrm{D}_i \mid z^{(i,m)}, \phi}}}\eqsp.
\end{align}

\subsection{Main Result}

In order to show non-asymptotic convergence guarantees for \texttt{FedSOUK} detailed in \Cref{algo:FedSOUL}, we need additional assumptions ensuring some stability of the sequence $(Z_k^{(i,m)}\, :\, m \in \{0,\ldots,M\}, i \in [b])_{k \in \N}$.
These conditions are stated hereafter.

\begin{assumptionsup}\label{ass:markov_kernel}
  For any $i\in[b]$, the following conditions hold.
  \begin{enumerate}[wide, labelwidth=!, labelindent=0pt,label=(\roman*),noitemsep,nolistsep]

    \item \label{ass:markov_kernel1} There exists $A_1 \ge 1$ such that for any $p,k \in \N$ and $m \in \{0,\ldots,M\}$,
    $$
    \mathbb{E}\br{[Q^{(i)}_{\gamma_k,\theta_k}]^p V(Z^{(i,m)}_k) \mid Z^{(i,0)}_0} \le A_1 V(Z^{(i,0)}_0)\eqsp, \ \mathbb{E}\br{V(Z^{(i,0)}_0)} < \infty\eqsp,
    $$
    where $(Z_k^{(i,m)}\, :\, m \in \{0,\ldots,M\}, i \in [b])_{k \in \N}$ is defined in \eqref{eq:MC}.

    \item \label{ass:markov_kernel2} There exists $A_2,A_3 \ge 1$, $\rho \in [0,1)$ such that for any $\gamma \in (0,\bar{\gamma}]$, $\theta \in \Theta$, $z \in \R^d$ and $k \in \N$, $Q^{(i)}_{\gamma,\theta}$ admits $\pi_{\gamma,\theta}^{(i)}$ as stationary distribution and 
    \begin{align}
      \norm{\updelta_z [Q^{(i)}_{\gamma,\theta}]^k - \pi_{\gamma,\theta}^{(i)}}_V &\le A_2 \rho^{k\gamma} V(z) \\
       \pi_{\gamma,\theta}^{(i)}(V) &\le A_3 \eqsp.
    \end{align}

    \item \label{ass:markov_kernel3} There exists $\boldsymbol{\Psi}: \R_+^* \rightarrow \R_+$ such that for any $\gamma \in (0,\bar{\gamma}]$ and $\theta \in \Theta$,
    \begin{align}
      \norm{\pi_{\gamma,\theta}^{(i)} - \pi_\theta^{(i)}}_{V^{1/2}} \le \boldsymbol{\Psi}(\gamma)\eqsp.
    \end{align}
    \end{enumerate}
\end{assumptionsup}

\begin{assumptionsup}\label{ass:markov_kernel_bis}
  There exists a measurable function $V: \R^d \rightarrow [1,\infty)$, $\boldsymbol{\Gamma}_1: (\R^*_+)^2 \rightarrow \R_+$ and $\boldsymbol{\Gamma}_2: (\R^*_+)^2 \rightarrow \R_+$ such that for any $\gamma_1,\gamma_2 \in (0,\bar{\gamma}]$ with $\gamma_2 < \gamma_1$, $\theta_1,\theta_2 \in \Theta$, $z \in \R^d$, $a \in [1/4,1/2]$, we have for any $i \in [b]$,
  $$
  \norm{\updelta_z Q_{\gamma_2,\theta_2}^{(i)} - \updelta_z Q_{\gamma_1,\theta_1}^{(i)}}_{V^a} \le [\boldsymbol{\Gamma}_1(\gamma_1,\gamma_2) + \boldsymbol{\Gamma}_2(\gamma_1,\gamma_2)\norm{\theta_2-\theta_1}]V^{2a}(z)\eqsp.
  $$
\end{assumptionsup}

We are now ready to show our main result.
To ease the presentation, assume for any $k \in \N$ that $\eta^{(1)}_{k+1} = \eta^{(2)}_{k+1} = \eta_{k+1}$ and, for any $i \in [b], \gamma_{k+1}^{(i)} = \gamma_{k+1}$.

\begin{theorem}
  \label{thm:general}
  Assume \Cref{ass:convex_set}, \Cref{ass:function_f}, \Cref{ass:compression}, \Cref{ass:A_k_supp2}, \Cref{ass:gradient_H}, \Cref{ass:markov_kernel} and \Cref{ass:markov_kernel_bis} and let for any $k \in [K]$, $\eta_k \in (0,1/L_f]$.
  In addition, for any $\theta \in \Theta$, $z \in \R^d$ and $i\in[b]$, assume that $\|H^{(i)}_\theta(z)\| \le V^{1/4}(z)$.
  Then, for any $K \in \N^*$, we have
    \begin{align}
        \mathbb{E}\br{\frac{\sum_{k=1}^K \eta_{k}\{f(\theta_{k}) - f(\thetas)\}}{\sum_{k=1}^K \eta_{k}}}
        &\le \frac{E_K}{\sum_{k=1}^K \eta_{k}} \eqsp,
    \end{align}
    where, for any $K \in \N^*$,
    \begin{align}
      E_K
      &= 2R_\Theta^2 + 2 A_1 \sup_{i \in [b], m \in [M]} \bbr{\mathbb{E}\br{V^{1/2}(Z_0^{(i,m)})}} \sum_{k=1}^K \eta_k^2 \pr{8 b L_f^2R_\Theta^2 + \sum_{i=1}^{b}\frac{(\omega_i + 1 + p_i)}{p_i}} \\
      &+ b \sup_{i \in [b], m \in [M]} \bbr{C_3^{(i,m)}} \br{\sum_{k=1}^{K} |\eta_{k} - \eta_{k-1}|\gamma_{k-1}^{-1} + \sum_{k=1}^{K} \eta_{k}^2 \gamma_{k-1}^{-1} + \eta_K/\gamma_K - \eta_1/\gamma_1} \\
      &+ b . A_1 C_{c,2} \sup_{i \in [b], m \in [M]} \bbr{\mathbb{E}\br{V(Z_0^{(i,m)})}}\sum_{k=1}^K \eta_k \gamma_{k}^{-1}\br{\gamma_k^{-1}\bbr{\boldsymbol{\Lambda}_1(\gamma_{k-1},\gamma_k) + \boldsymbol{\Lambda}_2(\gamma_{k-1},\gamma_k)\eta_k} + \eta_k} \\
      &+ b \sum_{k=1}^K \eta_k \boldsymbol{\Psi}(\gamma_{k-1})\eqsp,
    \end{align}
    with $\{C_3^{(i,m)}\}_{i \in [b],m \in [M]}$ defined in \Cref{lemma:epsilon_b} and $C_{c,2}$ defined in \Cref{lemma:epsilon_c}.
\end{theorem}

\begin{proof}
  The proof follows by using the fact that \eqref{eq:epsilon_a} is a $(\mathcal{F}_{k-1})_{k \in \N^*}$-martingale increment and by combining \Cref{lemma_convex_general}-\ref{lemma:epsilon_d}.
\end{proof}

\subsection{Supporting Lemmata}
\label{subsec:SA_convex_lemmata}

For convenience, we define the following quantities that will naturally appear in our derivations.
For any $k \in \N^*$, let
\begin{align}
      \boldsymbol{\epsilon}_{k} = \boldsymbol{\Delta}_{\theta_{k-1}}\pr{Z_{k}^{(1:M)},X_{k}^{(1)},X_{k}^{(2)}} - \nabla f(\theta_{k-1})\eqsp, \label{eq:epsilon_k}
\end{align}
where $\boldsymbol{\Delta}_{\theta}$ is defined in \eqref{eq:def_global_gradient}.

The following lemma first provides a non-asymptotic upper bound on $\sum_{k=1}^K \eta_{k}\{f(\theta_{k}) - f(\thetas)\}$ involving key quantities to control such as the Monte Carlo approximation error term \eqref{eq:epsilon_k}.

\begin{lemma}
  \label{lemma_convex_general}
  Assume \Cref{ass:convex_set} and \Cref{ass:function_f}, and let for any $k \in [K]$, $\eta_k \in (0,1/L_f]$.
  Then, for any $K \in \N^*$, we have
    \begin{align}
        \sum_{k=1}^K \eta_{k}\{f(\theta_{k}) - f(\thetas)\}
        &\le 2R_\Theta^2 + \sum_{k=1}^K \eta_k^2\norm{\boldsymbol{\epsilon}_k}^2 - \sum_{k=1}^K\eta_k\left\langle \Pi_{\Theta}\pr{\theta_{k-1} - \eta_{k} \nabla f(\theta_{k-1})} - \thetas, \boldsymbol{\epsilon}_k \right \rangle \eqsp,
    \end{align}
    where $\{\boldsymbol{\epsilon}_k\}_{k=1}^K$ is defined in \eqref{eq:epsilon_k}.
\end{lemma}

\begin{proof}
    Let $k \in \N$.
    Since $\Theta$ is closed and convex by \Cref{ass:convex_set}, the indicator function $\iota_{\Theta}$, defined for any $u \in \R^{d_\Phi + d_{\mathsf{B}}}$ by $\iota_{\Theta}(u) = 0$ if $u \in \Theta$ and $\iota_{\Theta}(u) = \infty$ otherwise, is lower semi-continuous and convex. Therefore by \citet[Lemma 7]{AtchadePerturbedProxGradient} we have
    \begin{align}
        \iota_{\mathsf{B}}(\beta_{k+1}) - \iota_{\mathsf{B}}(\beta_{\star}) &\le -\frac{1}{\eta_{k+1}} \left\langle \beta_{k+1} - \beta_\star, \beta_{k+1} - \beta_k + \eta_{k+1} \Delta_{\beta_k}\pr{Z_{k+1}^{(1:M)},X_{k+1}^{(2)}} \right\rangle \eqsp, \label{eq:lemma1_1} \\
        \iota_{\Phi}(\phi_{k+1}) - \iota_{\Phi}(\phi_{\star}) &\le -\frac{1}{\eta_{k+1}} \left\langle \phi_{k+1} - \phi_\star, \phi_{k+1} - \phi_k + \eta_{k+1} \Delta_{\phi_k}\pr{Z_{k+1}^{(1:M)},X_{k+1}^{(1)},X_{k+1}^{(2)}} \right\rangle \eqsp,\label{eq:lemma1_2}
    \end{align}
    where $\thetas = (\phi_\star,\beta_\star)$ is defined in \eqref{eq:FL_problem}.
    In addition by \Cref{ass:function_f}-\ref{ass:function_f_2}, we have for any $i \in [b]$,
    \begin{equation}
        \label{eq:f_i_convex}
        f_i(\theta_{k+1}) - f_i(\theta_{k}) \le \langle \nabla f_i(\theta_{k}), \theta_{k+1} - \theta_{k} \rangle + \frac{L_f}{2}\norm{\theta_{k+1} - \theta_{k}}^2\eqsp.
    \end{equation} 
    Using \eqref{eq:f_i_convex} and the fact that for any $k \in \N$, $\eta_{k+1} \le 1/L_f$, we have 
        \begin{align}
            f(\theta_{k+1}) - f(\theta_k) 
            &\le \langle \nabla_{\beta} f(\theta_{k}), \beta_{k+1} - \beta_{k} \rangle + \frac{L_f}{2}\norm{\beta_{k+1} - \beta_{k}}^2 \\
            &+  \langle \nabla_{\phi} f(\theta_{k}), \phi_{k+1} - \phi_{k} \rangle + \frac{L_f}{2}\norm{\phi_{k+1} - \phi_{k}}^2 \\
            &\le \langle \nabla_{\beta} f(\theta_{k}), \beta_{k+1} - \beta_{k}\rangle + \frac{1}{2\eta_{k+1}}\norm{\beta_{k+1} - \beta_{k}}^2 \\
            &+  \langle \nabla_{\phi} f(\theta_{k}), \phi_{k+1} - \phi_{k} \rangle + \frac{1}{2\eta_{k+1}}\norm{\phi_{k+1} - \phi_{k}}^2 \eqsp.\label{eq:lemma1_3}
    \end{align}
    Finally, \Cref{ass:function_f}-\ref{ass:function_f_1} implies for any $i \in [b]$,
    \begin{equation}
        f_i(\theta_{k}) - f_i(\thetas) \le  - \langle \nabla f_i(\theta_{k}), \theta_{\star} - \theta_{k} \rangle\eqsp.\label{eq:lemma1_4}
    \end{equation} 
    For any $i \in [b]$, let $F_i = f_i + \iota_{\Theta}$ and let $F = (1/b)\sum_{i=1}^b F_i$.
    Using this notation and combining \eqref{eq:lemma1_1}, \eqref{eq:lemma1_2}, \eqref{eq:lemma1_3} and \eqref{eq:lemma1_4}, we have
    \begin{align}
        &F(\theta_{k+1}) - F(\thetas) \\
        =& \ f(\theta_{k+1}) - f(\theta_k) + f(\theta_k) - f(\thetas) + \iota_{\Phi}(\phi_{k+1}) - \iota_{\Phi}(\phi_\star) + \iota_{\mathsf{B}}(\beta_{k+1}) - \iota_{\mathsf{B}}(\beta_\star) \\
        \le& - \left\langle \beta_{k+1} - \beta_{\star}, \Delta_{\beta_k}\pr{Z_{k+1}^{(i,1:M)},X_{k+1}^{(2)}} - \nabla_{\beta} f(\theta_{k}) \right \rangle - \left\langle \beta_{k+1} - \beta_{\star}, \beta_{k+1} - \beta_{k} \right \rangle \\
        &- \left\langle \phi_{k+1} - \phi_{\star}, \Delta_{\phi_k}\pr{Z_{k+1}^{(1:M)},X_{k+1}^{(1)},X_{k+1}^{(2)}} - \nabla_{\phi} f(\theta_{k}) \right \rangle - \left\langle \phi_{k+1} - \phi_{\star}, \phi_{k+1} - \phi_{k} \right \rangle \\
        &+ \frac{1}{2\eta_{k+1}}\norm{\beta_{k+1} - \beta_{k}}^2 + \frac{1}{2\eta_{k+1}}\norm{\phi_{k+1} - \phi_{k}}^2 \\
        =& - \left\langle \theta_{k+1} - \theta_{\star}, \boldsymbol{\Delta}_{\theta_k}\pr{Z_{k+1}^{(1:M)},X_{k+1}^{(1)},X_{k+1}^{(2)}} - \nabla f(\theta_{k}) \right \rangle \\
        &+ \frac{1}{2\eta_{k+1}} \br{\norm{\phi_{k} - \phi_\star}^2 - \norm{\phi_{k+1} - \phi_\star}^2} + \frac{1}{2\eta_{k+1}} \br{\norm{\beta_{k} - \beta_{\star}}^2 - \norm{\beta_{k+1} - \beta_{\star}}^2}\eqsp. \label{eq:lemma1_5}
    \end{align}
    From \eqref{eq:lemma1_5}, it follows for any $K \in \N^*$ that
    \begin{align}
        &\sum_{k=1}^K\eta_{k}\{ F(\theta_{k}) - F(\thetas)\} \\
        &\le - \sum_{k=1}^K\eta_k\left\langle \theta_{k} - \theta_{\star}, \boldsymbol{\Delta}_{\theta_{k-1}}\pr{Z_{k}^{(1:M)},X_{k}^{(1)},X_{k}^{(2)}} - \nabla f(\theta_{k-1}) \right \rangle \\
        &+ \frac{1}{2}\norm{\phi_{0} - \phi_\star}^2 -\frac{1}{2}\norm{\phi_{K} - \phi_\star}^2 + \frac{1}{2}\norm{\beta_{0} - \beta_\star}^2 - \frac{1}{2}\norm{\beta_{K} - \beta_\star}^2 \\
        &\le - \sum_{k=1}^K \eta_k \left\langle \theta_{k} - \theta_{\star}, \boldsymbol{\Delta}_{\theta_{k-1}}\pr{Z_{k}^{(1:M)},X_{k}^{(1)},X_{k}^{(2)}} - \nabla f(\theta_{k-1}) \right \rangle + \frac{1}{2}\norm{\theta_{0} - \thetas}^2 \\
        &= - \sum_{k=1}^K\eta_k\left\langle \theta_{k} - \Pi_{\Theta}\pr{\theta_{k-1} - \eta_k \nabla f(\theta_{k-1})}, \boldsymbol{\Delta}_{\theta_{k-1}}\pr{Z_{k}^{(1:M)},X_{k}^{(1)},X_{k}^{(2)}} - \nabla f(\theta_{k-1}) \right \rangle \\
        &- \sum_{k=1}^K\eta_k\left\langle \Pi_{\Theta}\pr{\theta_{k-1} - \eta_k  \nabla f(\theta_{k-1})} - \thetas, \boldsymbol{\Delta}_{\theta_{k-1}}\pr{Z_{k}^{(1:M)},X_{k}^{(1)},X_{k}^{(2)}} - \nabla f(\theta_{k-1}) \right \rangle \\
        &+ \frac{1}{2}\norm{\theta_{0} - \thetas}^2 \\
        &\le \sum_{k=1}^K \eta_k^2\norm{\boldsymbol{\Delta}_{\theta_{k-1}}\pr{Z_{k}^{(1:M)},X_{k}^{(1)},X_{k}^{(2)}} - \nabla f(\theta_{k-1})}^2 + \frac{1}{2}\norm{\theta_{0} - \thetas}^2 \\
        &- \sum_{k=1}^K\eta_k\left\langle \Pi_{\Theta}\pr{\theta_{k-1} - \eta_k \nabla f(\theta_{k-1})} - \thetas, \boldsymbol{\Delta}_{\theta_{k-1}}\pr{Z_{k}^{(1:M)},X_{k}^{(1)},X_{k}^{(2)}} - \nabla f(\theta_{k-1}) \right \rangle \eqsp,
    \end{align}
    where we used \citet[Lemma 7]{AtchadePerturbedProxGradient} and the Cauchy-Schwarz inequality in the last inequality.
    The proof is concluded using $f \le F$, $f(\thetas) = F(\thetas)$ since $\thetas \in \Theta$, and by noting that under \Cref{ass:convex_set} we have $\norm{\theta_0 - \thetas} \le 2R_\Theta$.
\end{proof}

\Cref{lemma_convex_general} involves two key quantities to upper bound namely $\norm{\boldsymbol{\epsilon}_k}$ and $\left\langle \Pi_{\Theta}\pr{\theta_{k-1} - \eta_k \nabla f(\theta_{k-1})} - \thetas, \boldsymbol{\epsilon}_k \right \rangle$ for any $k \in \N^*$.
Our next lemmata aim at controlling the expectations of these two terms.
In particular, \Cref{lemma:convex_error_term}  and  \Cref{lemma:convex_innerprod} show that the impacts of Monte Carlo approximation, partial participation and compression can be decoupled.

To this end, define for any $k \in \N^*$ and $i \in [b]$
\begin{align}
    \varepsilon_{\beta,k}^{(i)} &= \frac{1}{M}\sum_{m=1}^M H_{\beta_{k-1}}^{(i)}\pr{Z_k^{(i,m)}} - \nabla_{\beta} f_i(\theta_{k-1})\eqsp,\\
    \varepsilon_{\phi,k}^{(i)} &= \frac{1}{M}\sum_{m=1}^M H^{(i)}_{\phi_{k-1}}\pr{Z_k^{(i,m)}} - \nabla_{\phi} f_i(\theta_{k-1})\eqsp, \\
    \varepsilon_{\theta,k}^{(i)} &= \br{\varepsilon_{\beta,k}^{(i)},\varepsilon_{\phi,k}^{(i)}}\eqsp, \label{eq:epsilon_theta}
\end{align}
where, for any $k \in \N^*$ and $i \in [b]$, $H^{(i)}_{\theta_{k-1}}(Z_k^{(i,m)}) = [H^{(i)}_{\phi_{k-1}}(Z_k^{(i,m)}),H^{(i)}_{\beta_{k-1}}(Z_k^{(i,m))}]$ is defined in \eqref{eq:grad_f}.

\Cref{lemma:convex_error_term} shows that $\norm{\boldsymbol{\epsilon}_k}$ can be upper bounded by a quantity involving the norm of $\{H_\theta^{(i)}\}_{i \in [b]}$.

\begin{lemma}
  \label{lemma:convex_error_term}
  Assume \Cref{ass:convex_set}, \Cref{ass:function_f}, \Cref{ass:compression} and \Cref{ass:A_k_supp2}.
  Then, for any $k \in \N^*$, we have
    \begin{align}
        \mathbb{E}\br{\norm{\boldsymbol{\epsilon}_k}^2} \le \frac{1}{M}\sum_{i=1}^{b}\frac{(\omega_i + 1 + p_i)}{p_i}\bbr{\sum_{m=1}^M\mathbb{E}\br{\norm{H_{\theta_{k-1}}^{(i)}(Z_k^{(i,m)})}^2}} + 8 b L_f^2R_\Theta^2\eqsp,
    \end{align}
    where $\{\boldsymbol{\epsilon}_k\}_{k=1}^K$ is defined in \eqref{eq:epsilon_k}.
\end{lemma}

\begin{proof}
    Let $k \in \N^*$.
    Then by using \eqref{eq:def_global_gradient}, we have
    \begin{align}
        \mathbb{E}\br{\norm{\boldsymbol{\epsilon}_k}^2} &= \mathbb{E}\br{\norm{\sum_{i=1}^b \mathscr{S}_i\br{\mathscr{C}_i\pr{\frac{1}{M}\sum_{m=1}^M H_{\phi_{k-1}}^{(i)}(Z_k^{(i,m)}),X_k^{(i,1)}},X_k^{(i,2)}} - \nabla_{\phi}f(\theta_{k-1})}^2} \\
        &+ \mathbb{E}\br{\norm{\sum_{i=1}^b \mathscr{S}_i\br{\frac{1}{M}\sum_{m=1}^M H_{\beta_{k-1}}^{(i)}(Z_k^{(i,m)}),X_k^{(i,2)}} - \nabla_{\beta}f(\theta_{k-1})}^2}\eqsp.\label{eq:lemma2_3}
    \end{align}
    Using \Cref{ass:compression} and \Cref{ass:A_k_supp2}, it follows that
    \begin{multline}\label{eq:eq:lem_compr}
    \mathbb{E}\br{\norm{\sum_{i=1}^b \mathscr{S}_i\br{\mathscr{C}_i\pr{\frac{1}{M}\sum_{m=1}^M H_{\phi_{k-1}}^{(i)}(Z_k^{(i,m)}),X_k^{(i,1)}},X_k^{(i,2)}} - \nabla_{\phi}f(\theta_{k-1})}^2} \\
    = \mathbb{E}\Bigg[\Big\|\sum_{i=1}^b \Big\{\mathscr{S}_i\br{\mathscr{C}_i\pr{\frac{1}{M}\sum_{m=1}^M H_{\phi_{k-1}}^{(i)}(Z_k^{(i,m)}),X_k^{(i,1)}},X_k^{(i,2)}} \\
    -\mathscr{C}_i\pr{\frac{1}{M}\sum_{m=1}^M H_{\phi_{k-1}}^{(i)}(Z_k^{(i,m)}),X_k^{(i,1)}}\Big\}\Big\|^2\Bigg] \\
    +\mathbb{E}\br{\norm{\sum_{i=1}^b \mathscr{C}_i\pr{\frac{1}{M}\sum_{m=1}^M H_{\phi_{k-1}}^{(i)}(Z_k^{(i,m)}),X_k^{(i,1)}} - \nabla_{\phi}f(\theta_{k-1})}^2}\eqsp.
  \end{multline}
  In addition, by \Cref{ass:compression}-\ref{ass:compression:unbiased} and \Cref{ass:compression}-\ref{ass:compression:variance}, we obtain
  \begin{align} 
    &\mathbb{E}\Bigg[\Big\|\sum_{i=1}^b \Big\{\mathscr{S}_i\br{\mathscr{C}_i\pr{\frac{1}{M}\sum_{m=1}^M H_{\phi_{k-1}}^{(i)}(Z_k^{(i,m)}),X_k^{(i,1)}},X_k^{(i,2)}} \\
    &-\mathscr{C}_i\pr{\frac{1}{M}\sum_{m=1}^M H_{\phi_{k-1}}^{(i)}(Z_k^{(i,m)}),X_k^{(i,1)}}\Big\}\Big\|^2\Bigg]\\
    &= \sum_{i=1}^b\mathbb{E}\Bigg[\Big\| \mathscr{S}_i\br{\mathscr{C}_i\pr{\frac{1}{M}\sum_{m=1}^M H_{\phi_{k-1}}^{(i)}(Z_k^{(i,m)}),X_k^{(i,1)}},X_k^{(i,2)}} \\
    &-\mathscr{C}_i\pr{\frac{1}{M}\sum_{m=1}^M H_{\phi_{k-1}}^{(i)}(Z_k^{(i,m)}),X_k^{(i,1)}}\Big\|^2\Bigg] \\
    &\le \sum_{i=1}^{b} \pr{\frac{1-p_i}{p_i}} \mathbb{E}\br{\norm{\mathscr{C}_i\pr{\frac{1}{M}\sum_{m=1}^M H_{\phi_{k-1}}^{(i)}(Z_k^{(i,m)}),X_k^{(i,1)}}^2}} \\
    &= \sum_{i=1}^{b} \pr{\frac{1-p_i}{p_i}} \mathbb{E}\br{\norm{\up_i\pr{\frac{1}{M}\sum_{m=1}^M H_{\phi_{k-1}}^{(i)}(Z_k^{(i,m)}),X_k^{(1,i)}} - \frac{1}{M}\sum_{m=1}^M H_{\phi_{k-1}}^{(i)}(Z_k^{(i,m)})}^2} \\
    &+ \sum_{i=1}^{b} \pr{\frac{1-p_i}{p_i}} \mathbb{E}\br{\norm{\frac{1}{M}\sum_{m=1}^M H_{\phi_{k-1}}^{(i)}(Z_k^{(i,m)})}}^2 \\
    &\le \sum_{i=1}^{b} \br{\pr{\frac{1-p_i}{p_i}}(\omega_i + 1)} \mathbb{E}\br{\norm{\frac{1}{M}\sum_{m=1}^M H_{\phi_{k-1}}^{(i)}(Z_k^{(i,m)})}^2} \\
    &= \frac{1}{M^2}\sum_{i=1}^{b} \br{\pr{\frac{1-p_i}{p_i}}(\omega_i + 1)} \mathbb{E}\br{\norm{\sum_{m=1}^M H_{\phi_{k-1}}^{(i)}(Z_k^{(i,m)})}^2} \eqsp. \label{eq:bound:lem_compr:1} 
  \end{align}
Similarly, by \Cref{ass:compression}-\ref{ass:compression:unbiased} and \Cref{ass:compression}-\ref{ass:compression:variance}, we have 
\begin{align}
    &\mathbb{E}\br{\norm{\sum_{i=1}^b \mathscr{C}_i\pr{\frac{1}{M}\sum_{m=1}^M H_{\phi_{k-1}}^{(i)}(Z_k^{(i,m)}),X_k^{(i,1)}} - \nabla_{\phi}f(\theta_{k-1})}^2} \\
    &= \mathbb{E}\Bigg[\Big\|\sum_{i=1}^{b}\br{\up_i\pr{\frac{1}{M}\sum_{m=1}^M H_{\phi_{k-1}}^{(i)}(Z_k^{(i,m)}),X_k^{(i,1)}} - \frac{1}{M}\sum_{m=1}^M H_{\phi_{k-1}}^{(i)}(Z_k^{(i,m)})} \\
    & + \sum_{i=1}^b \frac{1}{M}\sum_{m=1}^M H_{\phi_{k-1}}^{(i)}(Z_k^{(i,m)})-\nabla_\phi f(\theta_{k-1})\Big\|^2\Bigg] \\
    &= \sum_{i=1}^{b}\mathbb{E}\br{\norm{\up_i\pr{\frac{1}{M}\sum_{m=1}^M H_{\phi_{k-1}}^{(i)}(Z_k^{(i,m)}),X_k^{(i,1)}} - \frac{1}{M}\sum_{m=1}^M H_{\phi_{k-1}}^{(i)}(Z_k^{(i,m)})}^2} \\
    &+ \mathbb{E}\br{\norm{\sum_{i=1}^b \frac{1}{M}\sum_{m=1}^M H_{\phi_{k-1}}^{(i)}(Z_k^{(i,m)})-\nabla_\phi f\parentheseLigne{\theta_{k-1}}}^2} \\
    &\le \sum_{i=1}^{b}\omega_i\mathbb{E}\br{\norm{\frac{1}{M}\sum_{m=1}^M H_{\phi_{k-1}}^{(i)}(Z_k^{(i,m)})}^2} \\
    &+ \mathbb{E}\br{\norm{\sum_{i=1}^b\frac{1}{M}\sum_{m=1}^M H_{\phi_{k-1}}^{(i)}(Z_k^{(i,m)})-\nabla_\phi f(\theta_{k-1})}^2} \\
    &= \frac{1}{M^2}\sum_{i=1}^{b}\omega_i\mathbb{E}\br{\norm{\sum_{m=1}^M H_{\phi_{k-1}}^{(i)}(Z_k^{(i,m)})}^2} \\
    &+ \mathbb{E}\br{\norm{\sum_{i=1}^b\frac{1}{M}\sum_{m=1}^M H_{\phi_{k-1}}^{(i)}(Z_k^{(i,m)})-\nabla_\phi f(\theta_{k-1})}^2} \label{eqmax1}\eqsp.
\end{align}
By plugging \eqref{eq:bound:lem_compr:1} and \eqref{eqmax1} into \eqref{eq:eq:lem_compr}, we finally obtain
\begin{align}
    &\mathbb{E}\br{\norm{\sum_{i=1}^b \mathscr{S}_i\br{\mathscr{C}_i\pr{\frac{1}{M}\sum_{m=1}^M H_{\phi_{k-1}}^{(i)}(Z_k^{(i,m)}),X_k^{(i,1)}},X_k^{(i,2)}} - \nabla_{\phi}f(\theta_{k-1})}^2}  \\
    &\le \frac{1}{M^2}\sum_{i=1}^{b}\frac{(\omega_i + 1 - p_i)}{p_i}\mathbb{E}\br{\norm{\sum_{m=1}^M H_{\theta_{k-1}}^{(i)}(Z_k^{(i,m)})}^2} + \sum_{i=1}^b\mathbb{E}\br{\norm{\varepsilon_{\phi,k}^{(i)}}^2}\eqsp.\label{eq:lemma2_2}
\end{align}
Finally, using the same arguments, we have under \Cref{ass:A_k_supp},
\begin{align}
  &\mathbb{E}\br{\norm{\sum_{i=1}^b \mathscr{S}_i\br{\frac{1}{M}\sum_{m=1}^M H_{\beta_{k-1}}^{(i)}(Z_k^{(i,m)}),X_k^{(i,2)}} - \nabla_{\beta}f(\theta_{k-1})}^2} \\
  &\le \frac{1}{M^2}\sum_{i=1}^{b}\pr{\frac{1-p_i}{p_i}}\mathbb{E}\br{\norm{\sum_{m=1}^M H_{\beta_{k-1}}^{(i)}(Z_k^{(i,m)})}^2} \\
    &+ \mathbb{E}\br{\norm{\sum_{i=1}^b\frac{1}{M}\sum_{m=1}^M H_{\beta_{k-1}}^{(i)}(Z_k^{(i,m)})-\nabla_\beta f(\theta_{k-1})}^2} \\
  &\le \frac{1}{M^2}\sum_{i=1}^{b}\pr{\frac{1-p_i}{p_i}}\mathbb{E}\br{\norm{\sum_{m=1}^M H_{\beta_{k-1}}^{(i)}(Z_k^{(i,m)})}^2} \\
  &+ \sum_{i=1}^b\mathbb{E}\br{\norm{\varepsilon_{\beta,k}^{(i)}}^2} \label{eqmax2}\eqsp.
\end{align}
Combining \eqref{eq:lemma2_3} and \eqref{eq:lemma2_2} and using \eqref{eq:epsilon_theta}, lead to 
\begin{align}
        \mathbb{E}\br{\norm{\boldsymbol{\epsilon}_k}^2} &\le \frac{1}{M^2}\sum_{i=1}^{b}\frac{(\omega_i + 1 - p_i)}{p_i}\mathbb{E}\br{\norm{\sum_{m=1}^M H_{\theta_{k-1}}^{(i)}(Z_k^{(i,m)})}^2} + \sum_{i=1}^b\mathbb{E}\br{\norm{\varepsilon_{\theta,k}^{(i)}}^2} \label{eq:bound_ek}\\
        &\le \frac{1}{M}\sum_{i=1}^{b}\frac{(\omega_i + 1 + p_i)}{p_i}\bbr{\sum_{m=1}^M\mathbb{E}\br{\norm{H_{\theta_{k-1}}^{(i)}(Z_k^{(i,m)})}^2}} + 2 \sum_{i=1}^b\sup_{\theta \in \Theta} \norm{\nabla f_i(\theta)}^2 \\
        &\le \frac{1}{M}\sum_{i=1}^{b}\frac{(\omega_i + 1 + p_i)}{p_i}\bbr{\sum_{m=1}^M\mathbb{E}\br{\norm{H_{\theta_{k-1}}^{(i)}(Z_k^{(i,m)})}^2}} + 2L_f^2\sum_{i=1}^b\sup_{\theta \in \Theta}\norm{\theta - \thetasi}^2 \eqsp,
\end{align}
where we used \Cref{ass:function_f} for the last inequality and $\thetasi$ is a minimizer of $f_i$.
The proof is concluded using for any $i \in [b]$ that $\norm{\theta - \thetasi} \le 2 R_\Theta$ by \Cref{ass:convex_set}.
\end{proof}

We now control the quantity $\left\langle \Pi_{\Theta}\pr{\theta_{k-1} - \eta_k \nabla f(\theta_{k-1})} - \thetas, \boldsymbol{\epsilon}_k \right \rangle$ which appears in \Cref{lemma_convex_general}.

\begin{lemma}
  \label{lemma:convex_innerprod}
  Assume \Cref{ass:convex_set}, \Cref{ass:compression} and \Cref{ass:A_k_supp2}.
  Then, for any $k \in \N^*$, we have
    \begin{align}
      &\mathbb{E}\br{\left\langle \Pi_{\Theta}\pr{\theta_{k-1} - \eta_k \nabla f(\theta_{k-1})} - \thetas, \boldsymbol{\epsilon}_k \right \rangle} \le \sum_{i=1}^b\mathbb{E}\br{\left\langle \Pi_{\Theta}\pr{\theta_{k-1} - \eta_k \nabla f(\theta_{k-1})} - \thetas, \varepsilon_{\theta,k}^{(i)}\right \rangle} \eqsp,
    \end{align}
    where $\{\boldsymbol{\epsilon}_k\}_{k=1}^K$ is defined in \eqref{eq:epsilon_k}.
\end{lemma}

\begin{proof}
  Let $a_k = \Pi_{\Theta}\pr{\theta_{k-1} - \eta_k \nabla f(\theta_{k-1})} - \thetas$, $a_k^{(\phi)} = \Pi_{\Phi}\pr{\phi_{k-1} - \eta_k \nabla_{\phi} f(\theta_{k-1})} - \phi_\star$ and $a_k^{(\beta)} = \Pi_{\mathsf{B}}\pr{\beta_{k-1} - \eta_k \nabla_{\beta} f(\theta_{k-1})} - \beta_\star$.
  We have
  \begin{align}
    \left\langle a_k, \boldsymbol{\epsilon}_k \right \rangle 
    &= \left\langle a_k^{(\phi)},\sum_{i=1}^b \bbr{\mathscr{S}_i\br{\mathscr{C}_i\pr{\frac{1}{M}\sum_{m=1}^M H_{\phi_{k-1}}^{(i)}(Z_k^{(i,m)}),X_k^{(i,1)}},X_k^{(i,2)}} -
    \frac{1}{M}\sum_{m=1}^M H_{\phi_{k-1}}^{(i)}(Z_k^{(i,m)})} \right \rangle \\
    &+ \sum_{i=1}^b\left\langle a_k^{(\phi)},\frac{1}{M}\sum_{m=1}^M H_{\phi_{k-1}}^{(i)}(Z_k^{(i,m)}) - \nabla_{\phi}f_i(\theta_{k-1})\right \rangle \\
    &+ \sum_{i=1}^b\left\langle a_k^{(\beta)},\frac{1}{M}\sum_{m=1}^M H_{\beta_{k-1}}^{(i)}(Z_k^{(i,m)}) - \nabla_{\beta}f_i(\theta_{k-1})\right \rangle \\
    &= \left\langle a_k^{(\phi)},\sum_{i=1}^b \mathscr{S}_i\br{\mathscr{C}_i\pr{\frac{1}{M}\sum_{m=1}^M H_{\phi_{k-1}}^{(i)}(Z_k^{(i,m)}),X_k^{(i,1)}},X_k^{(i,2)}} -
    \frac{1}{M}\sum_{m=1}^M H_{\phi_{k-1}}^{(i)}(Z_k^{(i,m)}) \right \rangle \\
    &+ \sum_{i=1}^b\left\langle a_k, \varepsilon_{\theta,k}^{(i)}\right \rangle \eqsp, \label{eq:control_inner}
  \end{align}
  where the last line follows from \eqref{eq:epsilon_theta}.
  Using \Cref{ass:compression} and \Cref{ass:A_k_supp}, we have
  \begin{align}
    &\mathbb{E}\br{\left\langle a_k^{(\phi)},\sum_{i=1}^b \mathscr{S}_i\br{\mathscr{C}_i\pr{\frac{1}{M}\sum_{m=1}^M H_{\phi_{k-1}}^{(i)}(Z_k^{(i,m)}),X_k^{(i,1)}},X_k^{(i,2)}} -
    \frac{1}{M}\sum_{m=1}^M H_{\phi_{k-1}}^{(i)}(Z_k^{(i,m)}) \right \rangle} \\
    &= \mathbb{E}\br{\left\langle a_k^{(\phi)},\sum_{i=1}^b \mathbb{E}^{\mathcal{F}_{k-1}}\br{\mathscr{S}_i\bbr{\mathscr{C}_i\pr{\frac{1}{M}\sum_{m=1}^M H_{\phi_{k-1}}^{(i)}(Z_k^{(i,m)}),X_k^{(i,1)}},X_k^{(i,2)}} -
    \frac{1}{M}\sum_{m=1}^M H_{\phi_{k-1}}^{(i)}(Z_k^{(i,m)})} \right \rangle} \\
    &= 0 \eqsp.
  \end{align}
  The proof is concluded by taking the expectation in \eqref{eq:control_inner} and using the previous result.
\end{proof}

Similar to \citet[Appendix C.3]{SOUL}, we now decompose the Monte Carlo error terms $\{\varepsilon_{\theta,k}^{(i)}\}_{i \in [b], k \in [K]}$ in order to end up with an upper bound on $\sum_{k=1}^K \eta_{k}\{f(\theta_{k}) - f(\thetas)\} / (\sum_{k=1}^K \eta_{k})$ which vanishes when $\lim_{k \rightarrow \infty} \eta_k = 0_+$ and $\lim_{k \rightarrow \infty} \gamma_k = 0_+$.

For any $\theta \in \Theta$ and $\gamma \in (0,\bar{\gamma}]$, let for any $i\in [b]$, a function $\hat{H}^{(i)}_{\gamma,\theta}: \R^d \rightarrow \R^{d_\Theta}$ defined for any $z \in \R^d$ by
$$
\hat{H}^{(i)}_{\gamma,\theta}(z) = \sum_{j \in \N} \bbr{\br{R^{(i)}_{\gamma,\theta}}^jH^{(i)}_\theta(z) - \pi^{(i)}_{\gamma,\theta}(H^{(i)}_\theta)}\eqsp,
$$ 
where $R^{(i)}_{\gamma,\theta}$ is the Markov kernel associated with the discretized overdamped Langevin dynamics targetting $\pi^{(i)}_\theta$, and where $\pi^{(i)}_{\gamma,\theta}$ denotes the invariant distribution of $R^{(i)}_{\gamma,\theta}$.
By \Cref{ass:gradient_H} and \Cref{ass:markov_kernel}-\ref{ass:markov_kernel1}-\ref{ass:markov_kernel2}, for any $\theta \in \Theta$, $\gamma \in (0,\bar{\gamma}]$ and $i\in [b]$, $\hat{H}^{(i)}_{\gamma,\theta}$ is solution of the \emph{Poisson} equation defined by
\begin{align}
  (\mathrm{Id} - R^{(i)}_{\gamma,\theta})\hat{H}^{(i)}_{\gamma,\theta} = H_\theta - \pi^{(i)}_{\gamma,\theta}(H_\theta)\eqsp. \label{eq:Poisson_equation}
\end{align}

In addition, note that using \Cref{ass:markov_kernel}-\ref{ass:markov_kernel1} and \citet[Lemma 10]{SOUL}, it follows for any $\theta \in \Theta$, $i \in [b]$ and $z \in \R^d$ that
\begin{align}
  \norm{\hat{H}^{(i)}_{\gamma,\theta}(z)} \le C_{\hat{H}} \gamma^{-1} V^{1/4}(z)\eqsp, \label{eq:CH}
\end{align}
where $C_{\hat{H}} = 8 A_2 \log^{-1}(1/\rho)\rho^{-\bar{\gamma}/4}$.

Using \eqref{eq:Poisson_equation}, we can decompose the Monte Carlo error terms, for any $i \in [b], k \in [K]$ as $\varepsilon_{\theta,k}^{(i)} = (1/M)\sum_{m=1}^M\{\varepsilon_{\theta,k,m}^{(i,a)} + \varepsilon_{\theta,k,m}^{(i,b)} + \varepsilon_{\theta,k,m}^{(i,c)} + \varepsilon_{\theta,k,m}^{(i,d)}\}$ with, for any $m \in [M]$,
\begin{align}
  \varepsilon_{\theta,k,m}^{(i,a)} &= \hat{H}^{(i)}_{\gamma_{k-1},\theta_{k-1}}(Z_k^{(i,m)}) - R^{(i)}_{\gamma_{k-1},\theta_{k-1}}\hat{H}^{(i)}_{\gamma_{k-1},\theta_{k-1}}(Z_{k-1}^{(i,m)}) \label{eq:epsilon_a}\\
  \varepsilon_{\theta,k,m}^{(i,b)} &= R^{(i)}_{\gamma_{k-1},\theta_{k-1}}\hat{H}^{(i)}_{\gamma_{k-1},\theta_{k-1}} (Z_{k-1}^{(i,m)}) - R^{(i)}_{\gamma_{k},\theta_{k}}\hat{H}^{(i)}_{\gamma_{k},\theta_{k}}(Z_{k}^{(i,m)}) \\
  \varepsilon_{\theta,k,m}^{(i,c)} &= R^{(i)}_{\gamma_{k},\theta_{k}}\hat{H}^{(i)}_{\gamma_{k},\theta_{k}}(Z_{k}^{(i,m)}) - R^{(i)}_{\gamma_{k-1},\theta_{k-1}}\hat{H}^{(i)}_{\gamma_{k-1},\theta_{k-1}}(Z_{k}^{(i,m)}) \\
  \varepsilon_{\theta,k,m}^{(i,d)} &= \pi^{(i)}_{\gamma_{k-1},\theta_{k-1}}(H^{(i)}_{\theta_{k-1}}) - \pi^{(i)}_{\theta_{k-1}}(H^{(i)}_{\theta_{k-1}}) \eqsp.
\end{align}

The following lemmata aim at upper bounding these four error terms.

\begin{lemma}
  Assume \Cref{ass:convex_set}, \Cref{ass:function_f}, \Cref{ass:gradient_H} and \Cref{ass:markov_kernel}, and for any $\theta \in \Theta$, $z \in \R^d$ and $i\in[b]$, assume that $\|H^{(i)}_\theta(z)\| \le V^{1/4}(z)$. Then, for any $i \in [b]$, $m \in [M]$, $k \in \N^*$, we have
  \begin{align}
    \mathbb{E}\br{\norm{\varepsilon_{\theta,k,m}^{(i,a)}}^2} \le A_1 C_{\hat{H}}^2 \gamma_{k-1}^{-2}\mathbb{E}\br{V^{1/2}\pr{Z_0^{(i,m)})}}\eqsp,
  \end{align}
  where $C_{\hat{H}}$ is defined in \eqref{eq:CH}.
\end{lemma}

\begin{proof}
  The proof follows from \citet[Lemma 14]{SOUL}.
\end{proof}

\begin{lemma}
  \label{lemma:epsilon_b}
  Assume \Cref{ass:convex_set}, \Cref{ass:function_f}, \Cref{ass:markov_kernel} and for any $\theta \in \Theta$, $z \in \R^d$ and $i\in[b]$, assume that $\|H^{(i)}_\theta(z)\| \le V^{1/4}(z)$. Then, for any $i \in [b]$, $m \in [M]$, $k \in \N^*$, we have
  \begin{align}
    &\mathbb{E}\br{\norm{\sum_{k=1}^{K} \eta_{k} \langle \Pi_{\Theta}\pr{\theta_{k-1} - \eta_{k} \nabla f(\theta_{k-1})} - \thetas, \varepsilon_{\theta,k,m}^{(i,b)} \rangle }} \\
    &\le C_3^{(i,m)} \br{\sum_{k=1}^{K} |\eta_{k} - \eta_{k-1}|\gamma_{k-1}^{-1} + \sum_{k=1}^{K} \eta_{k}^2 \gamma_{k-1}^{-1} + \eta_K/\gamma_K - \eta_1/\gamma_1} \eqsp,
  \end{align}
  where, for any $i \in [b]$ and $m \in [M]$,
  \begin{equation}
    C_3^{(i,m)} = A_1 C_{\hat{H}} (2R_\Theta(2 + L_f) + 1 + \eta_1 L_f)\mathbb{E}\br{V^{1/4}(Z_0^{(i,m)})}\eqsp. \label{eq:C3} 
  \end{equation}
\end{lemma}

\begin{proof}
  The proof follows from \citet[Lemma 15]{SOUL}.
\end{proof}

\begin{lemma}
  \label{lemma:epsilon_c}
  Assume  \Cref{ass:convex_set}, \Cref{ass:function_f}, \Cref{ass:gradient_H}, \Cref{ass:markov_kernel} and \Cref{ass:markov_kernel_bis}.
  In addition, for any $\theta \in \Theta$, $z \in \R^d$ and $i\in[b]$, assume that $\|H^{(i)}_\theta(z)\| \le V^{1/4}(z)$.
  Then, for any $i \in [b]$, $m \in [M]$, $k \in \N^*$, we have
  \begin{align}
    \mathbb{E}\br{\norm{\varepsilon_{\theta,k,m}^{(i,c)}}} \le A_1 \mathbb{E}\br{V(Z_0^{(i,m)})} C_{c,2} \gamma_{k}^{-1}\br{\gamma_k^{-1}\bbr{\boldsymbol{\Gamma}_1(\gamma_{k-1},\gamma_k) + \boldsymbol{\Gamma}_2(\gamma_{k-1},\gamma_k)\eta_k} + \eta_k}\eqsp,
  \end{align}
  where 
  \begin{align}
    C_{c,2} &= 4 A_2 \log^{-1}(1/\rho)\rho^{-\bar{\gamma}/2} \max\{L_H C_{c,1} + 2A_2 \log^{-1}(1/\rho)\rho^{-\bar{\gamma}/2}\} \eqsp, \\
    C_{c,1} &= 4 A_1 A_2 \log^{-1}(1/\rho)\rho^{-\bar{\gamma}/2} \mathbb{E}\br{V(Z_0^{(i,m)})} \eqsp.
  \end{align}
\end{lemma}

\begin{proof}
  The proof follows from \citet[Lemma 16]{SOUL}.
\end{proof}

\begin{lemma}
  \label{lemma:epsilon_d}
  Assume \Cref{ass:convex_set}, \Cref{ass:function_f}, \Cref{ass:markov_kernel} and for any $\theta \in \Theta$, $z \in \R^d$ and $i\in[b]$, assume that $\|H^{(i)}_\theta(z)\| \le V^{1/4}(z)$. Then, for any $i \in [b]$, $m \in [M]$, $k \in \N^*$, we have
  \begin{align}
    \mathbb{E}\br{\norm{\varepsilon_{\theta,k,m}^{(i,d)}}} \le \boldsymbol{\Psi}(\gamma_{k-1})\eqsp.
  \end{align}

\end{lemma}

\begin{proof}

  The proof follows from \citet[Lemma 17]{SOUL}.

\end{proof}

\section{Application to \texttt{FedSOUL}}
\label{sec:fedsoul}

We now apply \Cref{thm:general} to \texttt{FedSOUL} where for any $i \in [b]$, $\gamma \in (0,\bar{\gamma}]$ and $\theta \in \Theta$, the Markov kernel $Q^{(i)}_{\gamma,\theta}$ is associated with a Gaussian probability density function $q^{(i)}_{\gamma,\theta}(z^{(i)},\cdot)$ with mean $z^{(i)} - \gamma \nabla_{z} \log p(z^{(i)} \mid \mathrm{D}_i, \theta)$ and variance $2\gamma \mathrm{I}_d$.
To this end, we show explicit conditions on the family of posterior distributions $\{\pi_\theta^{(i)}\}_{i \in [b]}$ such that \Cref{ass:markov_kernel} and \Cref{ass:markov_kernel_bis} are satisfied.

\subsection{Assumptions}

For any $i \in [b]$, let $U_\theta^{(i)}: \Rd \to \R$ such that for any $z^{(i)} \in \Rd$, $\pi_\theta^{(i)}(z^{(i)}) \propto \exp\{-U_\theta^{(i)}(z^{(i)})\}$. In our case, this boils down to set $U_\theta^{(i)}(z^{(i)}) = - \log p(z^{(i)} \mid \mathrm{D}_i, \phi, \beta)$ for any $z^{(i)} \in \Rd$.

\begin{assumptionsup}\label{ass:posterior_densities}
  For any $i\in[b]$, the following conditions hold.
  \begin{enumerate}[wide, labelwidth=!, labelindent=0pt,label=(\roman*),noitemsep,nolistsep]

    \item \label{ass:posterior_densities_1} Assume that $(\theta,z^{(i)}) \mapsto U_\theta(z^{(i)})$ is continuous, $z^{(i)} \mapsto U_\theta^{(i)}(z^{(i)})$ is differentiable for any $\theta_1,\theta_2 \in \Theta$ and there exists $\texttt{L} \ge 0$ such that for any $z_1,z_2 \in \Rd$,
    $$
    \sup_{\theta \in \Theta}\norm{\nabla_{z} U_{\theta}^{(i)}(z_2) - \nabla_{z} U_{\theta}^{(i)}(z_1)} \le \texttt{L} \norm{\theta_2 - \theta_1}\eqsp,
    $$
    and $\{\nabla_{z} U_{\theta}^{(i)}(0)\, :\, \theta \in \Theta\}$ is bounded.

    \item \label{ass:posterior_densities_2} There exist $\texttt{m}_1, \texttt{m}_2 >0$ and $\texttt{c},R \ge 0$ such that for any $\theta \in \Theta$ and $z \in \Rd$,
    $$
    \langle \nabla_{z} U_{\theta}^{(i)}(z),z \rangle \ge \texttt{m}_1\norm{z}\mathbf{1}_{\mathrm{B}(0,R)^c}(z) + \texttt{m}_2\norm{\nabla_{z} U_{\theta}^{(i)}(z)}^2 - \texttt{c}\eqsp.
    $$

    \item \label{ass:posterior_densities_3} There exists $L_U \ge 0$ such that $z \in \Rd$ and $\theta_1,\theta_2 \in \Theta$,
    $$
    \norm{\nabla_{z} U_{\theta_2}^{(i)}(z) - \nabla_{z} U_{\theta_1}^{(i)}(z)} \le L_U \norm{\theta_2 - \theta_1}V(z)^{1/2}\eqsp,
    $$
    where $V: \Rd \to \R$ is defined under \Cref{ass:posterior_densities}-\ref{ass:posterior_densities_2}, for any $z \in \Rd$, as
    \begin{equation}
      \label{eq:function_V}
      V(z) = \exp\bbr{\texttt{m}_1\sqrt{1+\norm{z}^2}/4}\eqsp.
    \end{equation}
    \end{enumerate}
\end{assumptionsup}

\subsection{Verification of \Cref{ass:markov_kernel} and \Cref{ass:markov_kernel_bis}}

\begin{lemma}
  \label{lemma:convex_error_term_1}
  Assume \Cref{ass:posterior_densities}.
  Then, \Cref{ass:markov_kernel} and \Cref{ass:markov_kernel_bis} are satisfied with $V$ defined in \eqref{eq:function_V} and 
    \begin{align}
      &\bar{\gamma} < \min\{1,2\mathtt{m}_2\} \eqsp, \label{eq:bar_gamma}\\
      &\tilde{\mathtt{m}}_1 = \mathtt{m}_1 /4 \eqsp, \label{eq:tilde_m1}\\ 
      &b = \tilde{\mathtt{m}}_1(d + \mathtt{c} + \sqrt{2}\tilde{\mathtt{m}}_1)\exp(\tilde{\mathtt{m}}_1^2\{(d+\mathtt{c}+\tilde{\mathtt{m}}_1\bar{\gamma} + \sqrt{1+\mathtt{r}^2}\}) \eqsp, \label{eq:b}\\
      &\lambda = \exp(-\tilde{\mathtt{m}}_1^2[\sqrt{2}-1]) \eqsp, \label{eq:lambda}\\
      &\mathtt{r} = \max\{1,2(d+\mathtt{c})/\mathtt{m}_1,R\} \label{eq:r}\eqsp, \\
      &\boldsymbol{\Gamma}_1: (\gamma_{1},\gamma_2) \mapsto \gamma_1/\gamma_2 - 1\eqsp, \\
      &\boldsymbol{\Gamma}_2: (\gamma_{1},\gamma_2) \mapsto \gamma_2^{1/2}\eqsp, \\
      &\boldsymbol{\Psi}: \gamma \mapsto 2 C (1-\xi)^{-1} \gamma^{1/2} \tilde{D}_1^{1/2}(1+\bar{\gamma})^{1/2} \bbr{d + 2 \bar{\gamma}\pr{\mathtt{L}^2 M_V + \sup_{\theta \in \Theta, i \in [b]}\norm{\nabla_z U^{(i)}_\theta(0)}^2}\tilde{D}_1}^{1/2}\mathtt{L}\eqsp, \\
      &\tilde{D}_1 = \frac{\sqrt{2}\mathtt{\tilde{m}}_1\exp(\mathtt{\tilde{m}}_1\sqrt{1 + \max\{1,R\}^2})(1 + \mathtt{\tilde{m}}_1 + \mathtt{c} + d)}{3 \mathtt{\tilde{m}}_1^2} + b \lambda^{-\bar{\gamma}}\log^{-1}(1/\lambda)\eqsp,
  \end{align}
  with $M_V = \sup_{z \in \R^d} \{(1 + \|z\|)^2/V(z)\}$, $C \ge 0$, $\xi \in (0,1)$.
\end{lemma}

\begin{proof}
  The proof follows from \citet[Theorem 5]{SOUL}.
\end{proof}

\newpage
\section{Additional Experiments}

In this section, we provide additional experiments.
All the experimental details can be found in the ``code'' folder in the supplement.

\subsection{Synthetic datasets}

In this section, following the experiments from the main paper, we will show additional configurations of the toy example. 
We still use the same model (see Section~\ref{sec:experiments} and~ \citet{NEURIPS2021_5d44a2b0,pmlr-v139-collins21a}), but we choose different values of $(d, k, b)$.
First, let us test, how the total number of clients $b$ impacts the performances of the different approaches.
\Cref{fig:toy1} and \Cref{fig:toy2} depict our results for $b \in \{50,200\}$, with the size of minimal dataset being 5 and the share of clients with the minimal dataset 90\%.
We can see that in both cases, \texttt{FedSOUL} outperforms its competitors.

\begin{figure}
 \begin{center}
   \includegraphics[scale=0.23]{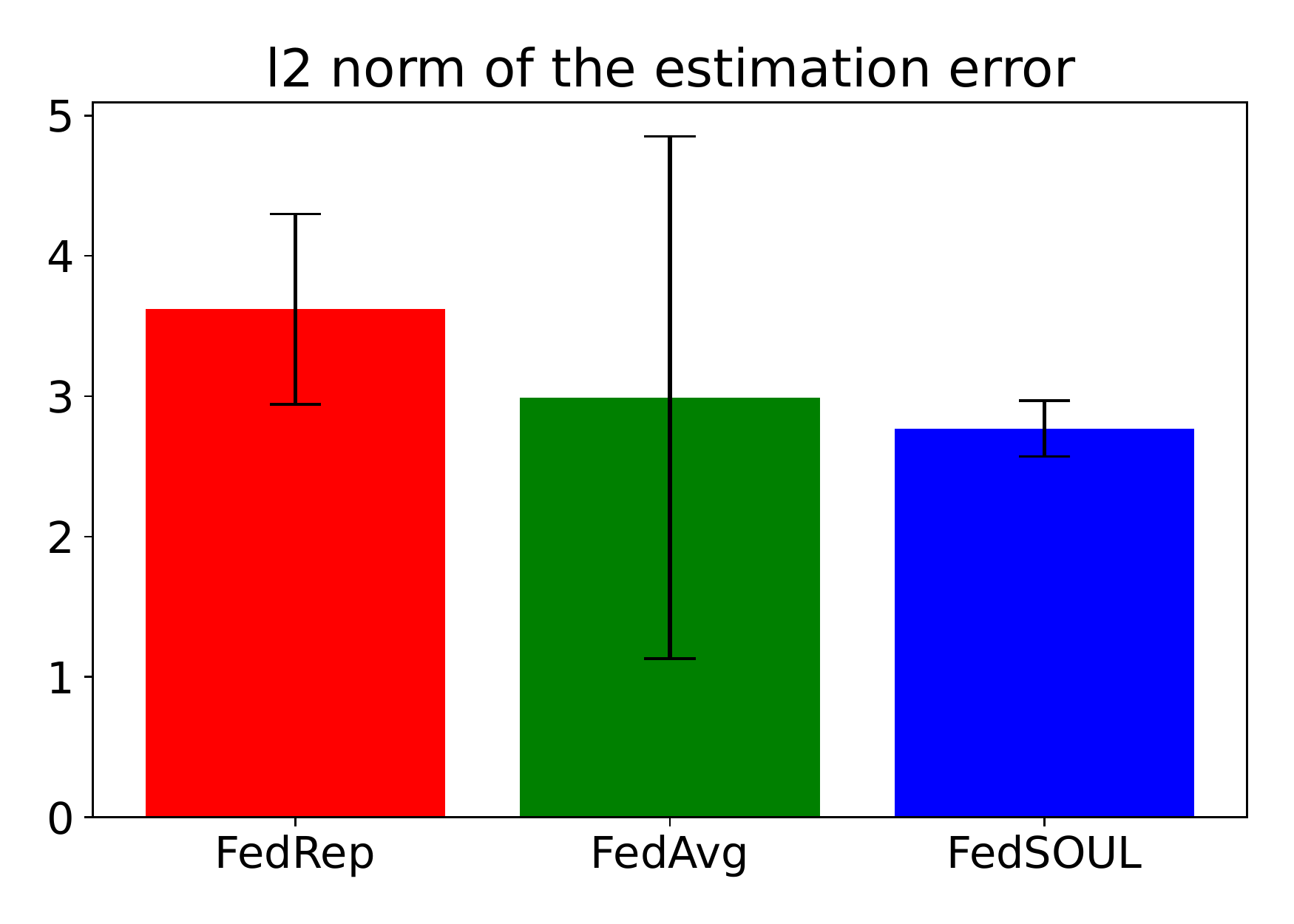}
   \includegraphics[scale=0.23]{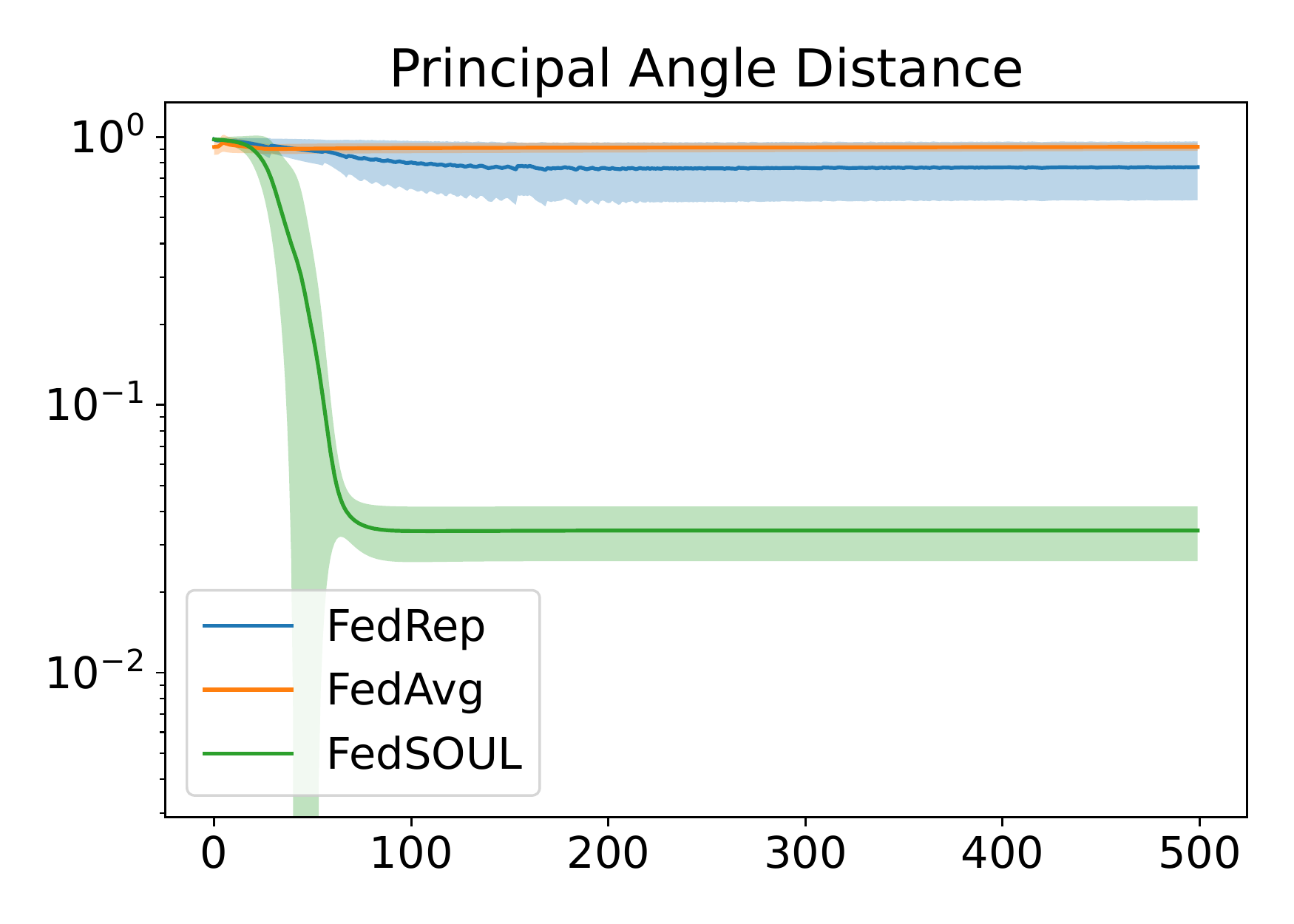}
 \end{center}
 \caption{Small data sets - synthetic data. $b=50$ clients.}
 \label{fig:toy1}
 \begin{center}
   \includegraphics[scale=0.23]{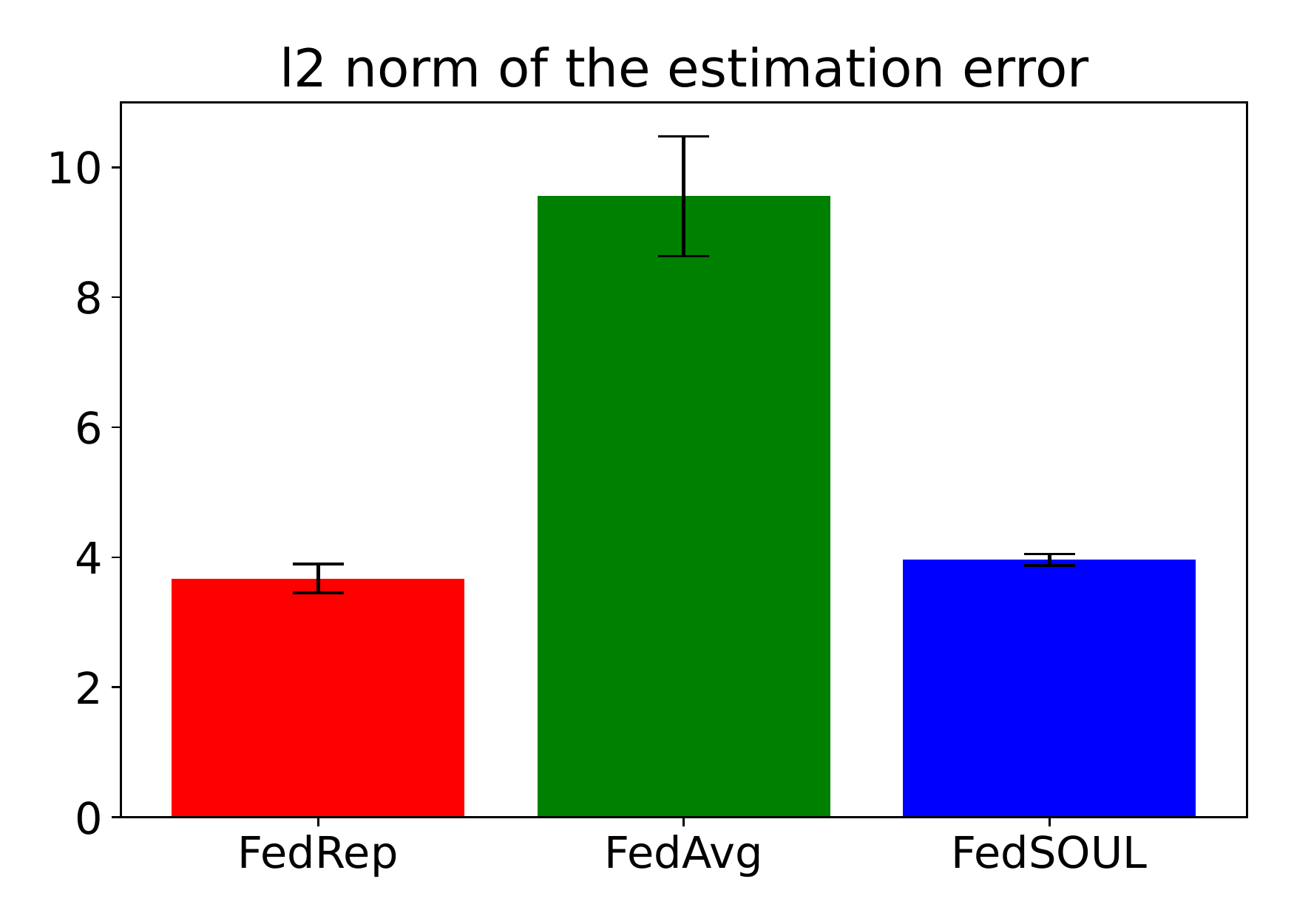}
   \includegraphics[scale=0.23]{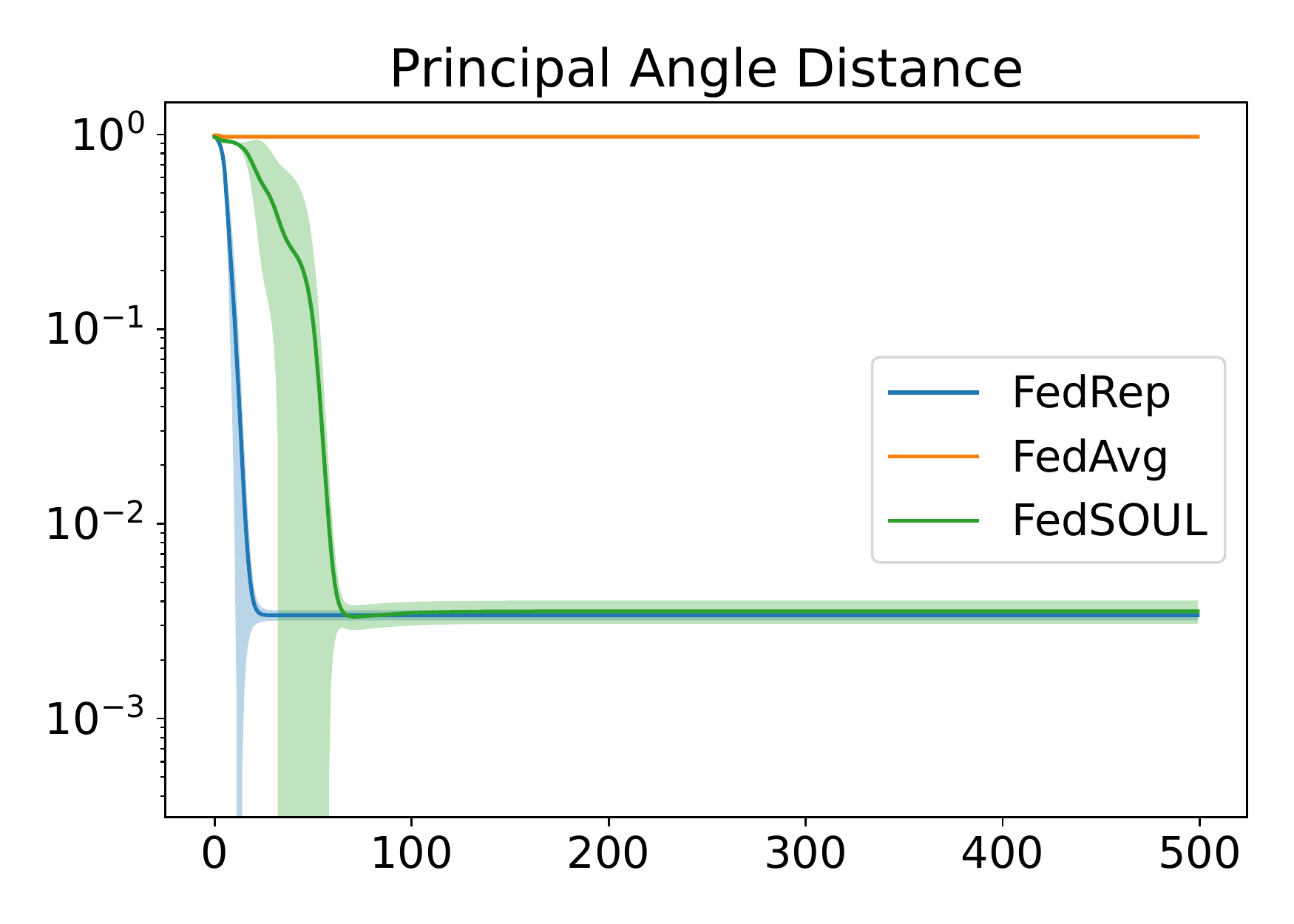}
 \end{center}
 \caption{Small data sets - synthetic data. $b=200$ clients.}
 \label{fig:toy2}
\end{figure}

Second, we test, how the dimensionality of raw data impacts on the result. 
\Cref{fig:toy3} and \Cref{fig:toy4} show our results with $k \in \{5,50\}$. All others parameters are the same as before.

\begin{figure}
 \begin{center}
   \includegraphics[scale=0.23]{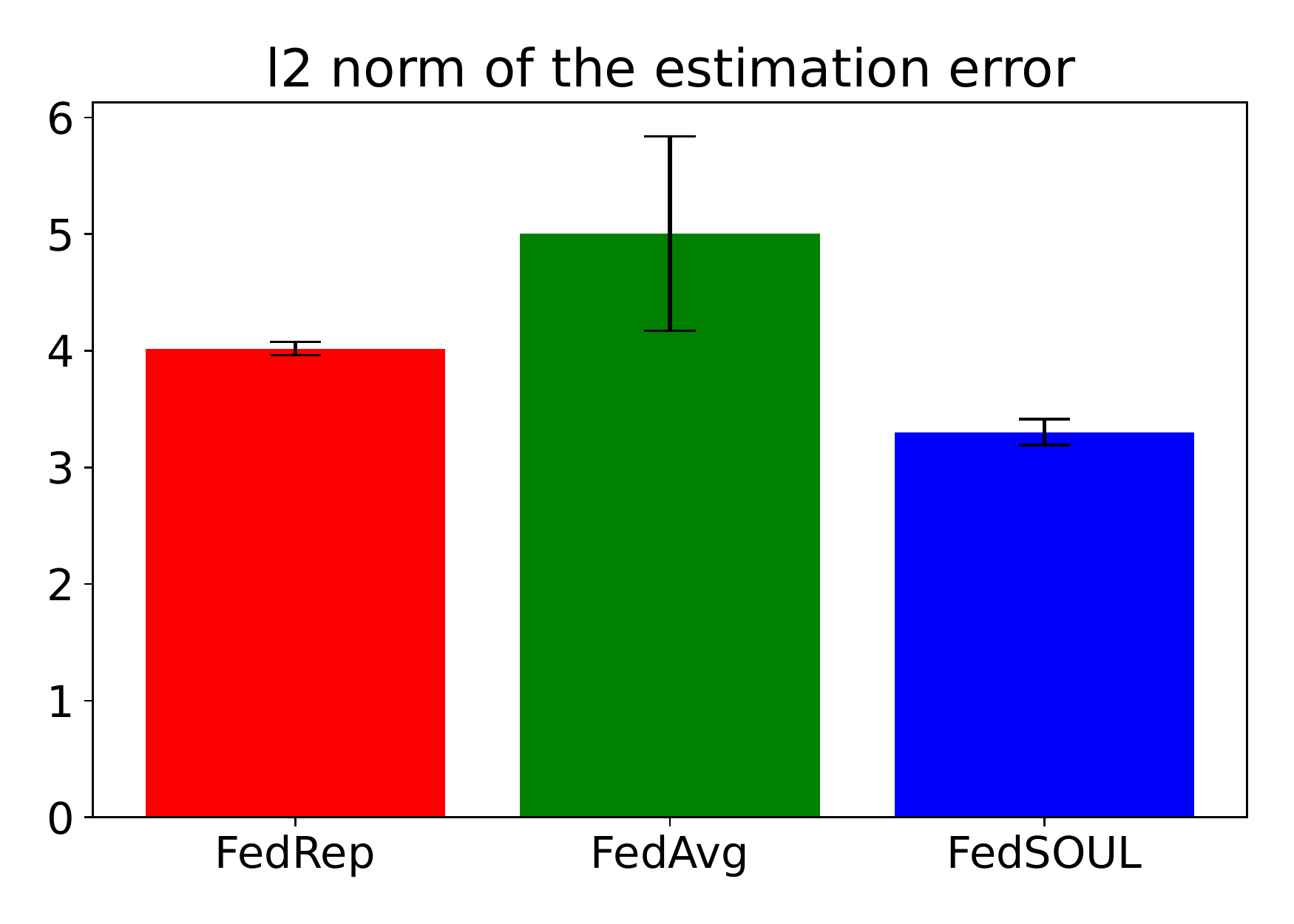}
   \includegraphics[scale=0.23]{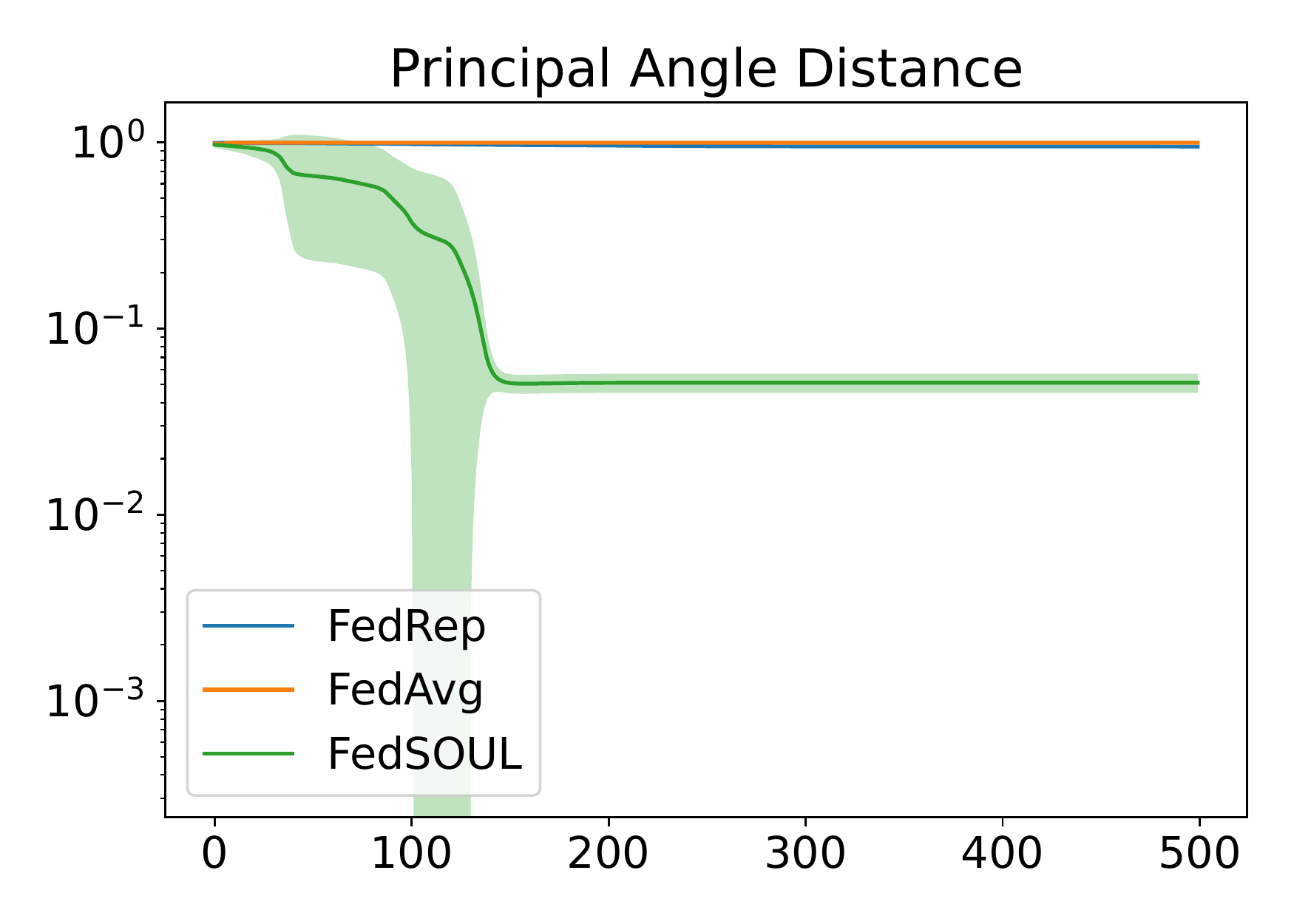}
 \end{center}
 \caption{Small data sets - synthetic data. Raw data dimensionality is $k=50$.}
 \label{fig:toy3}
 \begin{center}
   \includegraphics[scale=0.23]{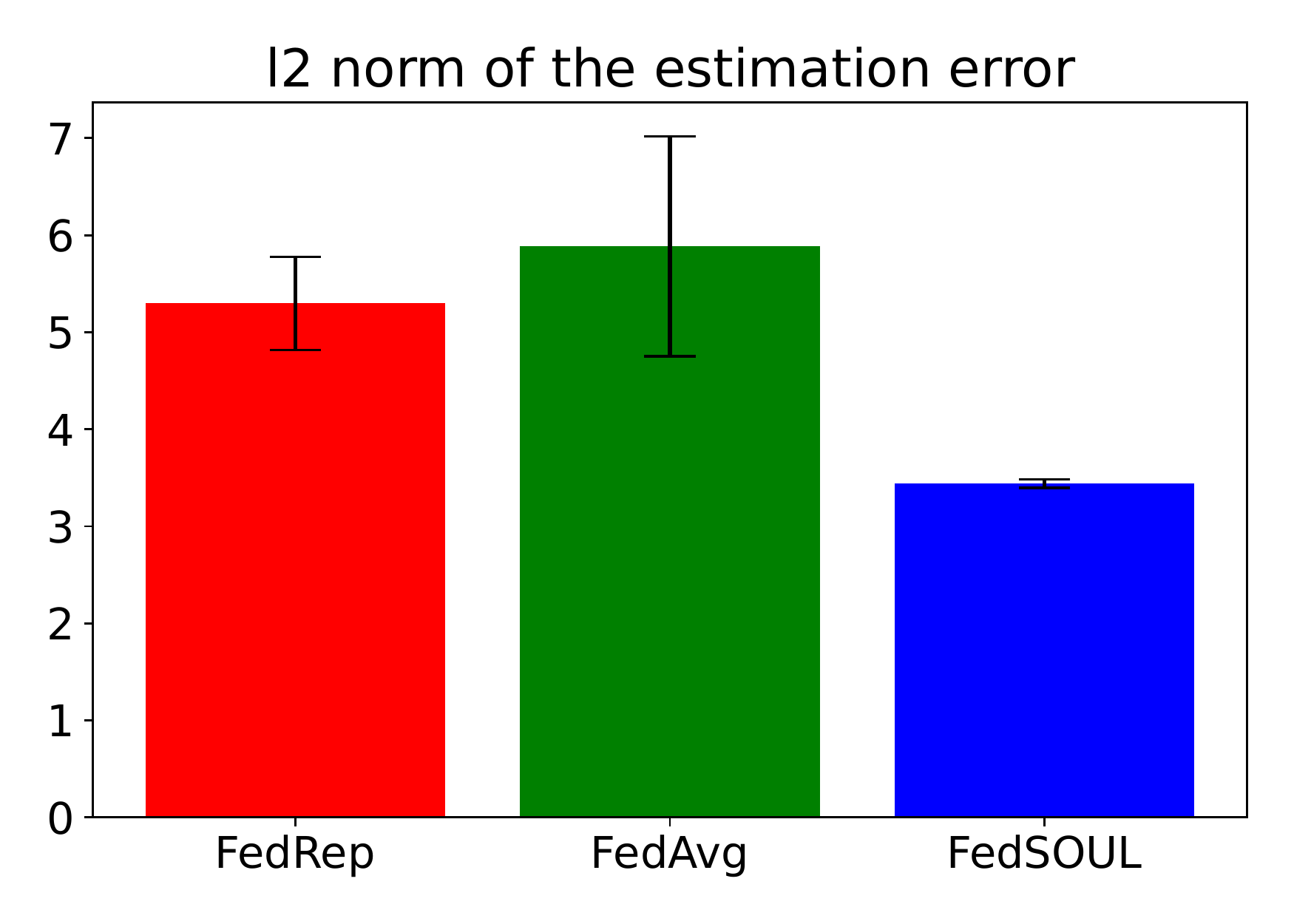}
   \includegraphics[scale=0.23]{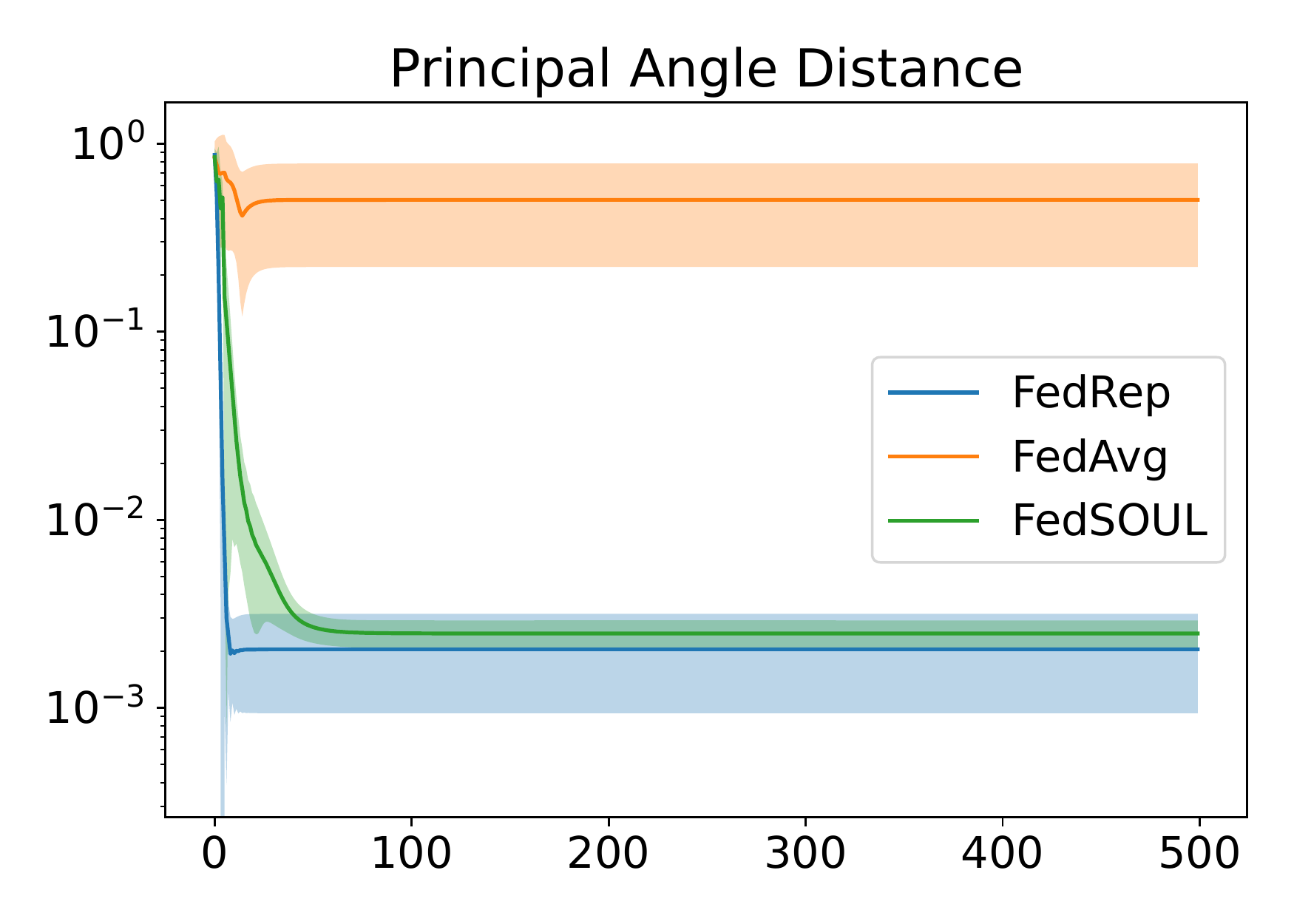}
 \end{center}
 \caption{Small data sets - synthetic data. Raw data dimensionality is $k=5$.}
 \label{fig:toy4}
\end{figure}

One more experiment we conducted is the dependence on latent dimensionality $d$.
We test two options $d=2$ (as in original experiments) and $d=5$ in \Cref{fig:toy5} and \Cref{fig:toy6}.
Again, the more parameters we have to learn (given the same small data budget), the better Bayesian methods (\emph{i.e.} \texttt{FedSOUL}) are better.

\begin{figure}
 \begin{center}
   \includegraphics[scale=0.23]{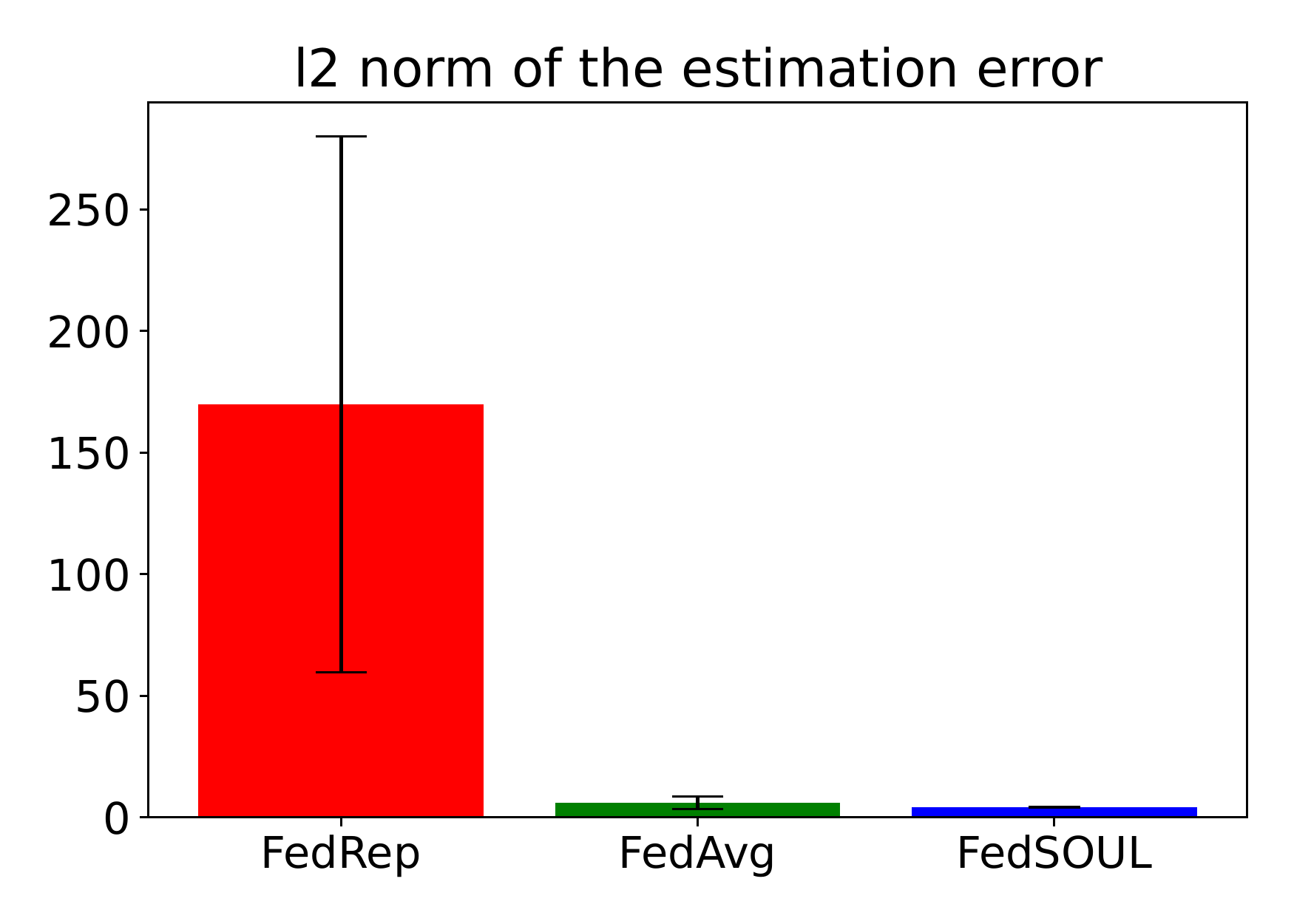}
   \includegraphics[scale=0.23]{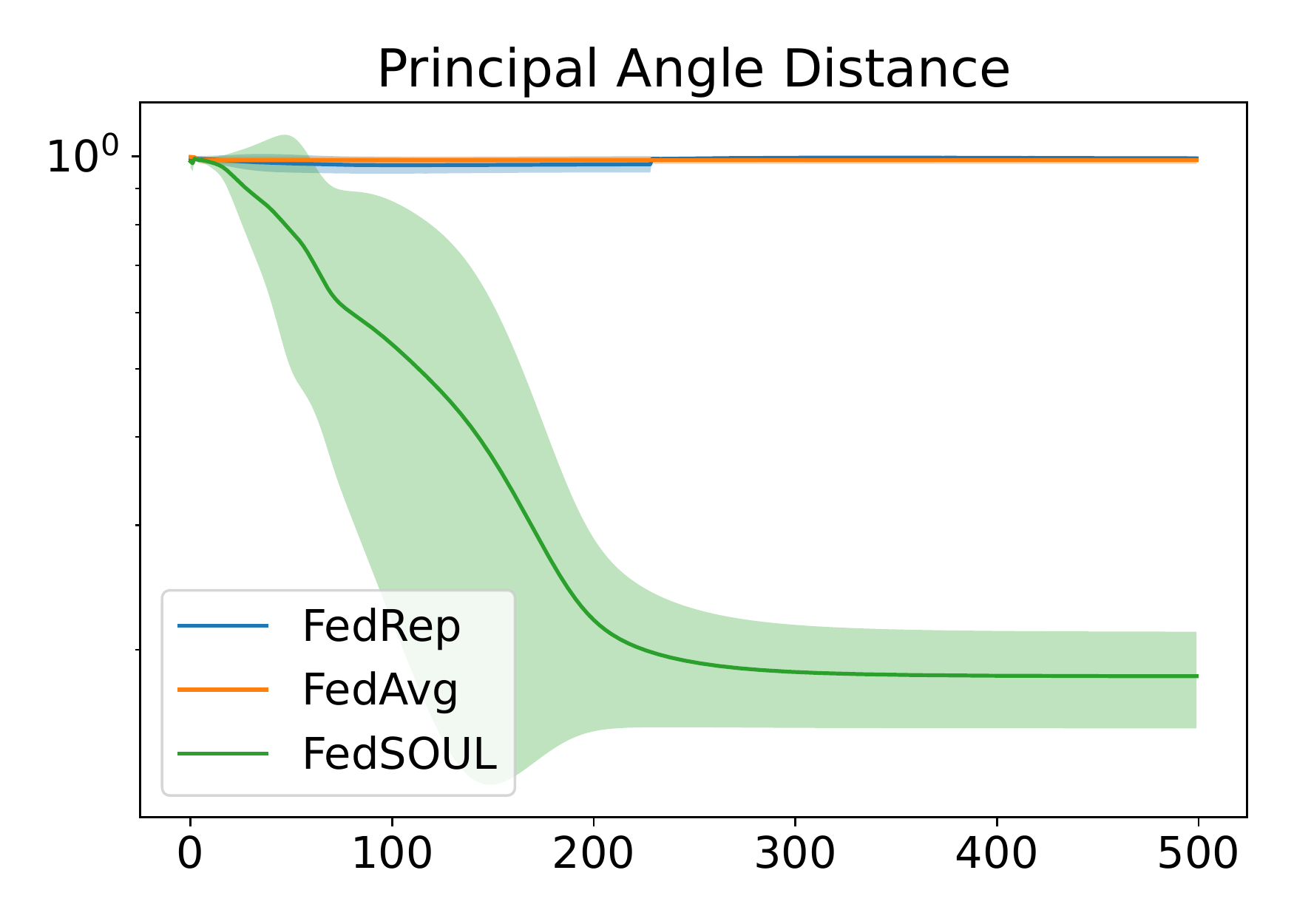}
 \end{center}
 \caption{Small data sets - synthetic data. Latent space dimensionality is $d=5$.}
 \label{fig:toy5}
 \begin{center}
   \includegraphics[scale=0.23]{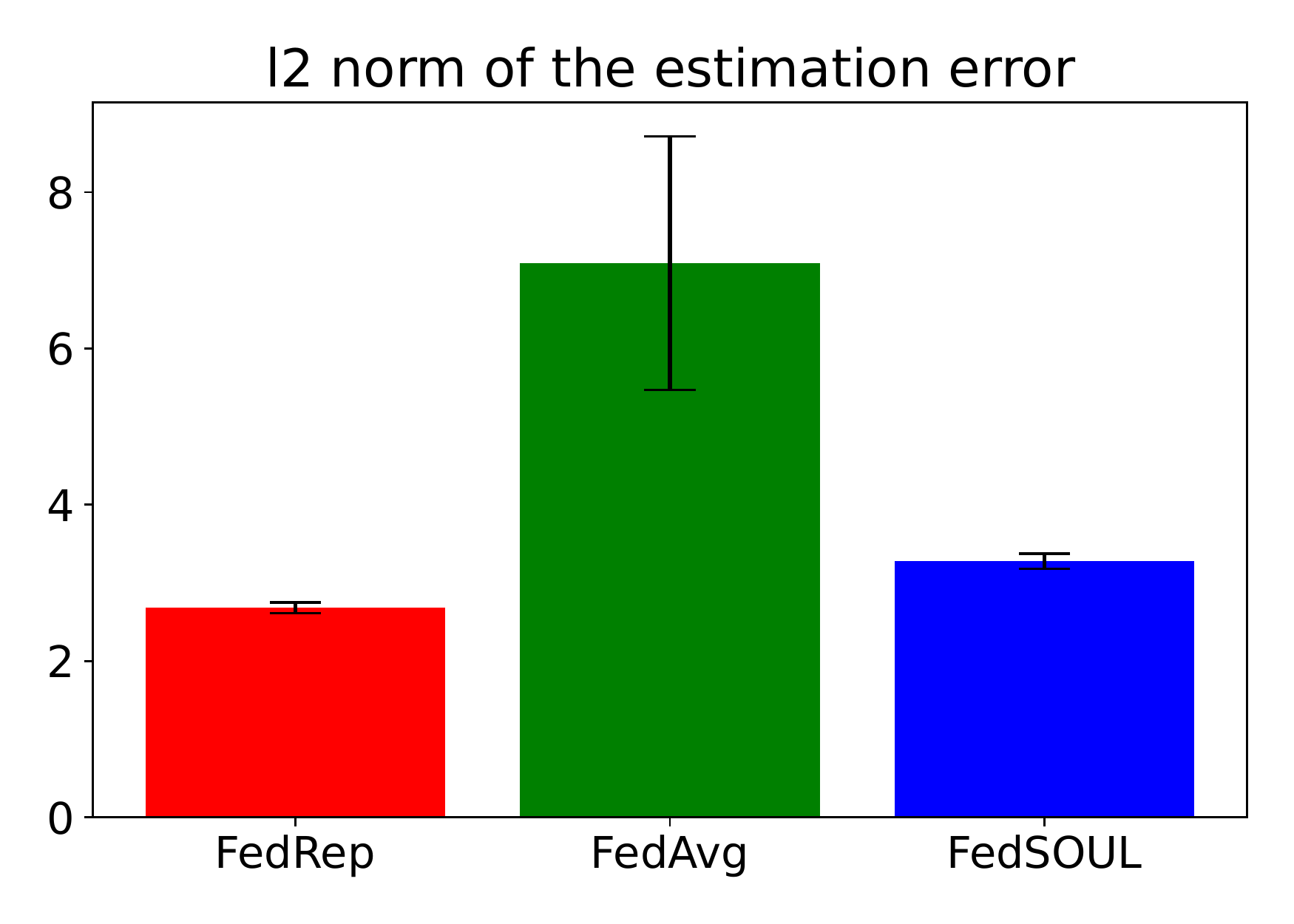}
   \includegraphics[scale=0.23]{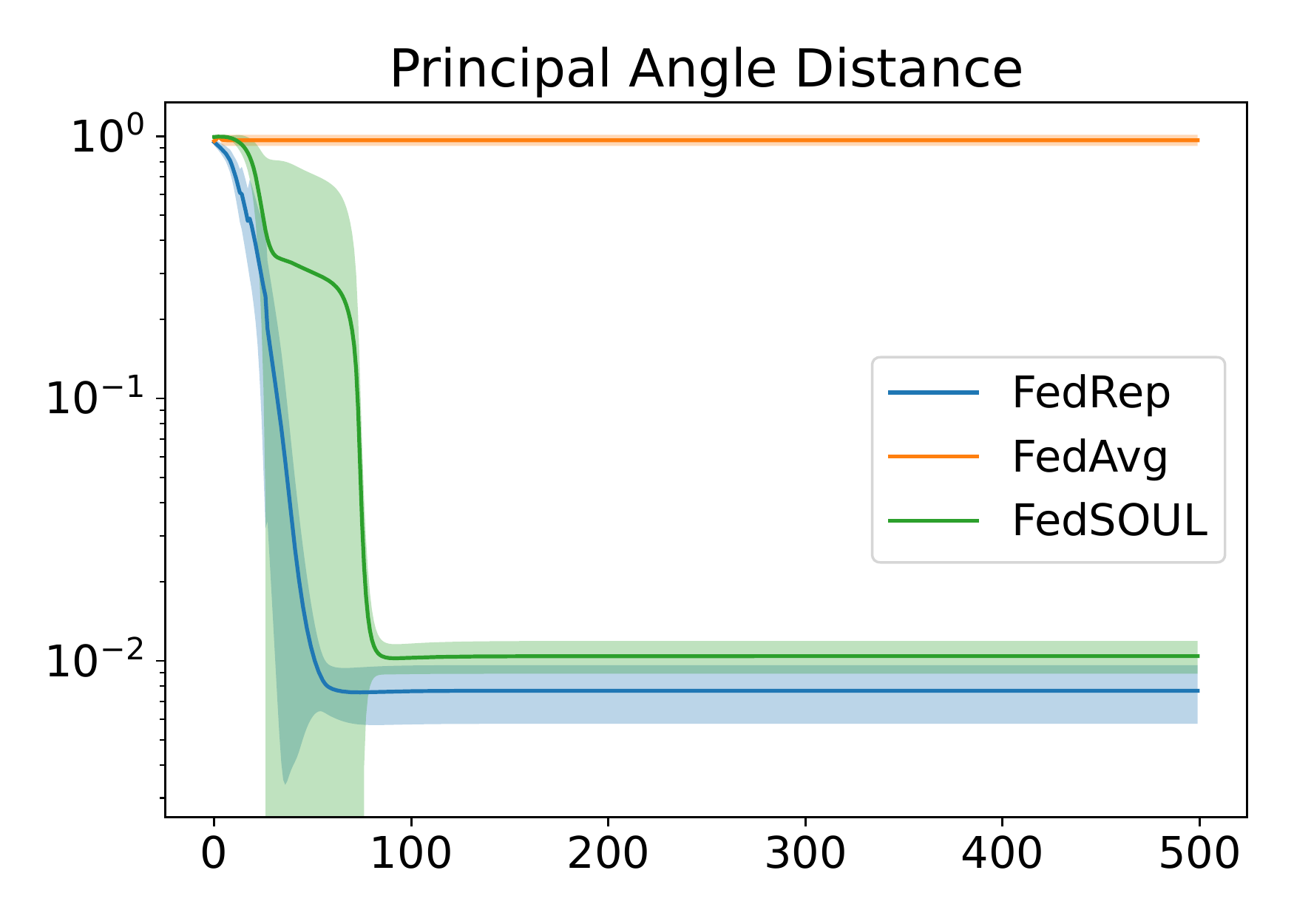}
 \end{center}
 \caption{Small data sets - synthetic data. Latent space dimensionality is $d=2$.}
 \label{fig:toy6}
\end{figure}

\subsection{Image datasets classification}
In this section we provide an additional baseline for the experiments with personalization, in case we have only a few heterogeneous data.
Specifically, we consider \texttt{APFL} \citep{deng2021adaptive} which is another personalized federated learning approach. 
We consider CIFAR-10 dataset with 100 clients.
Among these clients, there are 10, 50 or 90 which have local dataset of either 5 (one setup) or 10 (another setup). Else of size 25.

We see in \Cref{fig:toy7} that \texttt{FedSOUL} typically performs better than \texttt{FedRep}, but on par with \texttt{APFL}.
It is surprizing, that \texttt{APFL} is a very good baseline in these type of problem, which it was not specially designed for.

\begin{figure}
 \begin{center}
   \includegraphics[scale=0.23]{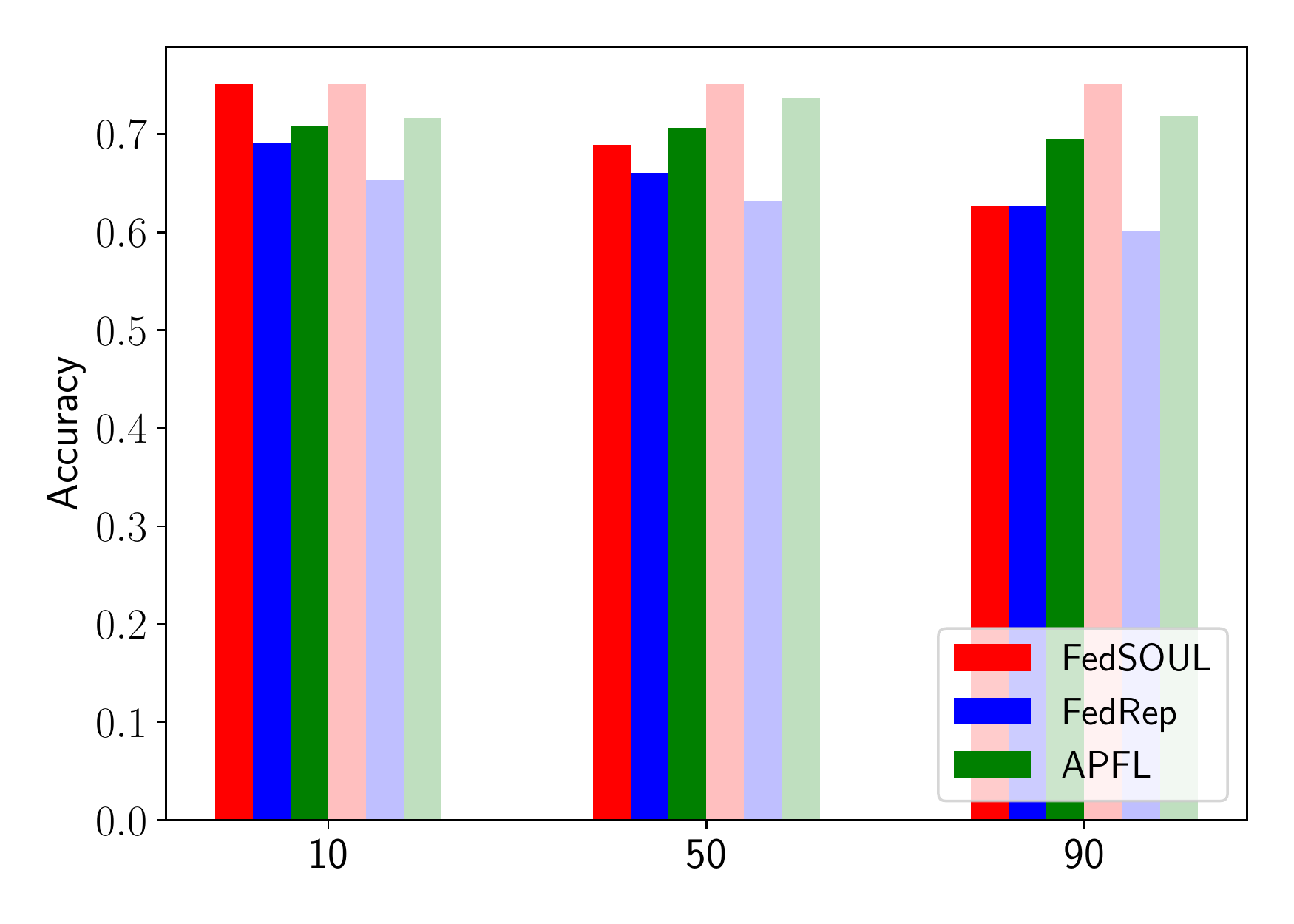}
   \includegraphics[scale=0.23]{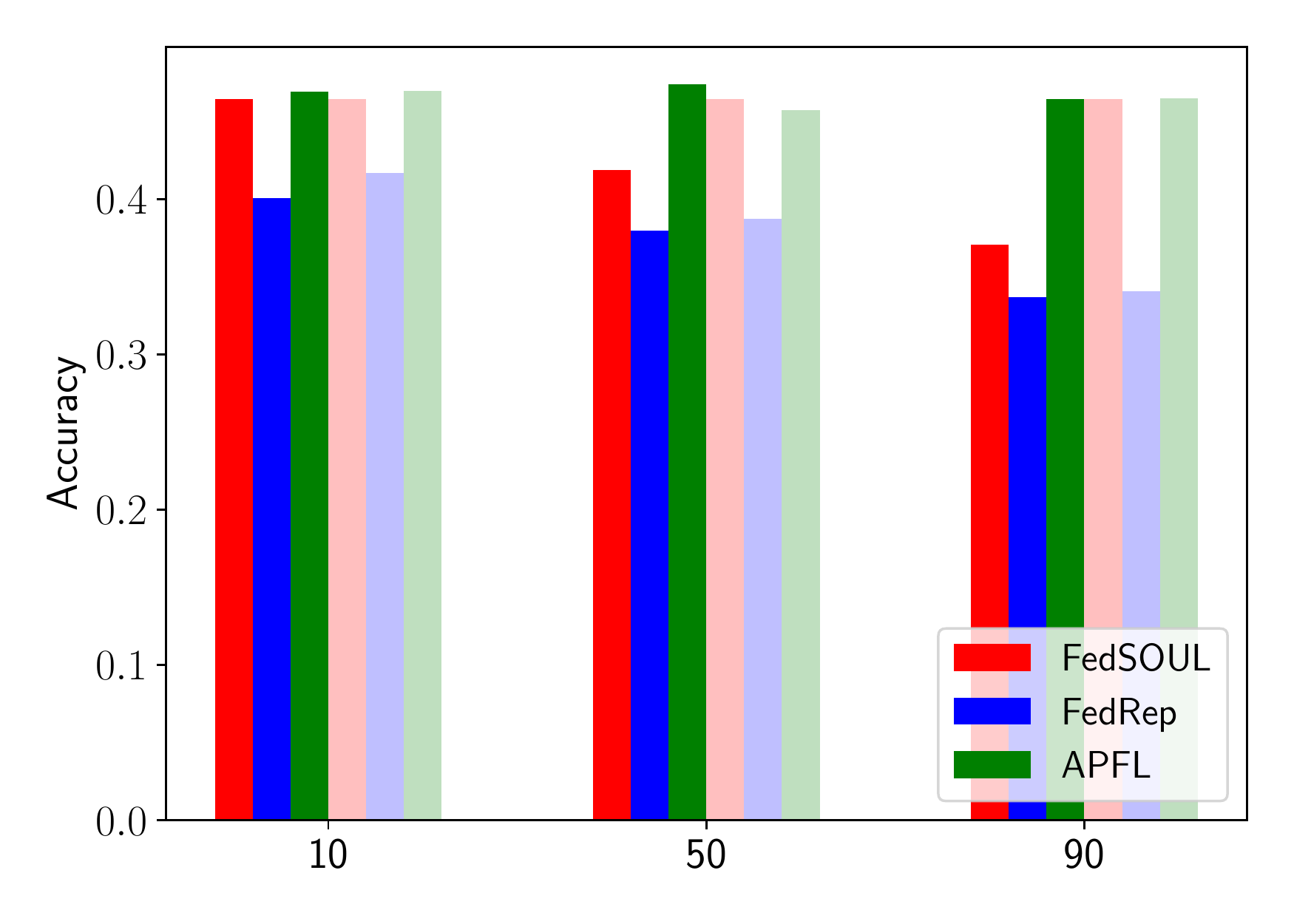}
 \end{center}
 \caption{Small image datasets. Minimal local dataset size is 2 (top) or 5 (bottom).}
 \label{fig:toy7}
\end{figure}

\subsection{Image datasets uncertainty quantification}

In this section, we provide additional experiments on image uncertainty with CIFAR-10 (in distribution) and SVHN (out of distribution) datasets.
As a measure of uncertainty, we will use predictive entropy. 
On \Cref{fig:OOD}, we present 4 different models among 100.
In the left part of the figure we see the distribution of entropy, assigned to the in-distribution objects (validational split, but same domain as training data).
In the right part we see the distribution for out-of-distribution (SVHN in our case).
Contraty to MNIST vs Fashion-MNIST example, here it is not that clear that \texttt{FedSOUL} captures uncertainty well.

\begin{figure}
 \begin{center}
   \includegraphics[scale=0.23]{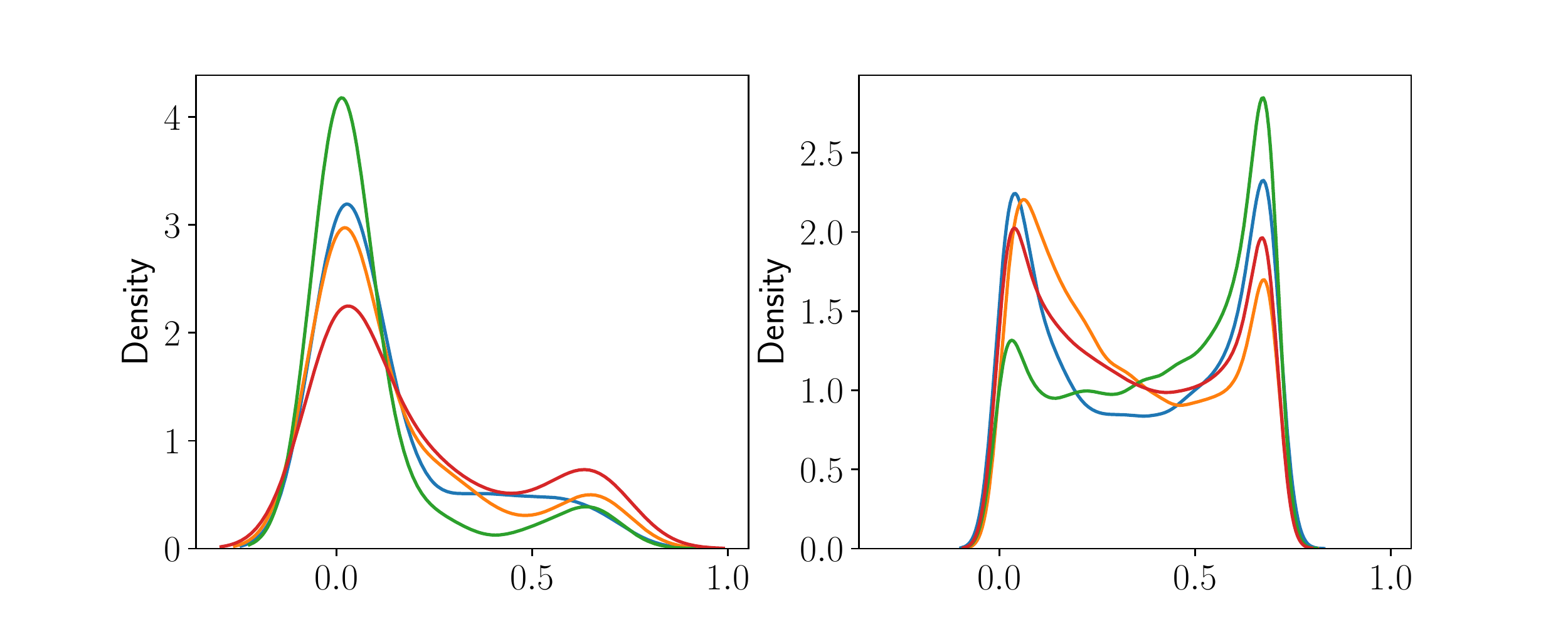}
   \includegraphics[scale=0.23]{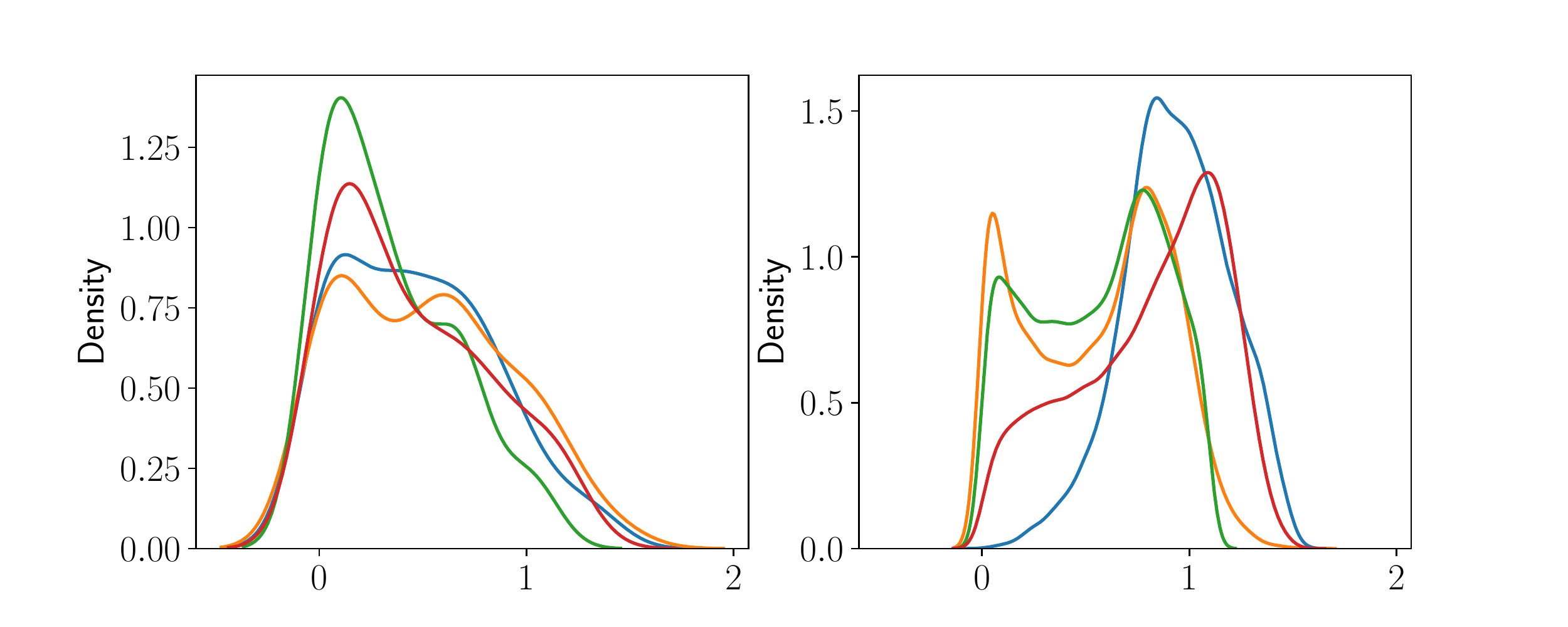}
 \end{center}
 \caption{Out-of-distribution detection. CIFAR 10 vs SVHN. 2 classes for model (top) and 5 (bottom).}
\label{fig:OOD}
\end{figure}

We also provide additional plots for calibration on CIFAR-10 again for two cases, when each client had 2 classes to predict or 5.

\begin{figure}
 \begin{center}
   \includegraphics[scale=0.23]{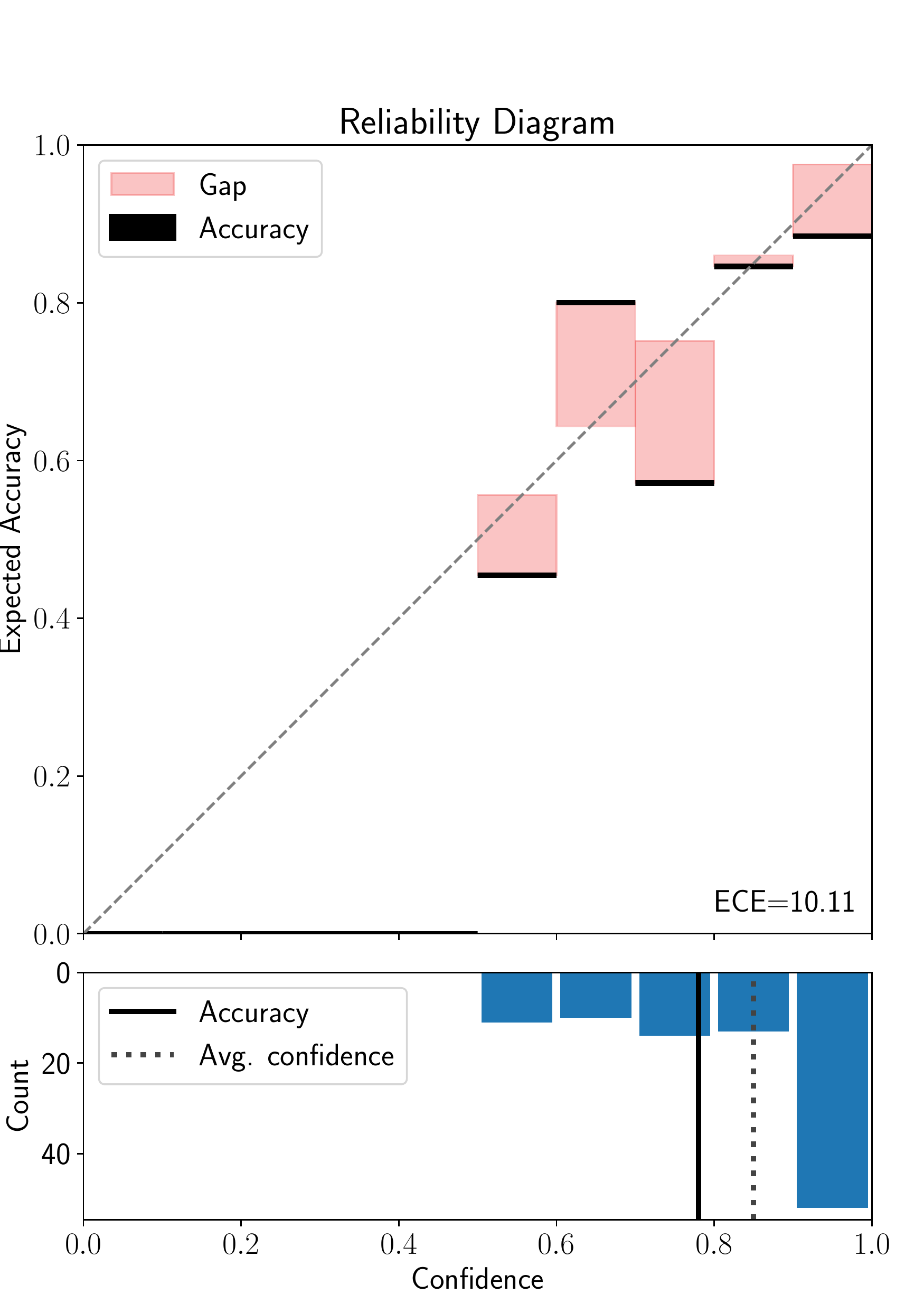}
   \includegraphics[scale=0.23]{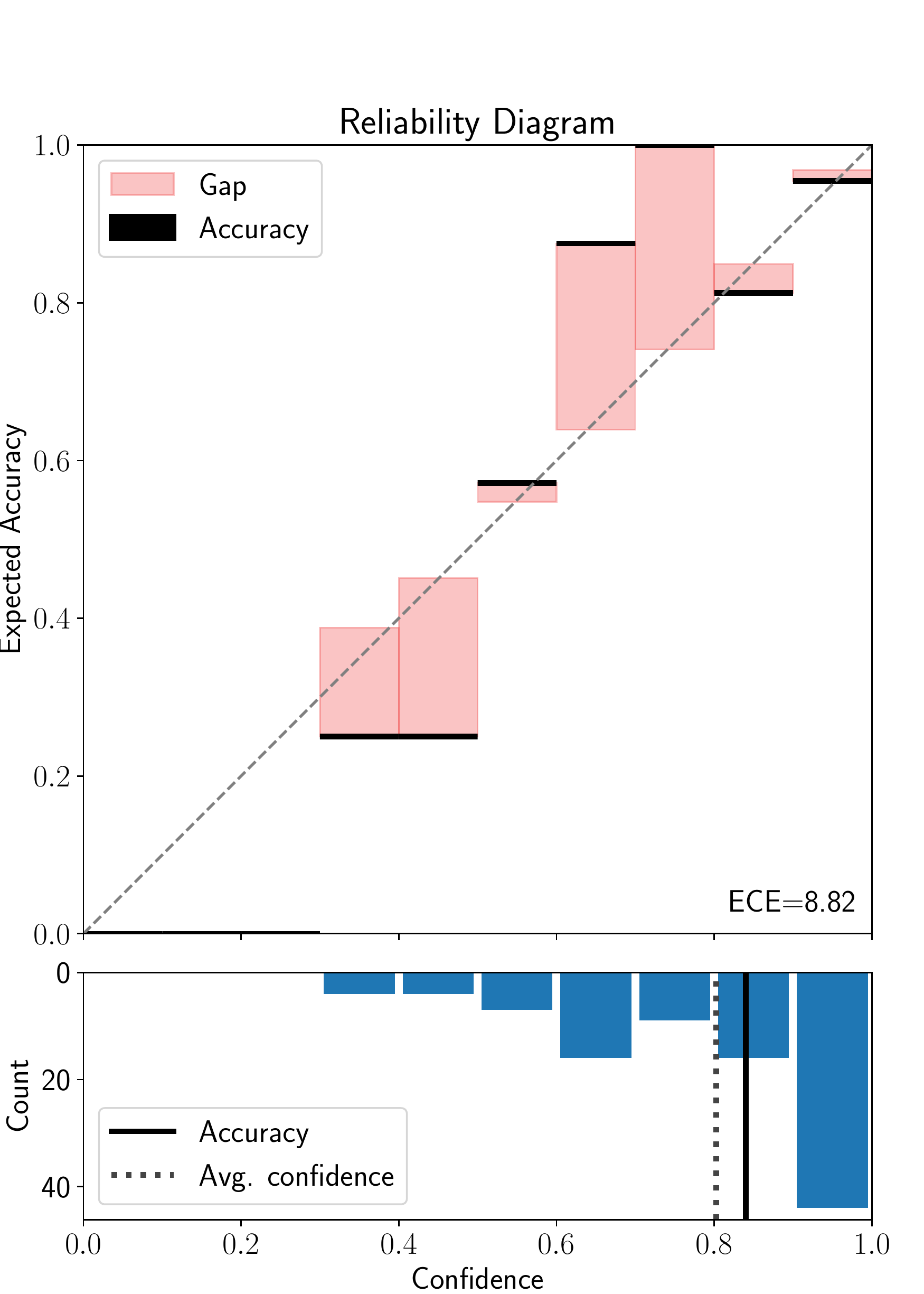}
 \end{center}
 \caption{Reliability diagram for CIFAR10. 2 classes for model (top) and 5 (bottom).}

\end{figure}

\end{document}